\documentclass{article}
\usepackage{graphicx}
\usepackage{amsthm,amsmath}
\usepackage{amsfonts,amssymb,mathrsfs}
\usepackage{subfig}
\usepackage{enumerate}
\usepackage{bm}
\usepackage{url}
\usepackage{algorithm,algorithmic}
\usepackage[colorlinks=true,pagebackref,linkcolor=magenta]{hyperref}
\usepackage[sort,compress,numbers]{natbib}
\usepackage{parskip}
\newtheorem{theorem}{Theorem}[section]
\newtheorem{lemma}[theorem]{Lemma}
\newtheorem{corollary}[theorem]{Corollary}
\newtheorem{proposition}[theorem]{Proposition}

\theoremstyle{definition}
\newtheorem{definition}{Definition}
\newtheorem*{remark}{Remark}
\numberwithin{equation}{section}

\renewcommand{\tilde}{\widetilde}
\renewcommand{\Pr}{\mathbb{P}}

\newcommand{\tr}{\mathrm{tr}}
\newcommand{\hX}{\hat{\mathbf{X}}}
\renewcommand{\Pr}{\mathbb{P}}

\makeatletter
\def\thm@space@setup{%
 \thm@preskip=\parskip \thm@postskip=0pt
}
\makeatother
\begin{document}
\title{Limit theorems for eigenvectors of the normalized Laplacian for random graphs}
\author{Minh Tang and Carey E. Priebe \\ Department of Applied Mathematics and Statistics \\ Johns Hopkins University}
\maketitle

\begin{abstract}
  We prove a central limit theorem for the components of the eigenvectors corresponding to the $d$ largest eigenvalues of the normalized Laplacian matrix of a finite dimensional random dot product graph. As a corollary, we show that for stochastic blockmodel graphs, the rows of the spectral embedding of the normalized Laplacian converge to multivariate normals and furthermore the mean and the covariance matrix of each row are functions of the associated vertex's block membership. Together with prior results for the eigenvectors of the adjacency matrix, we then compare, via the Chernoff information between multivariate normal distributions, how the choice of embedding method impacts subsequent inference. We demonstrate that neither embedding method dominates with respect to the inference task of recovering the latent block assignments.   
\end{abstract}


\section{Introduction}
\label{sec:introduction}
Statistical inference on graphs is a burgeoning field of research in machine learning and statistics, with numerous applications to social network, neuroscience, etc. Many statistical inference procedures for graphs involve a preprocessing step of finding a representation of the vertices as points in some low-dimensional Euclidean space. This representation is usually given by the truncated eigendecomposition of the adjacency matrix or related matrices such as the combinatorial Laplacian or the normalized Laplacian. For example, given a point cloud lying in some purported low-dimensional manifold in a high-dimensional ambient space, many manifold learning or non-linear dimension reduction algorithms such as Laplacian eigenmaps \cite{belkin03:_laplac} and diffusion maps \cite{coifman06:_diffus_maps} use the eigenvectors of the normalized Laplacian constructed from a neighborhood graph of the points as a low-dimensional Euclidean representation of the point cloud before performing inference such as clustering or classification. Spectral clustering algorithms such as the normalized cuts algorithm \cite{shi_malik} proceed by embedding a graph into a low-dimensional Euclidean space followed by running $K$-means on the embedding to obtain a partitioning of the vertices. Some network comparison procedures embed the graphs and then compute a kernel-based distance measure between the resulting point clouds \cite{tang14:_nonpar,Asta}. 

The choice of the matrix used in the embedding step and its effect on subsequent inference is, however, rarely addressed in the literature. 
In a recent pioneering work, the authors of \cite{bickel_sarkar_2013} addressed this issue by analyzing, in the context of stochastic blockmodel graphs where the subsequent inference task is the recovery of the block assignments, a metric given by the average distance between the vertices of a block and its cluster centroid for the spectral embedding of the adjacency matrix and the normalized Laplacian matrix.  The metric is then used as a surrogate measure for the performance of the subsequent inference task, i.e., the metric is a surrogate measure for the error rate in recovering the vertices to block assignments. The stochastic blockmodel \cite{holland} is a popular generative  model for random graphs with latent community structure and many results are known regarding consistent recovery of the block assignments; see for example \cite{rohe2011spectral,sussman12,Bickel2009,lyzinski13:_perfec,rinaldo_2013,mossel:ptrf,Choi2010,Snijders1997Estimation,mcsherry} and the references therein. 

It was shown in \cite{bickel_sarkar_2013} that for two-block stochastic blockmodels, for a large regime of parameters the normalized Laplacian spectral embedding reduces the within-block variance (occasionally by a factor of four) while preserving the between-block variance, as compared to that of the adjacency spectral embedding.
This suggests that for a large region of the parameters space for two-block stochastic blockmodels, the spectral embedding of the Laplacian is to be preferred over that of the adjacency matrix for subsequent inference. However, we observed that the metric in \cite{bickel_sarkar_2013} is intrinsically tied to the use of $K$-means as the clustering procedure, i.e., a smaller value of the metric for the Laplacian spectral embedding as compared to that for the adjacency spectral embedding only implies that clustering the Laplacian spectral embedding using $K$-means is possibly better than clustering the adjacency spectral embedding using $K$-means. 

Motivated by the above observation, one main goal of this paper is to propose a metric that is {\em independent} of any specific clustering procedure, i.e., a metric that characterizes the minimum error achievable by {\em any} clustering procedure that uses only the spectral embedding, for the recovery of block assignments in stochastic blockmodel graphs. We achieve this by establishing distributional limit results for the eigenvectors corresponding to the few largest eigenvalues of the adjacency or Laplacian matrix and then characterizing, through the notion of statistical information, the distributional differences between the blocks for either embedding method. Roughly speaking, smaller statistical information implies less information to discriminate between the blocks of the stochastic blockmodel. 

More specifically, the limit result in \cite{athreya2013limit} states that, for stochastic blockmodel graphs, conditional on the block assignments the scaled eigenvectors corresponding to the few largest eigenvalues of the adjacency matrix converge to a multivariate normal (see e.g., Theorem~\ref{THM:NORMALITY-ASE}) as the number of vertices increases. Furthermore, the associated covariance matrix is not necessarily spherical and hence $K$-means clustering for the adjacency spectral embedding does not always yield minimum error for recovering the block assignment. Analogous limit results (see e.g., Theorem~\ref{THM:NORMALITY-LSE}) for the eigenvectors of the normalized Laplacian matrix then facilitate comparison between the two embedding methods via the classical notion of Chernoff information \cite{chernoff_1952}. The Chernoff information is a supremum of the Chernoff $\alpha$-divergences for $\alpha \in (0,1)$ and characterizes the error rate of the Bayes decision rule in hypothesis testing; the Chernoff $\alpha$-divergence is an example of a $f$-divergence \cite{Csizar,Ali-Shelvey} and it satisfies the information processing lemma and is invariant with respect to invertible transformations \cite{Liese_Vadja}. 

Our paper is thus structured as follows. We recall in Section~\ref{sec:setup} the definition of random dot product graphs, stochastic blockmodel graphs, and spectral embedding of the adjacency and Laplacian matrices. We then state in Section~\ref{sec:limit-result-adjacency} several limit results for the eigenvectors of the adjacency spectral embedding. These results are generalizations of results from \cite{athreya2013limit,tang14:_semipar}. The main technical contribution of this paper, namely analogous limit results for the eigenvectors of the Laplacian spectral embedding, are then given in Section~\ref{sec:limit-result-Laplacian}. We then discuss the implications of these limit results in Section~\ref{sec:simulations}; in particular Section~\ref{sec:chernoff2} characterizes, via the notion of Chernoff statistical information, the {\em large-sample optimal} error rate of spectral clustering procedures. We demonstrate that neither embedding method dominates for the inference task of recovering block assignments in stochastic blockmodels. We conclude the paper with some brief remarks on potential extensions of the results presented herein. Proofs of stated results are given in the appendix.


\section{Background and Setting}
\label{sec:setup}
We first recall the notion of a random dot product graph \cite{nickel2006random}.
\begin{definition}
\label{def:rdpg}
  Let $F$ be a distribution on a set $\mathcal{X}\subset \mathbb{R}^d$
  satisfying $x^{\top} y \in[0,1]$ for all $x,y\in
  \mathcal{X}$. We say $(\mathbf{X},\mathbf{A})\sim \mathrm{RDPG}(F)$
  with sparsity factor $\rho_n \leq 1$ if the following hold. Let
  $X_1,\dotsc, X_n {\sim} F$ be independent random variables and define
\begin{equation}
\label{eq:defXP}
\mathbf{X}=[X_1 \mid \cdots \mid X_n]^\top\in \mathbb{R}^{n\times d}\text{
  and } \mathbf{P}= \rho_n \mathbf{X} \mathbf{X}^\top\in [0,1]^{n\times n}.
\end{equation}
The $X_i$ are the {\em latent}  positions for the random graph, i.e., we do not observe $\mathbf{X}$, rather we observe only the matrix $\mathbf{A}$. The matrix $\mathbf{A}\in\{0,1\}^{n\times n}$ is defined to be symmetric
with all zeroes on the diagonal such that for all $i<j$, conditioned
on $X_i,X_j$ the $A_{ij}$ are independent and
\begin{equation}
 A_{ij} \sim\mathrm{Bernoulli}(\rho_n X_i^\top X_j),
\end{equation}
namely,
\begin{equation}
\Pr[\mathbf{A} \mid \mathbf{X}]=\prod_{i <j} (\rho_n X^{\top}_i
X_j)^{A_{ij}}(1- \rho_n X^{\top}_i X_j)^{(1-A_{ij})}.
\end{equation}
\end{definition}

\begin{remark}
We note that non-identifiability is an intrinsic property of random dot product graphs. More specifically, 
if $(\mathbf{X}, \mathbf{A}) \sim \mathrm{RDPG}(F)$ where $F$ is a distribution on $\mathbb{R}^{d}$, then for any orthogonal transformation $U$, $(\mathbf{Y}, \mathbf{B}) \sim \mathrm{RDPG}(F \circ U)$ is identically distributed to $(\mathbf{X}, \mathbf{A})$; we write $F \circ U$ to denote the distribution of $Y = UX$ whenever $X \sim F$. Furthermore, there also exists a distribution $F'$ on $\mathbb{R}^{d'}$ with $d' > d$ such that $(\mathbf{Y}, \mathbf{B}) \sim \mathrm{RDPG}(F')$ is identically distributed to $(\mathbf{X}, \mathbf{A})$. Non-identifiability due to orthogonal transformations cannot be avoided given the observed $\mathbf{A}$. We avoid the other source of non-identifiability by assuming throughout this paper that if $(\mathbf{X}, \mathbf{A}) \sim \mathrm{RDPG}(F)$ then $F$ is non-degenerate, i.e., $\mathbb{E}[X X^{\top}]$ is of full rank.  
\end{remark}


As an example of random dot product graphs, we could take $\mathcal{X}$ to be the unit simplex in $\mathbb{R}^{d}$ and let $F$ be a mixture of Dirichlet distributions or logistic-normal distribution. Random dot product graphs are a specific example of latent position graphs or inhomogeneous random graphs \cite{Hoff2002,bollobas07}, in which each vertex is associated with a latent position $X_i$ and, conditioned on the latent positions, the presence or absence of the edges in the graph are independent Bernoulli random variables where the probablity of an edge between any two vertices with latent positions $X_i$ and $X_j$ is given by $\kappa(X_i, X_j)$ for some symmetric function $\kappa$. A random dot product graph on $n$ vertices is also, when viewed as an induced subgraph of an infinite graph, an exchangeable random graph \cite{diaconis08:_graph_limit_exchan_random_graph}. Random dot product graphs are related to stochastic block model graphs \cite{holland} and degree-corrected stochastic block model graphs \cite{karrer2011stochastic}; for example, a stochastic blockmodel graph on $K$ blocks with a positive semidefinite block probability matrix $\mathbf{B}$ corresponds to a random dot product graph where $F$ is a mixture of $K$ point masses.

For a given matrix $\mathbf{M}$ with non-negative entries, denote
by $\mathcal{L}(\mathbf{M})$ the {\em normalized} Laplacian of $\mathbf{M}$
defined as
\begin{equation}
\label{eq:Laplacian}
  \mathcal{L}(\mathbf{M}) = (\mathrm{diag}(\mathbf{M} \bm{1}))^{-1/2}
  \mathbf{M} (\mathrm{diag}(\mathbf{M} \bm{1}))^{-1/2}
\end{equation}
where, given $\bm{z} = (z_1, \dots, z_n) \in \mathbb{R}^{n}$, $\mathrm{diag}(\bm{z})$ is the $n \times n$
diagonal matrix whose diagonal entries are the $z_i$'s. Our definition of the normalized Laplacian is slightly different from that often found in the literature, e.g., in \cite{chung1997spectral,shi_malik} the normalized Laplacian is $\mathbf{I} - \mathcal{L}(\mathbf{M})$. For the purpose of this paper, namely the notion of the Laplacian spectral embedding via the eigenvalues and eigenvectors of the normalized Laplacian, these two definitions of the normalized Laplacian are equivalent. We shall henceforth refer to $\mathcal{L}(\mathbf{M})$ as the Laplacian of $\mathbf{M}$, 
in contrast to the {\em combinatorial} Laplacian $\mathrm{diag}(\mathbf{M} \bm{1}) - \mathbf{M}$ of $\mathbf{M}$. See \cite{merris1994} for a survey of the combinatorial Laplacian and its connection to graph theory.


\begin{definition}[Adjacency and Laplacian spectral embedding]
\label{defn:ase-lse}
  Let $\mathbf{A}$ be a $n \times n$ adjacency matrix. Suppose the eigendecomposition of $\mathbf{A}$ is given by
  $\mathbf{A} = \sum_{i=1}^{n} \lambda_i \bm{u}_i \bm{u}_i^{\top} $ where $|\lambda_1| \geq |\lambda_2| \geq \dots$ are the eigenvalues and $\bm{u}_1, \bm{u}_2, \dots, \bm{u}_n$ are the corresponding orthonormal eigenvectors. Given a positive integer $d \leq n$, denote by $\mathbf{S}_{\mathbf{A}} = \mathrm{diag}(|\lambda_1|, \dots, |\lambda_d|)$ the diagonal matrix whose diagonal entries are the $|\lambda_1|, \dots, |\lambda_d|$, and denote by $\mathbf{U}_{\mathbf{A}}$ the $n \times d$ matrix whose columns are the corresponding eigenvectors $\bm{u}_1, \dots, \bm{u}_d$. The {\em adjacency spectral embedding} (ASE) of $\mathbf{A}$ into $\mathbb{R}^{d}$ is then the $n \times d$ matrix $\hat{\mathbf{X}} = \mathbf{U}_{\mathbf{A}} \mathbf{S}_{\mathbf{A}}^{1/2}$. Similarly, let $\mathcal{L}(\mathbf{A})$ denote the normalized Laplacian of $\mathbf{A}$ and suppose the eigendecomposition of $\mathcal{L}(\mathbf{A})$ is given by 
$\mathcal{L}(\mathbf{A}) = \sum_{i=1}^{n} \tilde{\lambda}_i \tilde{\bm{u}}_i \tilde{\bm{u}}_i^{\top}$ where $\tilde{\lambda}_1 \geq \tilde{\lambda}_2 \geq \dots \geq \tilde{\lambda}_n \geq 0$ are the eigenvalues and $\tilde{\bm{u}}_1, \tilde{\bm{u}}_2, \dots, \tilde{\bm{u}}_n$ are the corresponding orthonormal eigenvectors. Then given a positive integer $d \leq n$, denote by $\tilde{\mathbf{S}}_{\mathbf{A}} = \mathrm{diag}(\tilde{\lambda}_1, \dots, \tilde{\lambda}_d)$ the diagonal matrix whose diagonal entries are the $\tilde{\lambda}_1, \dots, \tilde{\lambda}_d$ and denote by $\tilde{\mathbf{U}}_{\mathbf{A}}$ the $n \times d$ matrix whose columns are the eigenvectors $\tilde{\bm{u}}_1, \dots, \tilde{\bm{u}}_d$. The {\em Laplacian spectral embedding} of $\mathbf{A}$ into $\mathbb{R}^{d}$ is then the $n \times d$ matrix $\breve{\mathbf{X}} = \tilde{\mathbf{U}}_{\mathbf{A}} \tilde{\mathbf{S}}_{\mathbf{A}}^{1/2}$.
\end{definition}

\begin{remark}
Let $(\mathbf{X}, \mathbf{A}) \sim \mathrm{RDPG}(F)$ with sparsity factor $\rho_n$ and suppose that the $d \times d$ matrix $\mathbb{E}[X X^{\top}]$ is of full-rank where $X \sim F$. The $n \times d$ matrix $\hat{\mathbf{X}}$, the adjacency spectral embedding $\hat{\mathbf{X}}$ of $\mathbf{A}$ into $\mathbb{R}^{d}$, can then be viewed as a consistent estimate of $\rho_n^{1/2} \mathbf{X}$. See \cite{sussman:_thesis} for a comprehensive overview of the consistency results and their implications for subsequent inference. 
On the other hand, as $\mathcal{L}(c \mathbf{M}) = \mathcal{L}(\mathbf{M})$ for any constant $c > 0$, the $n \times d$ matrix $\breve{\mathbf{X}}$ -- the normalized Laplacian embedding of $\mathbf{A}$ into $\mathbb{R}^{d}$ -- can be viewed as a consistent estimate of $(\rho_n \mathrm{diag}(\mathbf{X} \mathbf{X}^{\top} \bm{1}))^{-1/2} \rho_n^{1/2} \mathbf{X}$ which does
 not depend on the sparsity factor $\rho_n$. This is in contrast to the adjacency spectral embedding. For previous consistency results of $\breve{\mathbf{X}}$ as an estimator for $\tilde{\mathbf{X}}$ in various random graphs models, the reader is referred to \cite{rohe2011spectral,qin2013dcsbm,luxburg08:_consis} among others. However, to the best of our knowledge, Theorem~\ref{THM:NORMALITY-LSE} -- namely the distributional convergence of $\breve{\mathbf{X}}$ to a mixture of multivariate normals in the context of random dot product graphs and stochastic blockmodel graphs -- had not been established prior to this paper. Finally, we remark that $\hat{\mathbf{X}}$ and $\breve{\mathbf{X}}$ are estimating quantities that, while closely related -- $\mathbf{X}$ and $(\mathrm{diag}(\mathbf{X} \mathbf{X}^{\top} \bm{1}))^{-1/2} \mathbf{X}$ are one-to-one transformations of each other --  are in essence distinct ``parametrizations'' of random dot product graphs. It is therefore not entirely straightforward to facilitate a direct comparison of the ``efficiency'' of $\hat{\mathbf{X}}$ and $\breve{\mathbf{X}}$ as estimators. This thus motivates our consideration of the $f$-divergences between the multivariate normals since the family of $f$-divergences satisfy the information processing lemma and are invariant with respect to invertible transformations.
\end{remark}

\begin{remark}
\label{rem:sparsity}
For simplicity we shall assume henceforth that either $\rho_n = 1$ for all
$n$, or that $\rho_n \rightarrow 0$ with $n \rho_n =
\omega(\log^{4}{n})$. We note that for our purpose, namely the distributional limit results in Section~\ref{sec:limit-result-adjacency} and Section~\ref{sec:limit-result-Laplacian}, the assumption that $\rho_n = 1$ for all $n$ is 
equivalent to the assumption that there exists a constant $c > 0$ such that $\rho_n \rightarrow c$. The assumption that $n \rho_n = \omega(\log^{4}{n})$ is so that we can apply the concentration inequalties from \cite{lu13:_spect} to show concentration, in spectral norm, of $\mathbf{A}$ and $\mathcal{L}(\mathbf{A})$ around $\rho_n \mathbf{X} \mathbf{X}^{\top}$ and $\mathcal{L}(\mathbf{X} \mathbf{X}^{\top})$, respectively. 
\end{remark}

\subsection{Limit results for the adjacency spectral embedding}
\label{sec:limit-result-adjacency}
We now recall several limit results for $\hat{\mathbf{X}} -
\mathbf{X}$. These results are restatements of earlier results from \cite{athreya2013limit} and \cite{tang14:_semipar}. Theorem~\ref{THM:NORMALITY-ASE} as stated below is a slight generalization of Theorem~1 in \cite{athreya2013limit}; the result in \cite{athreya2013limit} assumed a more restrictive distinct eigenvalues assumption for the matrix $\mathbb{E}[X X^{\top}]$ where $X \sim F$. We shall assume throughout this paper that $d$, the rank of $\mathbb{E}[X X^{\top}]$ where $X \sim F$, is fixed and known {\em a priori}.

\begin{remark}
\label{rem:whp}
For ease of exposition, many of the bounds in this paper are said to hold ``with high probability''. We say that a random variable $\xi \in \mathbb{R}$ is $O_{\mathbb{P}}(f(n))$ if, for any positive constant $c > 0$ there exists a $n_0$ and a constant $C > 0$ (both of which possibly depend on $c$) such that for all $n \geq n_0$, $|\xi| \leq C f(n)$ with probability at least $1 - n^{-c}$; in addition, we say that a random variable $\xi \in \mathbb{R}$ is $o_{\mathbb{P}}(f(n))$ if for any positive constant $c > 0$ and any $\epsilon > 0$ there exists a $n_0$ such that for all $n \geq n_0$, $|\xi| \leq \epsilon f(n)$ with probability at least $1 - n^{-c}$. Similarly, when $\xi$ is a random vector in $\mathbb{R}^{d}$ or a random matrix in $\mathbb{R}^{d_1 \times d_2}$, $\xi = O_{\mathbb{P}}(f(n))$ or $\xi = o_{\mathbb{P}}(f(n))$ if $\|\xi\| = O_{\mathbb{P}}(f(n))$  or $\|\xi\| = o_{\mathbb{P}}(f(n))$, respectively. Here $\|x\|$ denotes the Euclidean norm of $x$ when $x$ is a vector and the spectral norm of $x$ when $x$ is a matrix. We write $\xi = \zeta + O_{\mathbb{P}}(f(n))$ or $\xi = \zeta + o_{\mathbb{P}}(f(n))$ if $\xi - \zeta = O_{\mathbb{P}}(f(n))$ or $\xi - \zeta = o_{\mathbb{P}}(f(n))$, respectively. 
\end{remark}

\begin{theorem}
  \label{THM:ASE}
Let $(\mathbf{X}_n, \mathbf{A}_n) \sim \mathrm{RDPG}(F)$ with sparsity factor $\rho_n$. Then there
exists a $d \times d$ orthogonal matrix $\mathbf{W}_n$ and a $n \times d$ matrix $\mathbf{R}_n$ such that
\begin{equation}
  \label{eq:6}
  \hat{\mathbf{X}}_n \mathbf{W}_n
 -  \rho_{n}^{1/2} \mathbf{X}_n = \rho_n^{-1/2} (\mathbf{A}_n - \mathbf{P}_n) \mathbf{X}_n (\mathbf{X}_n^{\top} \mathbf{X}_n)^{-1} + \mathbf{R}_n.
\end{equation}
Furthermore, $\|\mathbf{R}\| = O_{\mathbb{P}}((n \rho_n)^{-1/2})$.
Let $\mu_{F} =
\mathbb{E}[X_1]$ and $\Delta = \mathbb{E}[X_1 X_1^{\top}]$. If $\rho_n
= 1$ for all $n$, then there exists a sequence of orthogonal matrices $\mathbf{W}_n$ such that
  \begin{equation}
    \label{eq:5}
  \|\hX_n \mathbf{W}_n - \mathbf{X}_n \|_F^{2}
  \overset{\mathrm{a.s.}}{\longrightarrow} \mathrm{tr} \,\, \Delta^{-1} \Bigl(\mathbb{E}[X_1 X_1^{\top}
  (X_1^{\top} 
\mu_{F} -  X_1^{\top} \Delta X_1)] \Bigr)
\Delta^{-1}.
\end{equation}
If, however, $\rho_n \rightarrow 0$ and $n \rho_n =
\omega(\log^{4}{n})$, then
\begin{equation}
\label{eq:rho_n_decreasing}
   \|\hX_n \mathbf{W}_n - \rho_n^{1/2} \mathbf{X} \|_F^{2}
  \overset{\mathrm{a.s.}}{\longrightarrow} \mathrm{tr} \,\, \Delta^{-1} \Bigl(\mathbb{E}[X_1 X_1^{\top}
  (X_1^{\top} \mu_{F})] \Bigr)
\Delta^{-1}.
\end{equation}
\end{theorem}

\begin{theorem}
\label{THM:NORMALITY-ASE}
  Assume the setting and notations of Theorem~\ref{THM:ASE}. Denote by $\hat{X}_i$ the $i$-th row of
 $\hat{\mathbf{X}}_n$. Let
  $\Phi(z, \Sigma)$ denote the cumulative distribution function for
  the multivariate normal, with mean zero and covariance matrix
  $\Sigma$, evaluated at $z$.
Also denote by $\Sigma(x)$ the matrix
  \begin{equation*}
    \Sigma(x) = \Delta^{-1} \mathbb{E}[X_1 X_1^{\top} (x^{\top} X_1 - x^{\top} X_1 X_1^{\top} x)]
     \Delta^{-1}
  \end{equation*}
If $\rho_n = 1$ for all $n$, then there exists a sequence of 
  orthogonal matrices 
  $\mathbf{W}_n$ such that for each fixed index $i$ and any $z \in \mathbb{R}^{d}$,
  \begin{equation}
    \label{eq:Xhat1}
    \Pr\Bigl\{ \sqrt{n} (\mathbf{W}_n \hat{X}_i - X_i) \leq z\Bigr\}
\overset{\mathrm{d}}{\longrightarrow}  \int \Phi(z, \Sigma(x)) dF(x)
  \end{equation}
  That is, the sequence 
  $\sqrt{n}( \mathbf{W}_n \hat{X}_i - X_i)$ converges in distribution to
  a mixture of multivariate normals.  We denote this mixture by
  $\mathcal{N}(0, \tilde{\Sigma}(X_i))$. If, however, $\rho_n
  \rightarrow 0$ and $n \rho_n = \omega(\log^{4}{n})$ then there
  exists a sequence of orthogonal matrices $\mathbf{W}_n$ such that
  \begin{equation}
    \label{eq:Xhat2-clt-ase}
     \Pr\Bigl\{ \sqrt{n} (\mathbf{W}_n \hat{X}_i - \rho_n^{1/2} X_i) \leq z\Bigr\}
\overset{\mathrm{d}}{\longrightarrow}  \int \Phi(z, \Sigma_{o(1)}(x)) dF(x)
  \end{equation}
  where $\Sigma_{o(1)}(x) = \Delta^{-1} \mathbb{E}[X_1 X_1^{\top} x^{\top} X_1] \Delta^{-1}$.
\end{theorem}

An important corollary of Theorem~\ref{THM:NORMALITY-ASE} is the following result for when $F$ is a mixture of $K$ point masses, i.e., $(\mathbf{X}, \mathbf{A}) \sim \mathrm{RDPG}(F)$ is a $K$-block stochastic blockmodel graph. Then for any fixed index $i$, the event that $X_i$ is assigned to block $k \in \{1,2,\dots,K\}$ has non-zero probabilty and hence one can conditioned on the block assignment of $X_i$ to show that the conditional distribution of $\sqrt{n}(\mathbf{W}_n \hat{X}_i - X_i)$ converges to a multivariate normal. This is in contrast to the unconditional distribution being a mixture of multivariate normals as in Eq.~\eqref{eq:Xhat1} and Eq.~\eqref{eq:Xhat2-clt-ase}.

\begin{corollary}
\label{cor:ase_normality_sbm}
Assume the setting and notations of Theorem~\ref{THM:ASE} and let $$F = \sum_{k=1}^{K} \pi_{k} \delta_{\nu_k}, \quad \pi_1, \cdots, \pi_K > 0, \sum_{k} \pi_k = 1$$
be a mixture of $K$ point masses in $\mathbb{R}^{d}$ where $\delta_{\nu_k}$ is the Dirac delta measure at $\nu_k$. 
Then if $\rho_n \equiv 1$, there exists a sequence of orthogonal matrices $\mathbf{W}_n$ such that for any fixed index $i$,
\begin{equation}
\mathbb{P}\Bigl\{\sqrt{n}(\mathbf{W}_n \hat{X}_i - X_i) \leq z  \mid X_i = \nu_k \Bigr\} \overset{\mathrm{d}}{\longrightarrow} \mathcal{N}(0, \Sigma_k)
\end{equation}
where $\Sigma_k = \Sigma(\nu_k)$ is as defined in Eq.~\eqref{eq:Xhat1}. If $\rho_n \rightarrow 0$ and $n \rho_n = \omega(\log^{4}(n))$ as $n \rightarrow \infty$, then the sequence of orthogonal matrices $\mathbf{W}_n$ satisfies
\begin{equation}
\mathbb{P}\Bigl\{\sqrt{n}(\mathbf{W}_n \hat{X}_i - \rho_n^{1/2} X_i) \leq z  \mid X_i = \nu_k \Bigr\} \overset{\mathrm{d}}{\longrightarrow} \mathcal{N}(0, \Sigma_{o(1),k})
\end{equation}
where $\Sigma_{o(1),k} = \Sigma_{o(1)}(\nu_k)$ is as defined in Eq.~\eqref{eq:Xhat2-clt-ase}. 
\end{corollary}

\section{Limit results for Laplacian spectral embedding}
\label{sec:limit-result-Laplacian}
We now present the main technical results of this paper, namely analogues of the limit results in Section~\ref{sec:limit-result-adjacency} for the Laplacian spectral embedding.
\begin{theorem}
  \label{THM:LSE}
Let $(\mathbf{A}_n, \mathbf{X}_n) \sim \mathrm{RDPG}(F)$ for $n \geq 1$ be a sequence of random dot product graphs with sparsity factors $(\rho_n)_{n \geq 1}$. Denote by $\mathbf{D}_n$ and $\mathbf{T}_n$ the $n \times n$ diagonal matrices $\mathrm{diag}(\mathbf{A}_n \bm{1})$ and $\mathrm{diag}(\rho_n \mathbf{X}_n \mathbf{X}_n^{\top} \bm{1})$, respectively, i.e., the diagonal entries of $\mathbf{D}_n$ are the vertex degrees of $\mathbf{A}_n$ and the diagonal entries of $\mathbf{T}_n$ are the {\em expected} vertex degrees. Let $\tilde{\mathbf{X}}_n = \rho_n^{1/2} \mathbf{T}_n^{-1/2} \mathbf{X}_n = \mathrm{diag}(\mathbf{X}_n \mathbf{X}_n^{\top} \bm{1})^{-1/2} \mathbf{X}_n$.  
Then for any $n$, there
exists a $d \times d$ orthogonal matrix $\mathbf{W}_n$ and a $n \times d$ matrix $\mathbf{R}_n$ such that $\zeta_n := (\breve{\mathbf{X}}_n \mathbf{W}_n - \tilde{\mathbf{X}}_n )$ satisfies
\begin{equation}
\label{eq:LSE-main1}
 \zeta_n = \mathbf{T}_n^{-1/2} (\mathbf{A}_n - \mathbf{P}_n) \mathbf{T}_n^{-1/2}
  \tilde{\mathbf{X}}_n (\tilde{\mathbf{X}}_n^{\top} \tilde{\mathbf{X}}_n)^{-1} + \tfrac{1}{2}(\mathbf{I} - \mathbf{D}_n \mathbf{T}_n^{-1}) \tilde{\mathbf{X}}_n + \mathbf{R}_n.
\end{equation}
Furthermore, $\|\mathbf{R}_n\|_{F} = O_{\mathbb{P}}((n \rho_n)^{-1})$, i.e., $\|\mathbf{R}_n\|/\|\zeta_n\| \overset{\mathrm{a.s.}}{\longrightarrow} 0$ as $n \rightarrow \infty$. Define the following quantities
\begin{gather}
  \label{eq:mu_defn1}
  \mu = \mathbb{E}[X_1]; \quad \tilde{\mu} = \mathbb{E}\Bigl[
  \frac{X_1}{X_1^{\top} \mu} \Bigr]; \quad \tilde{\Delta} =
  \mathbb{E}\Bigl[ \frac{ X_1 X_1^{\top}}{X_1^{\top} \mu} \Bigr];\quad \text{and} \\ 
  \label{eq:mu_defn2}
  g(X_1,X_2) = \Bigl(\frac{\tilde{\Delta}^{-1} X_1}{X_1^{\top} \mu} - \frac{X_2}{2 X_2^{\top} \mu}\Bigr) \Bigl(\frac{\tilde{\Delta}^{-1} X_1}{X_1^{\top} \mu} - \frac{X_2}{2 X_2^{\top} \mu}\Bigr)^{\top}.
\end{gather}
 If $\rho_n
\equiv 1$ then the sequence of orthogonal matrices $(\mathbf{W}_n)_{n \geq 1}$ satisfies
  \begin{equation}
    \label{eq:4}
  n \|\breve{\mathbf{X}}_n \mathbf{W}_n - \tilde{\mathbf{X}}_n \|_F^{2}
  \overset{\mathrm{a.s.}}{\rightarrow} \mathrm{tr} \,\,
  \mathbb{E}\Bigl[g(X_1, X_2) \frac{X_1^{\top} X_2 - X_1^{\top} X_2 X_2^{\top} X_1}{X_2^{\top} \mu} 
  \Bigr].
  \end{equation}
  where the expectation in Eq.~\eqref{eq:4} is taken with respect to $X_1$ and $X_2$ being i.i.d drawn according to $F$. 
  Equivalently, 
  \begin{equation*}
  \begin{split}
    n \|\breve{\mathbf{X}}_n \mathbf{W}_n - \tilde{\mathbf{X}}_n \|_F^{2}
  & \overset{\mathrm{a.s.}}{\longrightarrow} \mathrm{tr} \,\,
  \mathbb{E}\Bigl[ \frac{ \tilde{\Delta}^{-2} X_1 X_1^{\top} (X_1^{\top} \tilde{\mu} - X_1^{\top} \tilde{\Delta} X_1)}{(X_1^{\top} \mu)^2} - \frac{3 X_1 X_1^{\top}}{4 (X_1^{\top} \mu)^2} \Bigr] \\ & + \mathrm{tr} \,\, \mathbb{E}\Bigl[\frac{\tilde{\Delta}^{-1} X_1 X_1^{\top} X_2 X_2^{\top} (X_1^{\top} X_2)}{X_1^{\top} \mu (X_2^{\top} \mu)^2} - \frac{X_1 X_1^{\top} (X_1^{\top} \Delta X_1)}{4 (X_1^{\top} \mu)^3} \Bigr].
  \end{split}
\end{equation*}
If $\rho_n \rightarrow 0$ and $n \rho_n =
\omega(\log^{4}{n})$ then the sequence $(\mathbf{W}_n)_{n \geq 1}$ satisfies
\begin{equation}
  \label{eq:14}
      n \rho_n \|\breve{\mathbf{X}} \mathbf{W}_n - \tilde{\mathbf{X}}_n \|_F^{2}
   \overset{\mathrm{a.s.}}{\longrightarrow} \mathrm{tr} \,\, \mathbb{E}\Bigl[\frac{ \tilde{\Delta}^{-2} X_1 X_1^{\top} (X_1^{\top} \tilde{\mu})}{(X_1^{\top} \mu)^2} - \frac{3 X_1 X_1^{\top}}{4 (X_1^{\top} \mu)^2} \Bigr].
\end{equation}

\end{theorem}

As a companion of Theorem~\ref{THM:LSE}, we have the following result on the asymptotic normality of the rows of $\breve{\mathbf{X}}_n \mathbf{W}_n - \tilde{\mathbf{X}}_n$. 
\begin{theorem}
\label{THM:NORMALITY-LSE}
 Assume the setting and notations of Theorem~\ref{THM:LSE}. 
 Denote by
 $\breve{X}_i$ and $\tilde{X}_i$ the $i$-th row of
 $\breve{\mathbf{X}}_n$ and $\tilde{\mathbf{X}}_n$, respectively. Also denote by $\tilde{\Sigma}(x)$ the matrix
  \begin{equation}
  \label{eq:lse-sigma}
    \mathbb{E}\Bigl[ \Bigl(\frac{\tilde{\Delta}^{-1}X_1}{X_1^{\top} \mu} - \frac{x}{2 x^{\top} \mu}\Bigr) \Bigl(\frac{X_1^{\top} \tilde{\Delta}^{-1}}{X_1^{\top} \mu} - \frac{x^{\top}}{2 x^{\top} \mu}\Bigr) \frac{(x^{\top} X_1 - x^{\top} X_1 X_1^{\top} x)}{x^{\top} \mu} \Bigr].
  \end{equation}
If $\rho_n \equiv 1$ then there exists a sequence of 
  orthogonal matrices
 $\mathbf{W}_n$ such that for each fixed index $i$ and any $z \in \mathbb{R}^{d}$,
  \begin{equation}
    \label{eq:Xtilde_clt}
    \Pr\Bigl\{n\bigl( \mathbf{W}_n \breve{X}_i - \tfrac{X_i}{\sqrt{\sum_{j} X_i^{\top} X_j}} \bigr) \leq z\Bigr\}
\overset{\mathrm{d}}{\longrightarrow}  \int \Phi(z, \tilde{\Sigma}(x)) dF(x)
  \end{equation}
  That is, the sequence 
  $n( \mathbf{W}_n \breve{X}_i - X_i/\sqrt{\sum_{j} X_i^{\top} X_j})$ converges in distribution to
  a mixture of multivariate normals.  We denote this mixture by
  $\mathcal{N}(0, \tilde{\Sigma}(X_i))$. If $\rho_n \rightarrow 0$ and $n \rho_n = \omega(\log^{4}{n})$ then there exists a sequence of orthogonal matrices $\mathbf{W}_n$ such that
  \begin{equation}
  \label{eq:Xtilde_clt_vanishing}
   \Pr\Bigl\{n \rho_n^{1/2} \bigl( \mathbf{W}_n \breve{X}_i - \tfrac{X_i}{\sqrt{\sum_{j} X_i^{\top} X_j}}\bigr) \leq z\Bigr\}
\overset{\mathrm{d}}{\longrightarrow}  \int \Phi(z, \tilde{\Sigma}_{o(1)}(x)) dF(x).
  \end{equation}
  where $\tilde{\Sigma}_{o(1)}(x)$ is defined by 
  \begin{equation}
  \label{eq:lse-sigma-o1}
  \tilde{\Sigma}_{o(1)}(x) = \mathbb{E}\Bigl[ \Bigl(\frac{\tilde{\Delta}^{-1}X_1}{X_1^{\top} \mu} - \frac{x}{2 x^{\top} \mu}\Bigr) \Bigl(\frac{X_1^{\top} \tilde{\Delta}^{-1}}{X_1^{\top} \mu} - \frac{x^{\top}}{2 x^{\top} \mu}\Bigr) \frac{x^{\top} X_1}{x^{\top} \mu} \Bigr].
  \end{equation}
\end{theorem}

The proofs of Theorem~\ref{THM:LSE} and Theorem~\ref{THM:NORMALITY-LSE} are given in Section~\ref{sec:lse}. 
We end this section by stating the conditional distribution of $n \rho_n (\breve{X}_i - \tilde{X}_i)$ when $(\mathbf{X}, \mathbf{A}) \sim \mathrm{RDPG}(F)$ is a $K$-block stochastic blockmodel graph. 
\begin{corollary}
\label{cor:lse_normality_sbm}
Assume the setting and notations of Theorem~\ref{THM:LSE} and let $$F = \sum_{k=1}^{K} \pi_{k} \delta_{\nu_k}, \quad \pi_1, \cdots, \pi_K > 0, \sum_{k} \pi_k = 1$$
be a mixture of $K$ point masses in $\mathbb{R}^{d}$. 
Then if $\rho_n \equiv 1$, there exists a sequence of orthogonal matrices $\mathbf{W}_n$ such that for any fixed index $i$,
\begin{equation}
\mathbb{P}\Bigl\{n \bigl(\mathbf{W}_n \breve{X}_i - \tfrac{\nu_k}{\sqrt{\sum_{l} n_l \nu_k^{\top} \nu_l}} \bigr) \leq z  \mid X_i = \nu_k \Bigr\} \overset{\mathrm{d}}{\longrightarrow} \mathcal{N}(0, \tilde{\Sigma}_k)
\end{equation}
where $\tilde{\Sigma}_k = \tilde{\Sigma}(\nu_k)$ is as defined in Eq.~\eqref{eq:lse-sigma} and $n_k$ for $k \in \{1,2,\dots,K\}$ denote the number of vertices in $\mathbf{A}_n$ that are assigned to block $k$. If instead $\rho_n \rightarrow 0$ and $n \rho_n = \omega(\log^{4}(n))$ as $n \rightarrow \infty$ then the sequence of orthogonal matrices $\mathbf{W}_n$ satisfies
\begin{equation}
\mathbb{P}\Bigl\{n \rho_n^{1/2} \bigl(\mathbf{W}_n \breve{X}_i - \tfrac{\nu_k}{\sqrt{\sum_{l} n_l \nu_k^{\top} \nu_l}} \bigr) \leq z  \mid X_i = \nu_k \Bigr\} \overset{\mathrm{d}}{\longrightarrow} \mathcal{N}(0, \tilde{\Sigma}_{o(1),k})
\end{equation}
where $\tilde{\Sigma}_{o(1),k} = \tilde{\Sigma}_{o(1)}(\nu_k)$ is as defined in Eq.~\eqref{eq:lse-sigma-o1}. 
\end{corollary}

\begin{remark}
As a special case of Corollary~\ref{cor:lse_normality_sbm}, we have that if $\mathbf{A}$ is an Erd\H{o}s-R\'{e}nyi graph on $n$ vertices with edge probability $p^2$ -- which corresponds to a random dot product graph where the latent positions are identically $p$ -- then for each fixed index $i$, the normalized Laplacian embedding satisfies
$$ n\bigl(\breve{X}_i - \tfrac{1}{\sqrt{n}}\bigr) \overset{\mathrm{d}}{\longrightarrow} \mathcal{N}\bigl(0, \tfrac{1 - p^2}{4p^2}\bigr),$$ 
while the adjacency spectral embedding satisfies $$\sqrt{n}(\hat{X}_i - p) \overset{\mathrm{d}}{\longrightarrow} \mathcal{N}(0, 1 - p^2).$$ 
As another example, if $\mathbf{A}$ is a stochastic blockmodel graph with block probabilities matrix $\mathbf{B} = \bigl[\begin{smallmatrix} p^2 & pq \\ pq & q^2 \end{smallmatrix}\bigr]$ and block assignment probabilities $(\pi, 1- \pi)$ -- which corresponds to a random dot product graph where the latent positions are either $p$ with probability $\pi$ or $q$ with probability $1 - \pi$ -- then for each fixed index $i$, the normalized Laplacian embedding satisfies
\begin{gather}
\label{eq:er-p-q-lse1}
 n\bigl(\breve{X}_i - \tfrac{p}{\sqrt{n_1 p^2 + n_2 pq}}\bigr) \overset{\mathrm{d}}{\longrightarrow} \mathcal{N}\Bigl(0, \tfrac{\pi p (1 - p^2) + (1 - \pi) q(1 - pq)}{4 (\pi p + (1 - \pi)q)^3}\Bigr) \,\, \text{if $X_i = p$}, \\
 \label{eq:er-p-q-lse2}
n\bigl(\breve{X}_i - \tfrac{q}{\sqrt{n_1 pq + n_2 q^2}}\bigr) \overset{\mathrm{d}}{\longrightarrow} \mathcal{N}\Bigl(0, \tfrac{\pi p (1 - pq) + (1 - \pi) q(1 - q^2)}{4 (\pi p + (1 - \pi)q)^3}\Bigr) \,\, \text{if $X_i = q$}. 
\end{gather}
where $n_1$ and $n_2 = n - n_1$ are the number of vertices of $\mathbf{A}$ with latent positions $p$ and $q$. The adjacency spectral embedding meanwhile satisfies
\begin{gather}
\label{eq:er-p-q-ase1}
 \sqrt{n}(\hat{X}_i  - p) \overset{\mathrm{d}}{\longrightarrow} \mathcal{N}\Bigl(0, \tfrac{\pi p^4(1 - p^2) + (1 - \pi) pq^3(1 - pq)}{(\pi p^2 + (1 - \pi)q^2)^2}\Bigr) \,\, \text{if $X_i = p$},\\
 \label{eq:er-p-q-ase2}
\sqrt{n}(\hat{X}_i  - q) \overset{\mathrm{d}}{\longrightarrow} \mathcal{N}\Bigl(0, \tfrac{\pi p^3q(1 - pq) + (1 - \pi) q^4(1 - q^2)}{(\pi p^2 + (1 - \pi)q^2)^2}\Bigr) \,\, \text{if $X_i = q$}.
\end{gather}
\end{remark}

\begin{remark}
We note that the quantity $n_k$ appears in Eq.~\eqref{eq:Xtilde_clt} and Eq.~\eqref{eq:Xtilde_clt_vanishing}. Replacing $n_k$ by $n \pi_k$ in Eq.~\eqref{eq:Xtilde_clt} and Eq.~\eqref{eq:Xtilde_clt_vanishing} is, however, not straightforward. For example, for the two-block stochastic blockmodel considered in Eq.~\eqref{eq:er-p-q-lse1}, letting $\zeta = \tfrac{np}{\sqrt{n_1 p^2 + n_2  pq}} - \tfrac{np}{\sqrt{n \pi p^2 + n (1 - \pi) p q}}$ we have
\begin{equation*}
\begin{split}
 \zeta &=  \tfrac{np(\sqrt{n \pi p^2 + n (1 - \pi) p q} - \sqrt{n_1 p^2 + n_2 p q})}{\sqrt{n_1 p^2 + n_2  pq}\sqrt{n \pi p^2 + n (1 - \pi) p q}} \\ &=   \tfrac{ np(n \pi p^2 + n (1 - \pi) p q - n_1 p^2 - n_2 p q)}{(\sqrt{n \pi p^2 + n (1 - \pi) p q} + \sqrt{n_1 p^2 + n_2 p q})\sqrt{n_1 p^2 + n_2  pq}\sqrt{n \pi p^2 + n (1 - \pi) p q}} \\
 &= \tfrac{ np(n \pi - n_1)(p^2 - pq)}{(\sqrt{n \pi p^2 + n (1 - \pi) p q} + \sqrt{n_1 p^2 + n_2 p q})\sqrt{n_1 p^2 + n_2  pq}\sqrt{n \pi p^2 + n (1 - \pi) p q}}.
\end{split}
\end{equation*}
By the strong law of large numbers and Slutsky's theorem, we have
$$ \tfrac{n^{3/2}}{(\sqrt{n \pi p^2 + n (1 - \pi) p q} + \sqrt{n_1 p^2 + n_2 p q})\sqrt{n_1 p^2 + n_2  pq}\sqrt{n \pi p^2 + n (1 - \pi) p q}} \overset{\mathrm{a.s.}}{\longrightarrow} \tfrac{1}{2(p^2 + pq)^{3/2}}. $$
We note that, as the $n_k$ are assumed to be random variables, i.e., we are not conditioning on the block sizes, by the central limit theorem we have
$$ \tfrac{1}{\sqrt{n}} (n \pi- n_1) \overset{\mathrm{d}}{\longrightarrow} \mathcal{N}(0, \pi (1 - \pi)).$$
Therefore, by Slutsky's theorem, we have
$$\zeta = \tfrac{np}{\sqrt{n_1 p^2 + n_2  pq}} - \tfrac{np}{\sqrt{n \pi p^2 + n (1 - \pi) p q}} \overset{\mathrm{d}}{\longrightarrow} \mathcal{N}\bigl(0, \tfrac{\pi (1 - \pi) p(p - q)^2}{4(p + q)^3}\bigr). $$
To replace $n_k$ by $n \pi_k$ in Eq.~\eqref{eq:Xtilde_clt} and Eq.~\eqref{eq:Xtilde_clt_vanishing}, we thus need to include the random term $\zeta$. While we surmise that Eq.~\eqref{eq:Xtilde_clt} and Eq.~\eqref{eq:Xtilde_clt_vanishing} can be adapt to account for this randomness in $n_k$, we shall not do so in this paper. 
\end{remark}

\subsection{Proofs sketch for Theorem~\ref{THM:LSE} and Theorem~\ref{THM:NORMALITY-LSE}}
We present in this subsection a sketch of the main ideas in the proofs of Theorem~\ref{THM:LSE} and Theorem~\ref{THM:NORMALITY-LSE}; the detailed proofs are given in Section~\ref{sec:lse} of the appendix. We start with the motivation behind Eq.~\eqref{eq:LSE-main1}. Given $\tilde{\mathbf{X}}_n$, the entries of the right hand side of Eq.~\eqref{eq:LSE-main1}, except for the term $\mathbf{R}_n$, can be expressed explicitly in terms of linear combinations of the entries $a_{ij} - p_{ij}$ of $\mathbf{A}_n - \mathbf{P}_n$. This is in contrast with the left hand side of Eq.~\eqref{eq:LSE-main1} which depends on the quantities $\tilde{\mathbf{U}}_{\mathbf{A}}$ and $\tilde{\mathbf{S}}_{\mathbf{A}}$ (recall Definition~\ref{defn:ase-lse}); since the quantities $\tilde{\mathbf{U}}_{\mathbf{A}}$ and $\tilde{\mathbf{S}}_{\mathbf{A}}$ cannot be express explicitly in terms of the entries of $\mathbf{A}_n$ and $\mathbf{P}_n$, we conclude that the right hand side of Eq.~\eqref{eq:LSE-main1} is simpler to analyze. From Eq.~\eqref{eq:LSE-main1}, the squared Frobenius norm $n \rho_n \|\breve{\mathbf{X}}_n \mathbf{W}_n - \tilde{\mathbf{X}}_n\|_{F}^{2}$ is 
$$n \rho_n \|\mathbf{T}_n^{-1/2} (\mathbf{A}_n - \mathbf{P}_n) \mathbf{T}_n^{-1/2}
  \tilde{\mathbf{X}}_n (\tilde{\mathbf{X}}_n^{\top} \tilde{\mathbf{X}}_n)^{-1} + \tfrac{1}{2}(\mathbf{I} - \mathbf{D}_n \mathbf{T}_n^{-1}) \tilde{\mathbf{X}}_n\|_{F}^{2} + O_{\mathbb{P}}((n \rho_n)^{-1/2}). 
$$
Then conditional on $\mathbf{P}_n$, the above expression is, up to the term of order $O_{\mathbb{P}}((n \rho_n)^{-1/2})$, a function of the {\em independent} random variables $\{a_{ij} - p_{ij}\}_{i < j}$. We can then apply concentration inequalities such as those in \cite{boucheron2003} to show that the squared Frobenius norm $n \rho_n \|\breve{\mathbf{X}}_n \mathbf{W}_n - \tilde{\mathbf{X}}_n\|_{F}^{2}$ is, conditional on $\mathbf{P}_n$, concentrated around its expectation. Here the expectation is taken with respect to the random entries of $\mathbf{A}_n$. Eq.~\eqref{eq:4} and Eq.~\eqref{eq:14} then follows by direct evaluation of this expectation, for the case when $\rho_n \equiv 1$ and for when $\rho_n \rightarrow 0$, respectively.

Once Eq.~\eqref{eq:LSE-main1} is established, we can derive Theorem~\ref{THM:NORMALITY-LSE} as follows. Let $\xi_i$ denotes the $i$-th row of $n \rho_n^{1/2} (\mathbf{W}_n \breve{\mathbf{X}}_n - \tilde{\mathbf{X}}_n)$ and let $r_i$ denotes the $i$-th row of $\mathbf{R}_n$. Eq.~\eqref{eq:LSE-main1} then implies 
\begin{equation*}
\begin{split} 
\xi_i &= (\tilde{\mathbf{X}}_n^{\top} \tilde{\mathbf{X}}_n)^{-1} \frac{n \rho_n^{1/2}}{\sqrt{t_i}} \Bigl( \sum_{j} \frac{a_{ij} -
     p_{ij}}{\sqrt{t_j}} \tilde{X}_j \Bigr) + \frac{n \rho_n^{1/2} (t_i - d_i)}{2t_i} \tilde{X}_i + n \rho_n^{1/2} r_i \\
     &= (\tilde{\mathbf{X}}_n^{\top} \tilde{\mathbf{X}}_n)^{-1}
   \frac{\sqrt{n \rho_n}}{\sqrt{t_i}} \Bigl( \sum_{j
   } \frac{\sqrt{n \rho_n} (a_{ij} - p_{ij})X_j}{ t_j} \Bigr) - \frac{n \rho_n X_i}{2 t_i^{3/2}} \sum_{j 
   } (a_{ij} - p_{ij}) + n \rho_n^{1/2} r_i \\
   &= 
   \frac{\sqrt{n \rho_n}}{\sqrt{t_i}} \sum_{j
   } \frac{(a_{ij} - p_{ij})}{\sqrt{n \rho_n} } \Bigl(\frac{(\tilde{\mathbf{X}}_n^{\top} \tilde{\mathbf{X}}_n)^{-1} X_j}{ t_j/(n \rho_n)} - \frac{X_i}{2 t_i/(n \rho_n)} \Bigr) + n \rho_n^{1/2} r_i
     \end{split}
     \end{equation*}
     We then show that $n \rho_n^{1/2} r_i \overset{\mathrm{d}}{\rightarrow} 0$. Indeed, there are $n$ rows in $\mathbf{R}_n$ and $\|\mathbf{R}_n\|_{F} = O((n \rho_n)^{-1})$; hence, on average, for each index $i$, $\|r_i\|^{2} = O_{\mathbb{P}}(n^{-3} \rho_n^{-2})$.
     Furthermore, $t_i/(n \rho_n) = \sum_{j} X_i^{\top} X_j/n \overset{\mathrm{a.s.}}{\longrightarrow} X_i^{\top} \mu$ as $n \rightarrow \infty$. Finally, $\tilde{\mathbf{X}}_n^{\top} \tilde{\mathbf{X}}_n = \sum_{i} \bigl(X_i X_i^{\top}/(\sum_{j} X_i^{\top} X_j)\bigr)$ which, as we show in Section~\ref{sec:lse}, converges to $\tilde{\Delta} = \mathbb{E}\bigl[\tfrac{X_1 X_1^{\top}}{X_1^{\top} \mu}\bigr]$ as $n \rightarrow \infty$. We therefore have, after additional manipulations, that
     \begin{equation*}
     \begin{split}
      \xi_i &=   \frac{\sqrt{n \rho_n}}{\sqrt{t_i}} \sum_{j
   } \frac{(a_{ij} - p_{ij})}{\sqrt{n \rho_n} } \Bigl(\frac{\tilde{\Delta}^{-1} X_j}{ X_j^{\top} \mu} - \frac{X_i}{2 X_i^{\top} \mu} \Bigr) + o_{\mathbb{P}}(1). \\
   &= \frac{\sqrt{n \rho_n}}{\sqrt{t_i}} \sum_{j
   } \frac{(a_{ij} - \rho_n X_i^{\top} X_j)}{\sqrt{n \rho_n} } \Bigl(\frac{\tilde{\Delta}^{-1} X_j}{ X_j^{\top} \mu} - \frac{X_i}{2 X_i^{\top} \mu} \Bigr) + o_{\mathbb{P}}(1).
   \end{split}
   \end{equation*}
    Then conditioning on $X_i = x$, the above expression for $\xi_i$ is roughly a sum of independent and identically distributed mean $0$ random variables. The multivariate central limit theorem can then be applied to the above expression for $\xi_i$, thereby yielding Theorem~\ref{THM:NORMALITY-LSE}. 

We now sketch the derivation of Eq.~\eqref{eq:LSE-main1}. For simplicity, we ignore the subscript $n$ in the matrices $\mathbf{A}_n$, $\mathbf{X}_n$, $\mathbf{P}_n$ and related matrices. First, consider the following expression.
\begin{equation*}
\begin{split}
\tilde{\mathbf{U}}_{\mathbf{A}} \tilde{\mathbf{S}}_{\mathbf{A}}^{1/2} - \tilde{\mathbf{U}}_{\mathbf{P}} \tilde{\mathbf{S}}_{\mathbf{P}}^{1/2} \tilde{\mathbf{U}}_{\mathbf{P}}^{\top} \tilde{\mathbf{U}}_{\mathbf{A}} & = \mathcal{L}(\mathbf{A}) \tilde{\mathbf{U}}_{\mathbf{A}} \tilde{\mathbf{S}}_{\mathbf{A}}^{-1/2} - \mathcal{L}(\mathbf{P}) \tilde{\mathbf{U}}_{\mathbf{P}} \tilde{\mathbf{S}}_{\mathbf{P}}^{-1/2} \tilde{\mathbf{U}}_{\mathbf{P}}^{\top} \tilde{\mathbf{U}}_{\mathbf{A}} \\
&= \mathcal{L}(\mathbf{A}) \tilde{\mathbf{U}}_{\mathbf{A}} \tilde{\mathbf{U}}_{\mathbf{A}}^{\top} \tilde{\mathbf{U}}_{\mathbf{A}} \tilde{\mathbf{S}}_{\mathbf{A}}^{-1/2} - \mathcal{L}(\mathbf{P}) \tilde{\mathbf{U}}_{\mathbf{P}} \tilde{\mathbf{S}}_{\mathbf{P}}^{-1/2} \tilde{\mathbf{U}}_{\mathbf{P}}^{\top} \tilde{\mathbf{U}}_{\mathbf{A}}
  \end{split}
\end{equation*}
Now $\mathcal{L}(\mathbf{A})$ is ``concentrated'' around  $\mathcal{L}(\mathbf{P})$, i.e., $\|\mathcal{L}(\mathbf{A}) - \mathcal{L}(\mathbf{P})\| = O_{\mathbb{P}}((n \rho_n)^{-1/2})$ (see Theorem~2 in \cite{lu13:_spect}). Since $\|\mathcal{L}(\mathbf{P})\| = \Theta(1)$ and the non-zero eigenvalues of $\mathcal{L}(\mathbf{P})$ are all of order $\Theta(1)$, this implies, by the Davis-Kahan theorem, that the eigenspace spanned by the $d$ largest eigenvalues of $\mathcal{L}(\mathbf{A})$ is ``close'' to that spanned by the $d$ largest eigenvalues of $\mathcal{L}(\mathbf{P})$. More precisely, $\tilde{\mathbf{U}}_{\mathbf{A}} \tilde{\mathbf{U}}_{\mathbf{A}}^{\top} = \tilde{\mathbf{U}}_{\mathbf{P}} \tilde{\mathbf{U}}_{\mathbf{P}}^{\top} + O_{\mathbb{P}}((n \rho_n)^{-1/2})$ and 
\begin{equation*}
\begin{split}
 \tilde{\mathbf{U}}_{\mathbf{A}} \tilde{\mathbf{S}}_{\mathbf{A}}^{1/2} - \tilde{\mathbf{U}}_{\mathbf{P}} \tilde{\mathbf{S}}_{\mathbf{P}}^{1/2} \tilde{\mathbf{U}}_{\mathbf{P}}^{\top} \tilde{\mathbf{U}}_{\mathbf{A}} = \mathcal{L}(\mathbf{A}) \tilde{\mathbf{U}}_{\mathbf{P}} \tilde{\mathbf{U}}_{\mathbf{P}}^{\top} \tilde{\mathbf{U}}_{\mathbf{A}} \tilde{\mathbf{S}}_{\mathbf{A}}^{-1/2} & - \mathcal{L}(\mathbf{P}) \tilde{\mathbf{U}}_{\mathbf{P}} \tilde{\mathbf{S}}_{\mathbf{P}}^{-1/2} \tilde{\mathbf{U}}_{\mathbf{P}}^{\top} \tilde{\mathbf{U}}_{\mathbf{A}} \\ &+ O_{\mathbb{P}}((n \rho_n)^{-1}).
 \end{split}
 \end{equation*}
We then consider the terms $\tilde{\mathbf{S}}_{\mathbf{P}}^{-1/2} \tilde{\mathbf{U}}_{\mathbf{P}}^{\top} \tilde{\mathbf{U}}_{\mathbf{A}}$ and $\tilde{\mathbf{U}}_{\mathbf{P}}^{\top} \tilde{\mathbf{U}}_{\mathbf{A}} \tilde{\mathbf{S}}_{\mathbf{A}}^{-1/2}$. Since $\tilde{\mathbf{U}}_{\mathbf{P}}$ and $\tilde{\mathbf{U}}_{\mathbf{A}}$ both have orthonormal columns, $\tilde{\mathbf{U}}_{\mathbf{A}} \tilde{\mathbf{U}}_{\mathbf{A}}^{\top} = \tilde{\mathbf{U}}_{\mathbf{P}} \tilde{\mathbf{U}}_{\mathbf{P}}^{\top} + O_{\mathbb{P}}((n \rho_n)^{-1/2})$ implies that there exists an orthogonal matrix $\mathbf{W}^{*}$ such that $\tilde{\mathbf{U}}_{\mathbf{P}}^{\top} \tilde{\mathbf{U}}_{\mathbf{A}} = \mathbf{W}^{*} + O_{\mathbb{P}}((n \rho_n)^{-1})$
(see Proposition~\ref{prop:2}). 
Furthermore, $\mathbf{W}^{*}$ satisfies an important property, namely that $\mathbf{W}^{*} \tilde{\mathbf{S}}_{\mathbf{A}}^{-1/2} - \tilde{\mathbf{S}}_{\mathbf{P}}^{-1/2} \mathbf{W}^{*} = O_{\mathbb{P}}((n \rho_n)^{-1})$.
(see Lemma~\ref{lem:2}). 
We can thus juxtapose $\tilde{\mathbf{U}}_{\mathbf{P}}^{\top} \tilde{\mathbf{U}}_{\mathbf{A}}$ and $\tilde{\mathbf{S}}_{\mathbf{A}}^{-1/2}$ in the above expression and replace $\tilde{\mathbf{U}}_{\mathbf{P}}^{\top} \tilde{\mathbf{U}}_{\mathbf{A}}$ by the orthogonal matrix $\mathbf{W}^{*}$, thereby yielding
$$ \tilde{\mathbf{U}}_{\mathbf{A}} \tilde{\mathbf{S}}_{\mathbf{A}}^{1/2} - \tilde{\mathbf{U}}_{\mathbf{P}} \tilde{\mathbf{S}}_{\mathbf{P}}^{1/2} \mathbf{W}^{*} = (\mathcal{L}(\mathbf{A}) - \mathcal{L}(\mathbf{P})) \tilde{\mathbf{U}}_{\mathbf{P}} \tilde{\mathbf{S}}_{\mathbf{P}}^{-1/2} \mathbf{W}^{*} + O_{\mathbb{P}}((n \rho_n)^{-1}).$$
As $\tilde{\mathbf{X}} \tilde{\mathbf{X}}^{\top} = \mathcal{L}(\mathbf{P}) = \tilde{\mathbf{U}}_{\mathbf{P}} \tilde{\mathbf{S}}_{\mathbf{P}}^{1/2} \tilde{\mathbf{U}}_{\mathbf{P}}^{\top}$, we have $\tilde{\mathbf{X}} = \tilde{\mathbf{U}}_{\mathbf{P}} \tilde{\mathbf{S}}_{\mathbf{P}} \tilde{\mathbf{W}}$ for some orthogonal matrix $\tilde{\mathbf{W}}$. 
Therefore, 
\begin{equation*}
\begin{split}
\tilde{\mathbf{U}}_{\mathbf{A}} \tilde{\mathbf{S}}_{\mathbf{A}}^{1/2} - \tilde{\mathbf{X}} \tilde{\mathbf{W}}^{\top} \mathbf{W}^{*} & = 
(\mathcal{L}(\mathbf{A}) - \mathcal{L}(\mathbf{P})) \tilde{\mathbf{U}}_{\mathbf{P}} \tilde{\mathbf{S}}_{\mathbf{P}}^{-1/2} \mathbf{W}^{*} + O_{\mathbb{P}}((n \rho_n)^{-1})\\
& = (\mathcal{L}(\mathbf{A}) - \mathcal{L}(\mathbf{P})) \tilde{\mathbf{U}}_{\mathbf{P}} \tilde{\mathbf{S}}_{\mathbf{P}}^{1/2} \tilde{\mathbf{W}} \tilde{\mathbf{W}}^{\top} \tilde{\mathbf{S}}_{\mathbf{P}}^{-1} \tilde{\mathbf{W}} \tilde{\mathbf{W}}^{\top} \mathbf{W}^{*} + O_{\mathbb{P}}((n \rho_n)^{-1})\\
& = (\mathcal{L}(\mathbf{A}) - \mathcal{L}(\mathbf{P})) \tilde{\mathbf{X}} (\tilde{\mathbf{X}}^{\top} \tilde{\mathbf{X}})^{-1} \tilde{\mathbf{W}}^{\top} \mathbf{W}^{*} + O_{\mathbb{P}}((n \rho_n)^{-1}).
\end{split}
\end{equation*}
Equivalently, 
\begin{equation} 
\label{eq:sketch-1}
\tilde{\mathbf{U}}_{\mathbf{A}} \tilde{\mathbf{S}}_{\mathbf{A}}^{1/2} (\mathbf{W}^{*})^{\top} \tilde{\mathbf{W}} - \tilde{\mathbf{X}} = (\mathcal{L}(\mathbf{A}) - \mathcal{L}(\mathbf{P})) \tilde{\mathbf{X}} (\tilde{\mathbf{X}}^{\top} \tilde{\mathbf{X}})^{-1} + O_{\mathbb{P}}((n \rho_n)^{-1}).
\end{equation}
The right hand side of Eq.~\eqref{eq:sketch-1} can be written explicitly in terms of the entries of $\mathbf{A}$. However, since $\mathcal{L}(\mathbf{A}) = \mathbf{D}^{-1/2} \mathbf{A} \mathbf{D}^{-1/2}$, the entries of the right hand side of Eq.~\eqref{eq:sketch-1} are not linear/affine combinations of the entries of $\mathbf{A}$. Nevertheless, by a Taylor-series expansion of the entries of $\mathbf{D}^{-1/2}$, we have $\mathbf{D}^{-1/2} = \mathbf{T}^{-1/2} + \tfrac{1}{2} \mathbf{T}^{-3/2} (\mathbf{T} - \mathbf{D}) + O_{\mathbb{P}}((n \rho_n)^{-3/2})$. 
Substituting this into Eq.~\eqref{eq:sketch-1} followed by further simplifications yield Eq.~\eqref{eq:LSE-main1}.

\section{Subsequent Inference}
\label{sec:simulations}
In this section we demonstrate how the results of Section~\ref{sec:limit-result-adjacency} and Section~\ref{sec:limit-result-Laplacian} provide insights into subsquent inference. 
We first consider graphs generated according to a stochastic blockmodel with parameters
\begin{equation}
\label{eq:example1}
 \mathbf{B} = \begin{bmatrix} 0.42 & 0.42 \\ 0.42 & 0.5 \end{bmatrix}; \quad \text{and} \quad \pi = (0.6,0.4).
\end{equation}
We sample an adjacency matrix $\mathbf{A}$ for graphs on $n$ vertices from the above model for various choices of $n$. For each adjacency matrix $\mathbf{A}$, 
we compute the normalized Laplacian embedding of $\mathbf{A}$. Figure~\ref{fig:clusplot} presents examples of the scatter plots for these embeddings for 
$n = 1000$, $2000$ and $4000$. The points in the scatter plots are colored according to the block membership of the corresponding vertices in the blockmodel. For each block, we also plot the ellipses showing the empirical (dashed lines) and theoretical (solid lines) $95\%$ level curves for the distribution of $\breve{X}_i$. The theoretical level curves are as specified in Theorem~\ref{THM:NORMALITY-LSE}. 

\begin{figure}[!htp]
  \centering
  \subfloat[$n = 1000$]{
    \includegraphics[width=3.5cm]{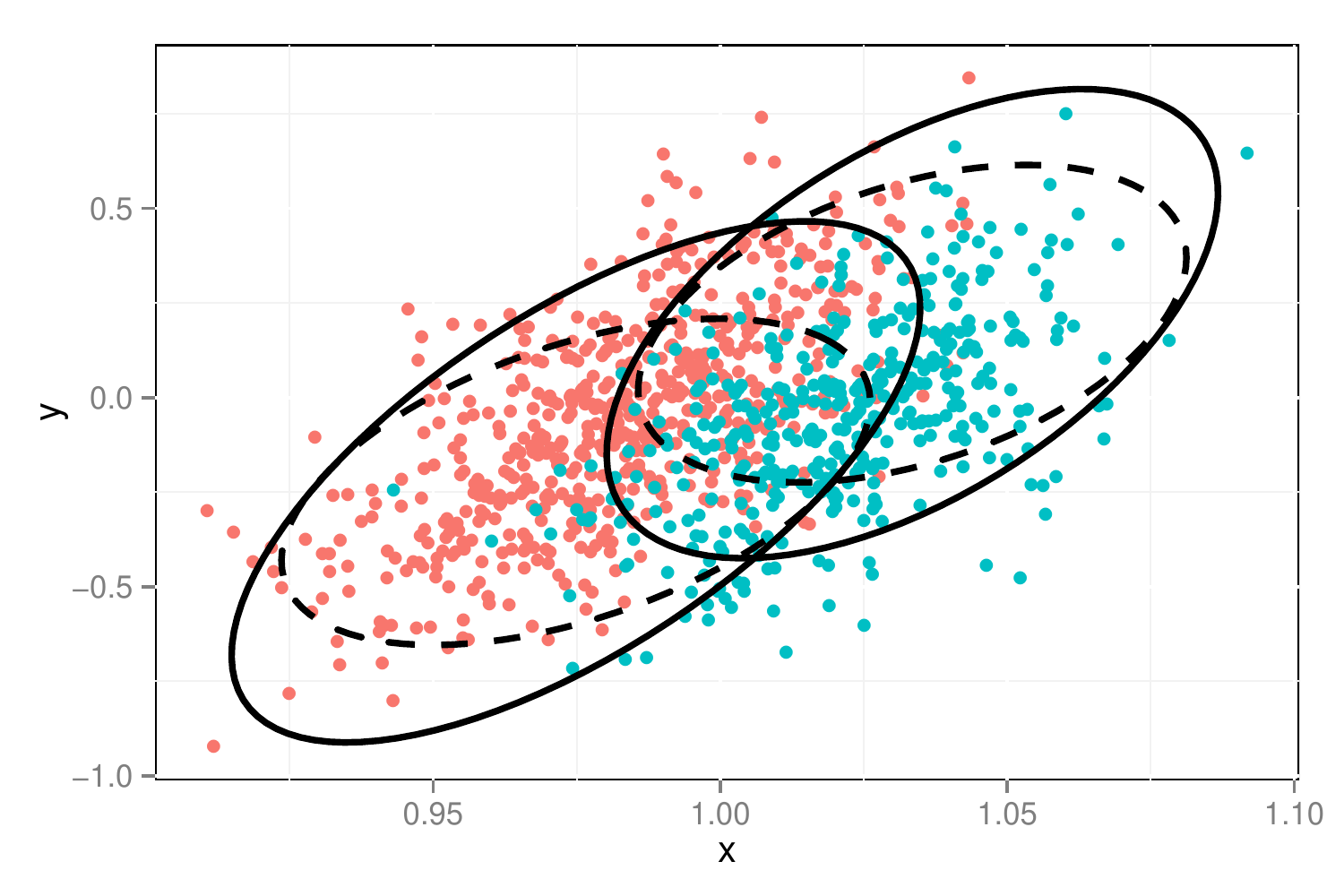}
  }
  \hfil
  \subfloat[$n = 2000$]{
    \includegraphics[width=3.5cm]{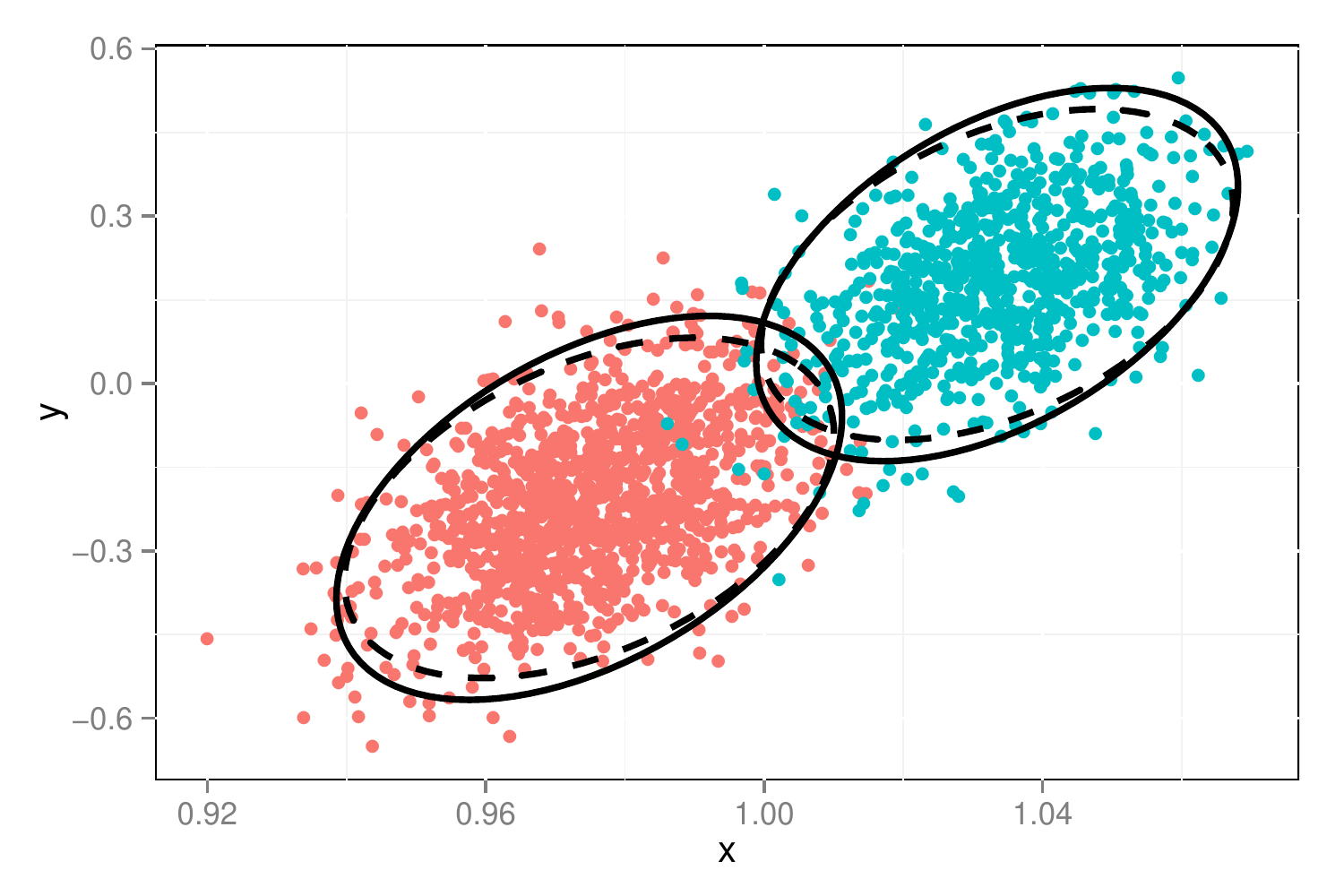}
  }
  \hfil
  \subfloat[$n = 4000$]{
    \includegraphics[width=3.5cm]{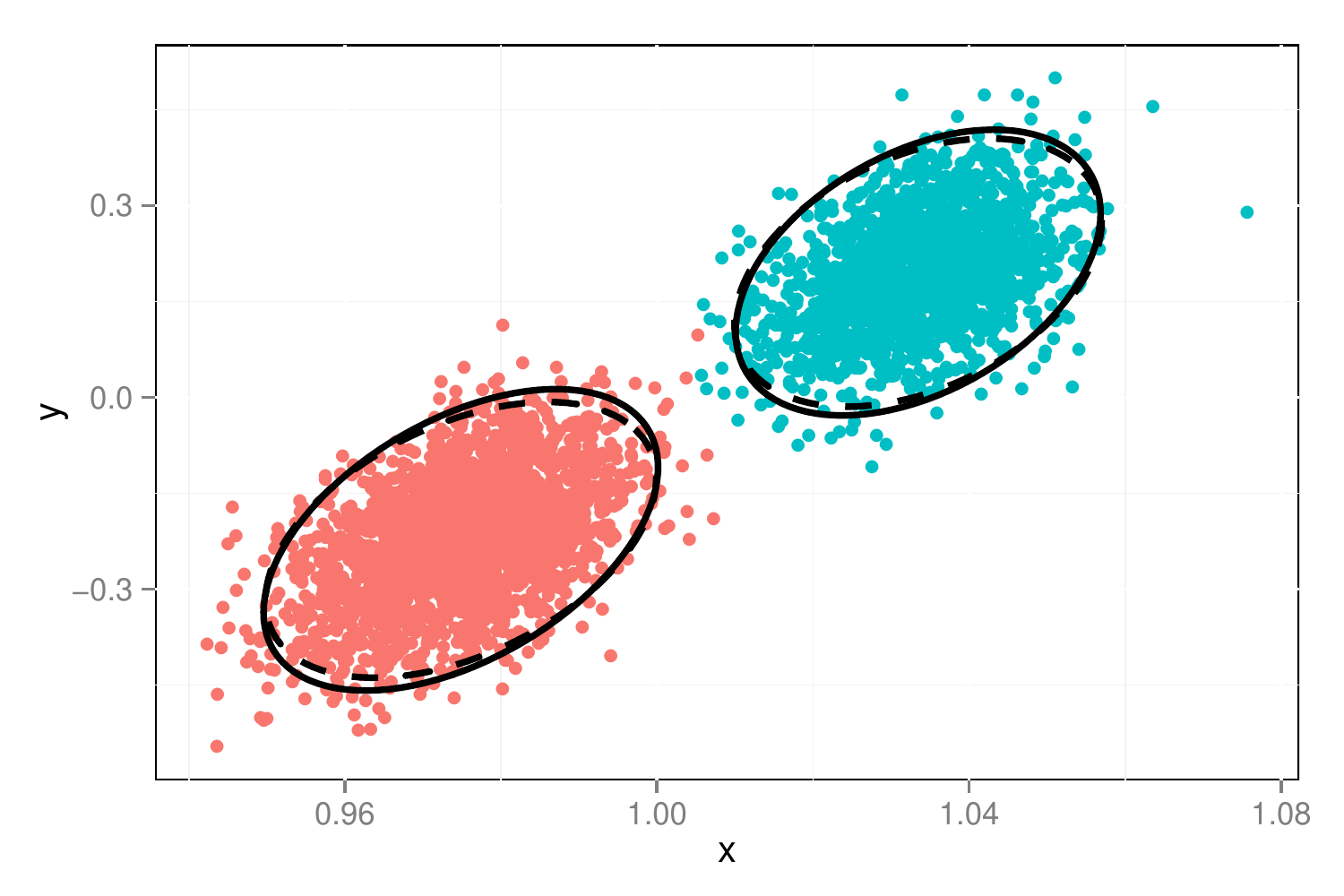}
    \label{fig5:subfig_scan}
  }
  \caption{Plot of the estimated latent positions in a two-block
    stochastic blockmodel graph on $n$ vertices. The
    points are colored according to the blockmembership of the
    corresponding vertices. Dashed ellipses give the 95\% level curves
    for the {\em empirical} distributions.
 Solid ellipses give the 95\% {\em theoretical} level curves for the distributions as specified by Theorem~\ref{THM:NORMALITY-LSE}. }
    \label{fig:clusplot}
\end{figure}
We next investigate the implication of the multivariate normal distribution from Theorem~\ref{THM:NORMALITY-LSE} on subsequent inference.
Spectral clustering refers to a large class of techniques used in partitionining data points into clusters that proceed by first performing 
a truncated eigendecomposition of a similarity matrix between the data points to obtain a low-dimensional Euclidean representation of these data points followed by clustering of the data points in this low-dimensional representation; see \cite{von2007tutorial} for a comprehensive introduction. The {\em normalized cuts} algorithm of \cite{shi_malik} is a popular and widely-used instance of spectral clustering where the similarity matrix is a normalized Laplacian matrix and clustering is done using the $K$-means algorithm. 

\begin{figure}[htbp!]
\centering
\includegraphics[width=0.65\textwidth]{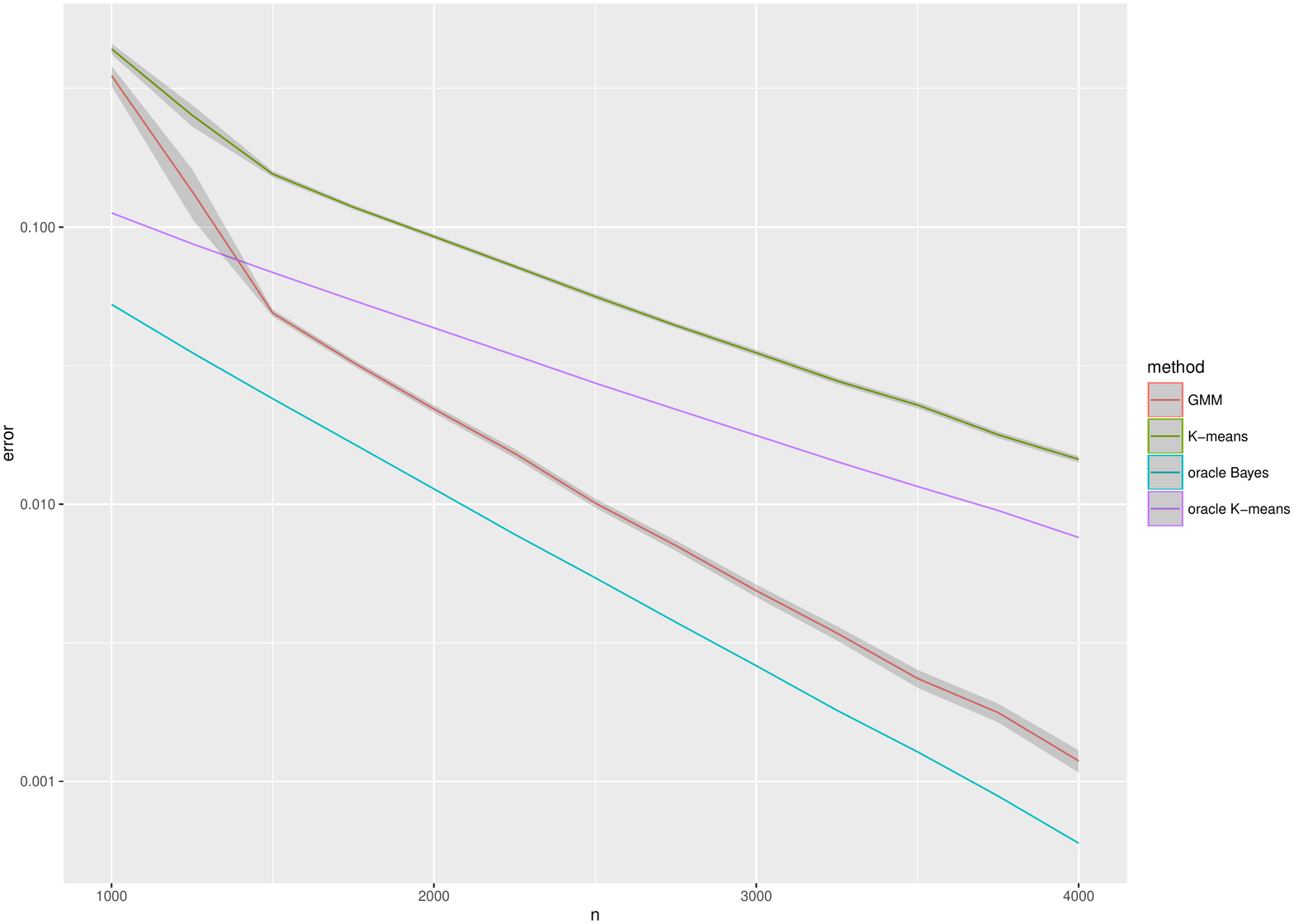}
\caption{Comparison of clustering error rates for Gaussian mixture model (GMM) clustering, $K$-means clustering, linear classifier, and Bayes-optimal classifier. The error rate for each $n \in \{1000, 1250, 1500, \dots, 4000\}$ was obtained by averaging 100 Monte Carlo iterations and are plotted on a $\log_{10}$ scale. The plot indicates that the assumption of a mixture of multivariate normals can yield significant improvement in the accuracy of the spectral clustering procedure.}
\label{fig:GMM-kmeans}
\end{figure}

It was shown in \cite{rohe2011spectral} that the normalized cuts algorithm, i.e., the normalized Laplacian embedding followed by $K$-means, is consistent for estimating the block memberships of stochastic blockmodels graphs. The result of Corollary~\ref{cor:lse_normality_sbm}, however, suggests that $K$-means clustering is suboptimal unless the covariance matrices of the estimated latent positions for the blocks are spherical. We illustrate this by generating sequences of stochastic blockmodel graphs on $n$ vertices with parameters as given in Eq.~\eqref{eq:example1} where $n \in \{1000, 1250, 1500, \dots, 4000\}$. For each graph, we embed its normalized Laplacian matrix into $\mathbb{R}^{2}$ and cluster the embedded vertices via either $K$-means or the MCLUST Gaussian mixture model-based clustering algorithm \cite{fraley99:_mclus}. We then measure the error rate of the clustering solution. The error rates, averaged over $100$ replicates of the experiment, are presented on log-scale in Figure~\ref{fig:GMM-kmeans}. We see that the Gaussian mixture model-based clustering does yield significant improvement over $K$-means clustering. For further comparison, we plot the Bayes-optimal error rate and that of a linear classifier which assign an embedded point to the closest theoretical centroid. The error rate of the linear classifier is computed under the assumption that the rows of the Laplacian spectral embedding are indeed multivariate normal with known covariance matrices and centered around the centroid of the respective blocks; this error rate serves as a lower-bound for that of K-means clustering. 

\subsection{Comparison of ASE and LSE via within-class covariances}
We now discuss a comparison of the use of adjacency spectral embedding and Laplacian spectral embedding for subsequent inference. We consider as our subsequent inference task the problem of recovering the block assignments in stochastic blockmodel graphs.
Our first metric of comparison is the notion of within-block variance for each block of the stochastic blockmodel, following the work of \cite{bickel_sarkar_2013}. We partially extend the results of \cite{bickel_sarkar_2013} for two-block stochastic blockmodels to $K$-block stochastic blockmodels with positive semidefinite block probablity matrices.
However, while the collection of within-block variances is a meaningful surrogate for the performance of our subsequent inference task, we argue that it is not the ``right'' metric as it captures only the trace of the block-conditional covariance matrices and not the form of the block-conditional covariance matrices.
That is to say, the use of the within-block variances as a surrogate measure is similar to the oracle $K$-means lower bound in Figure~\ref{fig:GMM-kmeans}. A more appropriate surrogate is the collection of pairwise Chernoff informations between the block-conditional multivariate normals, which behave similarly to the oracle Bayes lower bound in Figure~\ref{fig:GMM-kmeans}. The discussion of Chernoff information is postponed to the next subsection. 

\begin{definition}[Within-block variances]
\label{defn:metric-bickel-sarkar}
Let $(\mathbf{X}, \mathbf{A}) \sim \mathrm{RDPG}(F)$ with sparsity factor $\rho_n$ where $F = \sum_{k} \pi_{k} \delta_{\nu_k}$ is a mixture of $K$ point masses at $\nu_1, \nu_2, \dots, \nu_K \in \mathbb{R}^{d}$ and $\delta_{\nu_k}$ denotes the Dirac delta function. Given $\mathbf{A}$, let $C_k$ for $k \in \{1,2,\dots,K\}$ denote the set of vertices of $\mathbf{A}$ assigned to block $k$. Recall the definitions of $\mathbf{U}_{\mathbf{A}}$ and $\tilde{\mathbf{U}}_{\mathbf{A}}$ in Definition~\ref{defn:ase-lse}, i.e., $\mathbf{U}_{\mathbf{A}}$ and $\tilde{\mathbf{U}}_{\mathbf{A}}$ are the $n \times d$ matrices containing the $d$ largest eigenvectors of the adjacency matrix and the Laplacian matrix, respectively. For any index $i$, let $\mathbf{U}_{\mathbf{A}}(i,\colon)$ and $\tilde{\mathbf{U}}_{\mathbf{A}}(i, \colon)$ denote the $i$-th row of $\mathbf{U}_{\mathbf{A}}$ and $\tilde{\mathbf{U}}_{\mathbf{A}}(i, \colon)$, respectively. Then for any $k,l \in \{1,2,\dots,K\}$, the ASE variance between block $k$ and block $l$ is defined as
\begin{equation}
\label{eq:within-block-variance-ASE}
\hat{d}_{kl} = \hat{d}_{kl}(\mathbf{A}) = \frac{1}{|C_k|} \sum_{i \in C_k} \|\mathbf{U}_{\mathbf{A}}(i,\colon) - \hat{\mu}_l \|^{2}; \quad \hat{\mu}_l = \frac{1}{|C_l|} \sum_{j \in C_l} \mathbf{U}_{\mathbf{A}}(j, \colon).
\end{equation} 
Similarly, the LSE variance between block $k$ and block $l$ is 
\begin{equation}
\label{eq:within-block-variance-LSE}
\tilde{d}_{kl} = \tilde{d}_{kl}(\mathbf{A}) = \frac{1}{|C_k|} \sum_{i \in C_k} \|\tilde{\mathbf{U}}_{\mathbf{A}}(i,\colon) - \tilde{\mu}_l \|^{2}; \quad \tilde{\mu}_l = \frac{1}{|C_l|} \sum_{j \in C_k} \tilde{\mathbf{U}}_{\mathbf{A}}(j, \colon).
\end{equation} 
When $k = l$, $\hat{d}_{kk}$ and $\tilde{d}_{kk}$ are refered to as the ASE within-block variance for block $k$ and the LSE within-block variance for block $k$, respectively. 
\end{definition}

We then have the following large-sample limit results for $\hat{d}_{kl}$ and $\tilde{d}_{kl}$. Their proofs are similar to those of Theorem~\ref{THM:ASE} and Theorem~\ref{THM:LSE} and therefore will be omitted. Nevertheless, we verify in Section~\ref{sec:within-variance-appendix} of the appendix that Theorem~\ref{THM:between-ASE} and Theorem~\ref{THM:between-LSE} are indeed generalizations of Theorem~3.1 and Theorem~3.2 from \cite{bickel_sarkar_2013}. We emphasize that neither Theorem~\ref{THM:between-ASE} nor Theorem~\ref{THM:between-LSE} assume distinct eigenvalues of the matrix $\mathbf{X} \mathbf{X}^{\top}$ or $\mathcal{L}(\mathbf{X} \mathbf{X}^{\top})$; distinct eigenvalues is a necessary assumption used in the proofs of Theorem~3.1 and Theorem~3.2 in \cite{bickel_sarkar_2013} (see Section~8 of the cited paper). 
\begin{theorem}
\label{THM:between-ASE}
Assume the setting and notations of Theorem~\ref{THM:ASE} and suppose furthermore that $F = \sum_{k} \pi_{k} \delta_{\nu_k}$ is a mixture of $K$ distinct point masses at $\nu_1, \nu_2, \dots, \nu_K \in \mathbb{R}^{d}$. Let $\mathbf{U}_{\mathbf{P}_n}$ denote the $n \times d$ matrix whose columns are the orthonormal eigenvectors corresponding to the non-zero eigenvalues of the matrix $\mathbf{P}_n = \rho_n \mathbf{X}_n \mathbf{X}_n^{\top}$. For any $k \in \{1,2,\dots,K\}$, let $\mathbf{S}_{k}$ be the $n \times n$ diagonal matrix with diagonal entries $(s_{k}(1), s_{k}(2), \dots, s_{k}(n))$ such that $s_{k}(i) = 1$ if $X_i = \nu_k$ and $s_{k}(i) = 0$ otherwise. We then have, for any $k \in \{1,2,\dots,K\}$
\begin{equation}
\begin{split}
n^2 \hat{d}_{kk} &= \frac{n^2}{|C_k|} \|\mathbf{S}_k (\mathbf{U}_{\mathbf{A}_n} \mathbf{W}_n - \mathbf{U}_{\mathbf{P}_n}) \|_{F}^{2} + o_{\mathbb{P}}(1) 
\\ &= \frac{n^2}{|C_k|} \|\mathbf{S}_{k} (\mathbf{A}_n - \mathbf{P}_n) \mathbf{X}_n (\mathbf{X}_n^{\top} \mathbf{X}_n)^{-3/2} \|_{F}^{2} + o_{\mathbb{P}}(1). 
\end{split}
\end{equation}
Therefore, if $\rho_n \equiv 1$, then for any $k \in \{1,2,\dots,K\}$
\begin{equation}
\label{eq:d_kk1}
n^2 \hat{d}_{kk} \overset{\mathrm{a.s.}}{\longrightarrow} \mathrm{tr} \,\, \Delta^{-3} 
\mathbb{E}[X_1 X_1^{\top} (\nu_k^{\top} X_1 - \nu_k^{\top} X_1 X_1^{\top} \nu_k)]
\end{equation}
as $n \rightarrow \infty$. If, however, $\rho_n \rightarrow 0$ and $n \rho_n = \omega(\log^{4}(n))$, then
\begin{equation}
\label{eq:d_kk2}
n^2 \hat{d}_{kk} \overset{\mathrm{a.s.}}{\longrightarrow} \mathrm{tr} \,\, \Delta^{-3} 
\mathbb{E}[X_1 X_1^{\top} \nu_k^{\top} X_1]
\end{equation}
as $n \rightarrow \infty$.
\end{theorem}

For the $\tilde{d}_{kl}$, we have the following result.
\begin{theorem}
\label{THM:between-LSE}
Assume the setting and notations of Theorem~\ref{THM:LSE} and suppose furthermore that $F = \sum_{k} \pi_{k} \delta_{\nu_k}$ is a mixture of $K$ distinct point masses at $\nu_1, \nu_2, \dots, \nu_K \in \mathbb{R}^{d}$. Let $\tilde{\mathbf{U}}_{\mathbf{P}_n}$ denote the $n \times d$ matrix whose columns are the orthonormal eigenvectors corresponding to the non-zero eigenvalues of the matrix $\mathcal{L}(\mathbf{P}_n) = \mathcal{L}(\rho_n \mathbf{X}_n \mathbf{X}_n^{\top}) = \mathcal{L}(\mathbf{X}_n \mathbf{X}_n^{\top})$. For any $k \in \{1,2,\dots,K\}$, let $\mathbf{S}_{k}$ be the $n \times n$ diagonal matrix with diagonal entries $(s_{k}(1), s_{k}(2), \dots, s_{k}(n))$ such that $s_{k}(i) = 1$ if $X_i = \nu_k$ and $s_{k}(i) = 0$ otherwise. We then have, for any $k \in \{1,2,\dots,K\}$
\begin{equation}
\begin{split}
n^2 \tilde{d}_{kk} &= \frac{n^2}{|C_k|} \|\mathbf{S}_k (\tilde{\mathbf{U}}_{\mathbf{A}_n} \mathbf{W}_n - \tilde{\mathbf{U}}_{\mathbf{P}_n}) \|_{F}^{2} + o_{\mathbb{P}}(1) \\
&= \frac{n^2}{|C_k|} \|\mathbf{S}_{k} \mathbf{M}_1 (\tilde{\mathbf{X}}_n^{\top} \tilde{\mathbf{X}}_n)^{-3/2}
  +  \frac{1}{2} \mathbf{S}_{k} \mathbf{M}_2  (\tilde{\mathbf{X}}_n^{\top} \tilde{\mathbf{X}}_n)^{-1/2}\|_{F}^{2} + o_{\mathbb{P}}(1)
\end{split}
\end{equation}
where $\mathbf{M}_1$ and $\mathbf{M}_2$ are defined as
\begin{gather}
\mathbf{M}_1 = \mathbf{T}_n^{-1/2} (\mathbf{A}_n - \mathbf{P}_n) \mathbf{T}_n^{-1/2} \tilde{\mathbf{X}}_n \\
\mathbf{M}_2 = \mathbf{T}_n^{-1/2} (\mathbf{T}_n - \mathbf{D}_n) \mathbf{T}_n^{-1/2} \tilde{\mathbf{X}}_n .
\end{gather}
Therefore, if $\rho_n \equiv 1$, then for any $k \in \{1,2,\dots,K\}$
\begin{equation}
\label{eq:d_kk_tilde1}
n^2 \tilde{d}_{kk} \overset{\mathrm{a.s.}}{\longrightarrow} 
 \mathrm{tr} \,\, \tilde{\Delta}^{-3} \mathbb{E}\Bigl[ \Bigl(\frac{X_1}{X_1^{\top} \mu} - \frac{\tilde{\Delta} \nu_k}{2 \nu_k^{\top} \mu}\Bigr) \Bigl(\frac{X_1^{\top}}{X_1^{\top} \mu} - \frac{\nu_k^{\top} \tilde{\Delta}}{2 \nu_k^{\top} \mu}\Bigr) \frac{(\nu_k^{\top} X_1 - \nu_k^{\top} X_1 X_1^{\top} \nu_k)}{\nu_k^{\top} \mu} \Bigr]
\end{equation}
as $n \rightarrow \infty$. If, however, $\rho_n \rightarrow 0$ and $n \rho_n = \omega(\log^{4}(n))$, then
\begin{equation}
\label{eq:d_kk_tilde2}
n^2 \tilde{d}_{kk} \overset{\mathrm{a.s.}}{\longrightarrow} 
 \mathrm{tr} \,\, \tilde{\Delta}^{-3} \mathbb{E}\Bigl[ \Bigl(\frac{X_1}{X_1^{\top} \mu} - \frac{\tilde{\Delta} \nu_k}{2 \nu_k^{\top} \mu}\Bigr) \Bigl(\frac{X_1^{\top}}{X_1^{\top} \mu} - \frac{\nu_k^{\top} \tilde{\Delta}}{2 \nu_k^{\top} \mu}\Bigr) \frac{\nu_k^{\top} X_1}{\nu_k^{\top} \mu} \Bigr]
\end{equation}
as $n \rightarrow \infty$.
\end{theorem}

\begin{remark}
We note that the $\hat{d}_{kl}$ and $\tilde{d}_{kl}$ are defined in terms of $\mathbf{U}_{\mathbf{A}}$
and $\tilde{\mathbf{U}}_{\mathbf{A}}$ and not in terms of $\hat{\mathbf{X}} = \mathbf{U}_{\mathbf{A}} \mathbf{S}_{\mathbf{A}}^{1/2}$ and $\breve{\mathbf{X}} = \tilde{\mathbf{U}}_{\mathbf{A}} \tilde{\mathbf{S}}_{\mathbf{A}}^{1/2}$. This is because $\|\mathbf{S}_{\mathbf{A}}^{1/2}\| \gg \|\tilde{\mathbf{S}}_{\mathbf{A}}^{1/2}\|$. In addition, as we alluded to previously, the $\hat{d}_{kl}$ and $\tilde{d}_{kl}$ do not explicitly take into account the structure of the block-conditional covariance matrices; instead they measure only the average Euclidean distance of a point to its block-conditional cluster centroid -- this coincides with taking the trace of the covariance matrices. Therefore, the $\hat{d}_{kl}$ and $\tilde{d}_{kl}$ serve as a surrogate only for the performance of the $K$-means $\circ \, \mathrm{ASE}$ and $K$-means $\circ \, \mathrm{LSE}$ procedures for recovering block assignments. As Figure~\ref{fig:GMM-kmeans} illustrates, the $K$-means $\circ \, \mathrm{ASE}$ and $K$-means $\circ \, \mathrm{LSE}$ procedures do not yield the {\em optimal} error rate for the inference task at hand. That is to say, the within-block variances cannot be use to compare the ASE and LSE for subsequent inference in a way that is {\em independent} of the clustering procedure used. Roughly speaking, what we want is to be able to compare, for a given stochastic blockmodel graph $\mathbf{A}$, the large-sample error rate of $\inf_{T} T \circ \mathrm{ASE}$ versus the large-sample error rate of $\inf_{T'} T' \circ \mathrm{LSE}$; here $T$ and $T'$ range over all possible transformations and clusterings procedure. This comparison is facilitated by the limit results of Corollary~\ref{cor:ase_normality_sbm} and Corollary~\ref{cor:lse_normality_sbm} and the notion of the Chernoff information.
 \end{remark}

\subsection{Chernoff Information}
\label{sec:chernoff}
Let $F_0$ and $F_1$ be two absolutely continuous multivariate distributions in $\Omega = \mathbb{R}^{d}$ with density functions $f_0$ and $f_1$, respectively. Suppose that $Y_1, Y_2, \dots, Y_m$ are independent and identically distributed random variables, with $Y_i$ distributed either $F_0$ or $F_1$. We are interested in testing the simple null hypothesis $\mathbb{H}_0 \colon F = F_0$ against the simple alternative hypothesis $\mathbb{H}_1 \colon F = F_1$. A test $T$ can be viewed as a sequence of mappings $T_m \colon \Omega^{m} \mapsto \{0,1\}$ such that given $Y_1 = y_1, Y_2 = y_2, \dots, Y_m = y_m$, the test rejects $\mathbb{H}_0$ in favor of $\mathbb{H}_1$ if $T_m(y_1, y_2, \dots, y_m) = 1$; similarly, the test favors $\mathbb{H}_0$ if $T_m(y_1, y_2, \dots, y_m) = 0$. 

The Neyman-Pearson lemma states that, given $Y_1 = y_1, Y_2 = y_2, \dots, Y_m = y_m$ and a threshold $\eta_m \in \mathbb{R}$, the likelihood ratio test which rejects $\mathbb{H}_0$ in favor of $\mathbb{H}_1$ whenever
$$ \Bigl(\sum_{i=1}^{m} \log{f_0(y_i)} - \sum_{i=1}^{m} \log{f_1(y_i)} \Bigr) \leq \eta_m $$
is the most powerful test at significance level $\alpha_m = \alpha(\eta_m)$, i.e., the likelihood ratio test minimizes the type-II error $\beta_m$ subject to the contrainst that the type-I error is at most $\alpha_m$. 

Assuming that $\pi \in (0,1)$ is a prior probability that $\mathbb{H}_0$ is true. Then, for a given $\alpha_m^{*} \in (0,1)$, let $\beta_m^{*} = \beta_m^{*}(\alpha_m^{*})$ be the type-II error associated with the likelihood ratio test when the type-I error is at most $\alpha_m^{*}$. The quantity $\inf_{\alpha_m^{*} \in (0,1)} \pi \alpha_m^{*} + (1 - \pi) \beta_m^{*}$ is then the Bayes risk in deciding between $\mathbb{H}_0$ and $\mathbb{H}_1$ given the $m$ independent random variables $Y_1, Y_2, \dots, Y_m$. A classical result of Chernoff \cite{chernoff_1952,chernoff_1956} states that the Bayes risk is intrinsically linked to a quantity known as the {\em Chernoff information}. More specifically, let $C(F_0, F_1)$ be the quantity
\begin{equation}
\label{eq:chernoff-defn}
\begin{split} C(F_0, F_1) & = - \log \, \Bigl[\, \inf_{t \in (0,1)} \int_{\mathbb{R}^{d}} f_0^{t}(\bm{x}) f_1^{1-t}(\bm{x}) \mathrm{d}\bm{x} \Bigr] \\
&= \sup_{t \in (0,1)} \Bigl[ - \log \int_{\mathbb{R}^{d}} f_0^{t}(\bm{x}) f_1^{1-t}(\bm{x}) \mathrm{d}\bm{x} \Bigr].
\end{split}
\end{equation}
Then we have
\begin{equation}
\label{eq:chernoff-binary}
\begin{split}
\lim_{m \rightarrow \infty} \frac{1}{m} \inf_{\alpha_m^{*} \in (0,1)} \log( \pi \alpha_m^{*} + (1 - \pi) \beta_m^{*}) & = - \, C(F_0, F_1).
\end{split}
\end{equation}
Thus $C(F_0, F_1)$, the Chernoff information between $F_0$ and $F_1$, is the {\em exponential} rate at which the Bayes error $\inf_{\alpha_m^{*} \in (0,1)} \pi \alpha_m^{*} + (1 - \pi) \beta_m^{*}$ decreases as $m \rightarrow \infty$; we note that the Chernoff information is independent of $\pi$. We also define, for a given $t \in (0,1)$ the Chernoff divergence $C_t(F_0, F_1) $ between $F_0$ and $F_1$ by
$$ C_{t}(F_0,F_1) = - \log \int_{\mathbb{R}^{d}} f_0^{t}(\bm{x}) f_1^{1-t}(\bm{x}) \mathrm{d}\bm{x}. $$
The Chernoff divergence is an example of a $f$-divergence as defined in \cite{Csizar,Ali-Shelvey}. When $t = 1/2$, $C_t(F_0,F_1)$ is the Bhattacharyya distance between $F_0$ and $F_1$. As we mentioned previously, any $f$-divergence satisfies the information processing lemma and is invariant with respect to invertible transformations \cite{Liese_Vadja}. Thus any $f$-divergence such as the Kullback-Liebler divergence can also be used to compare the two embedding methods. We chose the Chernoff information mainly because of its explicit relationship with the Bayes risk.

The result of Eq.~\eqref{eq:chernoff-binary} can be extended to $K + 1 \geq 2$ hypotheses. Let $F_0, F_1, \dots, F_{K}$ be distributions on $\mathbb{R}^{d}$ and suppose that $Y_1, Y_2, \dots, Y_m$ are independent and identically distributed random variables with $Y_i$ distributed $F \in \{F_0, F_1, \dots, F_K\}$. We are thus interested in determining the distribution of the $Y_i$ among the $K+1$ hypothesis $\mathbb{H}_0 \colon F = F_0, \dots, \mathbb{H}_{K} \colon F = F_K$. Suppose also that hypothesis $\mathbb{H}_k$ has {\em a priori} probabibility $\pi_k$. Then for any decision rule $\delta$, the risk of $\delta$ is $r(\delta) = \sum_{k} \pi_k \sum_{l \not = k} \alpha_{lk}(\delta) $ where $\alpha_{lk}(\delta)$ is the probability of accepting hypothesis $\mathbb{H}_l$ when hypothesis $\mathbb{H}_k$ is true. Then we have \cite{leang-johnson}
\begin{equation}
\label{eq:chernoff-multiple}
\inf_{\delta} \lim_{m \rightarrow \infty}  \frac{r(\delta)}{m} = - \min_{k \not = l} C(F_k, F_l).
\end{equation}
where the infimum is over all decision rules $\delta$. That is to say, for any $\delta$, $r(\delta)$ decreases to $0$ as $m \rightarrow \infty$ at a rate no faster than $\exp(- m \min_{k \not = l} C(F_k, F_l))$. It was also shown in \cite{leang-johnson} that the {\em Maximum A Posterior} decision rule achieves this rate. 


For this paper, we are interested in computing the Chernoff information $C(F_0,F_1)$ when $F_0$ and $F_1$ are multivariate normals. Suppose $F_0 =  \mathcal{N}(\mu_0, \Sigma_0)$ and $F_1 = \mathcal{N}(\mu_1, \Sigma_1)$; then, denoting by $\Sigma_t = t \Sigma_0 + (1 - t) \Sigma_1$, we have 
\begin{equation*}
C(F_0, F_1) = \sup_{t \in (0,1)} \Bigl(\frac{t(1 - t)}{2} (\mu_1 - \mu_2)^{\top}\Sigma_t^{-1}(\mu_1 - \mu_2) + \frac{1}{2} \log \frac{|\Sigma_t|}{|\Sigma_0|^{t} |\Sigma_1|^{1 - t}}  \Bigr).
\end{equation*}

\subsection{Comparison of ASE and LSE via Chernoff information}
\label{sec:chernoff2}
We now employ the limit results of Corollary~\ref{cor:ase_normality_sbm} and Corollary~\ref{cor:lse_normality_sbm} to compare the 
performance of the Laplacian spectral embedding and the adjacency spectral embedding for subsequent inference. Our subsequent inference task is once again the problem of recovering the block assignments in stochastic blockmodel graphs; furthermore, we are interested in estimating the {\em large-sample optimal} error rate possible for recovering the underlying block assignments after the spectral emebdding step is carried out. 
The discussion in Section~\ref{sec:chernoff} indicates that an appropriate measure for the large-sample optimal error rate for spectral clustering using adjacency or Laplacian spectral embedding is in terms of the minimum of the pairwise Chernoff informations between the multivariate normal distributions as specified in Corollary~\ref{cor:ase_normality_sbm} or Corollary~\ref{cor:lse_normality_sbm}. More specifically, let $\mathbf{B} \in [0,1]^{K \times K}$ and $\bm{\pi} \in \mathbb{R}^{K}$ be the matrix of block probabilities and the vector of block assignment probablities for a $K$-block stochastic blockmodel. We shall assume that $\mathbf{B}$ is positive semidefinite. Then given an $n$ vertex instantiation of the SBM graph with parameters $(\bm{\pi}, \mathbf{B})$, for sufficiently large $n$, the large-sample optimal error rate for recovering the block assignments when adjacency spectral embedding is used as the initial embedding step can be characterized by the quantity $\rho_{\mathrm{A}} = \rho_{\mathrm{A}}(n)$  defined by
\begin{equation}
\label{eq:rho_ASE}
\rho_{\mathrm{A}} = \min_{k \not = l} \! \sup_{t \in (0,1)} \frac{1}{2} \log \frac{|\Sigma_{kl}(t)|}{|\Sigma_{k}|^{t} |\Sigma_{l}|^{1-t}} + \frac{nt(1 - t)}{2}(\nu_k - \nu_l)^{\top} \Sigma_{kl}^{-1}(t) (\nu_k - \nu_l)
\end{equation}
where $\Sigma_{kl}(t) = t \Sigma_{k} + (1 - t) \Sigma_l$. We recall Eq.~\eqref{eq:chernoff-multiple}, in particular the fact that as $\rho_{\mathrm{A}}$ increases, the large-sample optimal error rate decreases. 
Similarly, the large-sample optimal error rate when Laplacian spectral embedding is used as the pre-processing step can be characterized by the quantity
$\rho_{\mathrm{L}} = \rho_{\mathrm{L}}(n)$ defined by
\begin{equation} 
\label{eq:rho_LSE}
\rho_{\mathrm{L}} = \min_{k \not = l} \sup_{t \in (0,1)} \frac{1}{2} \log \frac{|\tilde{\Sigma}_{kl}(t)|}{|\tilde{\Sigma}_{k}|^{t} |\tilde{\Sigma}_{l}|^{1-t}} + \frac{n t(1 - t)}{2} (\tilde{\nu}_k - \tilde{\nu}_l)^{\top} \tilde{\Sigma}_{kl}^{-1}(t) (\tilde{\nu}_k - \tilde{\nu}_l)
\end{equation}
where $\tilde{\Sigma}_{kl}(t) = t \tilde{\Sigma}_{k} + (1 - t) \tilde{\Sigma}_l$ and 
$\tilde{\nu}_k = \nu_k/(\sum_{k'} \pi_{k'} \nu_k^{\top} \nu_{k'})^{1/2}$. We emphasize that we have made the simplifying assumption that $n_k = n \pi_k$ in our expression for $\tilde{\nu}_k$ in Eq.~\eqref{eq:rho_LSE}. This is for ease of comparison between $\rho_{\mathrm{A}}$ and $\rho_{\mathrm{L}}$ in our subsequent discussion.

We thus propose to use the ratio $\rho_{\mathrm{A}}/\rho_{\mathrm{L}}$ as a measure of the relative large-sample performance of the adjacency spectral embedding as compared to the Laplacian spectral embedding for subsequent inference, at least in the context of stochastic blockmodel graphs. That is to say, for given parameters $\bm{\pi}$ and $\mathbf{B}$, if $\rho_{\mathrm{A}}/\rho_{\mathrm{L}} > 1$ then adjacency spectral embedding is to be preferred over Laplacian spectral embedding when $n$, the number of vertices in the graph, is sufficiently large; similarly, if $\rho_{\mathrm{A}}/\rho_{\mathrm{L}} < 1$ then Laplacian spectral embedding is to be preferred over adjacency spectral embedding. 

\begin{remark}
We note that if the block-conditional covariance matrices $\Sigma_{k}$ are all non-singular, then for sufficiently large $n$, the term $\log \tfrac{|\Sigma_{kl}(t)|}{|\Sigma_{k}|^{t} |\Sigma_{l}|^{1-t}}$ in the definition of $\rho_{\mathrm{A}}$ is negligible; similarly, the term $\log \frac{|\tilde{\Sigma}_{kl}(t)|}{|\tilde{\Sigma}_{k}|^{t} |\tilde{\Sigma}_{l}|^{1-t}}$ in the definition of $\rho_{\mathrm{L}}$ is also negligible. However, on occassion, some of the block-conditional covariance matrices $\Sigma_{k}$ are singular. As an example, we consider a completely associative two-block stochastic blockmodel with $\mathbf{B} = \Bigl[\begin{smallmatrix} p^2 & 0 \\ 0 & q^2 \end{smallmatrix} \Bigr]$ and $\bm{\pi} = (\pi_1, \pi_2)$. Then the block-conditional covariance matrices are 
\begin{gather*} \Sigma_{1} = (1 - p^2) \begin{bmatrix} \pi_1^{-1} & 0 \\ 0 & 0 \end{bmatrix}; \quad \Sigma_{2} = (1 - p^2) \begin{bmatrix} 0 & 0 \\ 0 & \pi_2^{-1} \end{bmatrix} \\
\tilde{\Sigma}_{1} = \frac{(1 - p^2)}{4 p^2} \begin{bmatrix} \pi_1^{-2} & 0 \\ 0 & 0 \end{bmatrix}; \quad \tilde{\Sigma}_2 = \frac{(1 - p^2)}{4 p^2} \begin{bmatrix} 0 & 0 \\ 0 & \pi_2^{-2} \end{bmatrix},
\end{gather*}
and $\rho_{\mathrm{A}} = \rho_{\mathrm{L}} = \infty$. Therefore, ASE and LSE are equivalent with respect to the subsequent inference task. In contrast, \cite{bickel_sarkar_2013} showed that the within-block variances for ASE are four times larger than that of the within-block variances of LSE, while the between-block variances for ASE and LSE are the same. We conclude that the within-block variances measure fails to capture the fact that the block-conditional covariance matrices $\Sigma_{1}$ and $\Sigma_{2}$ are singular but in different subspaces, and similarly $\tilde{\Sigma}_1$ and $\tilde{\Sigma}_2$ are also singular but in different subspaces, and thus if we had used the within-block variances measure as a surrogate, we would have been misled into believing that LSE is preferable to ASE for this particular subsequent inference task. Indeed, had we ignored the terms $\log \tfrac{|\Sigma_{kl}(t)|}{|\Sigma_{k}|^{t} |\Sigma_{l}|^{1-t}}$ and $\log \frac{|\tilde{\Sigma}_{kl}(t)|}{|\tilde{\Sigma}_{k}|^{t} |\tilde{\Sigma}_{l}|^{1-t}}$ in the definitions of $\rho_{\mathrm{A}}$ and $\rho_{\mathrm{L}}$, we would have come to the similar conclusion that $\rho_{\mathrm{L}} = \tfrac{2p^2}{1 - p^2} \max\{\pi_1, \pi_2\} = 4 \rho_{\mathrm{A}}$ for sufficiently large $n$. 
\end{remark}

\begin{figure}[tp!]
\center
\includegraphics[width=0.7\textwidth]{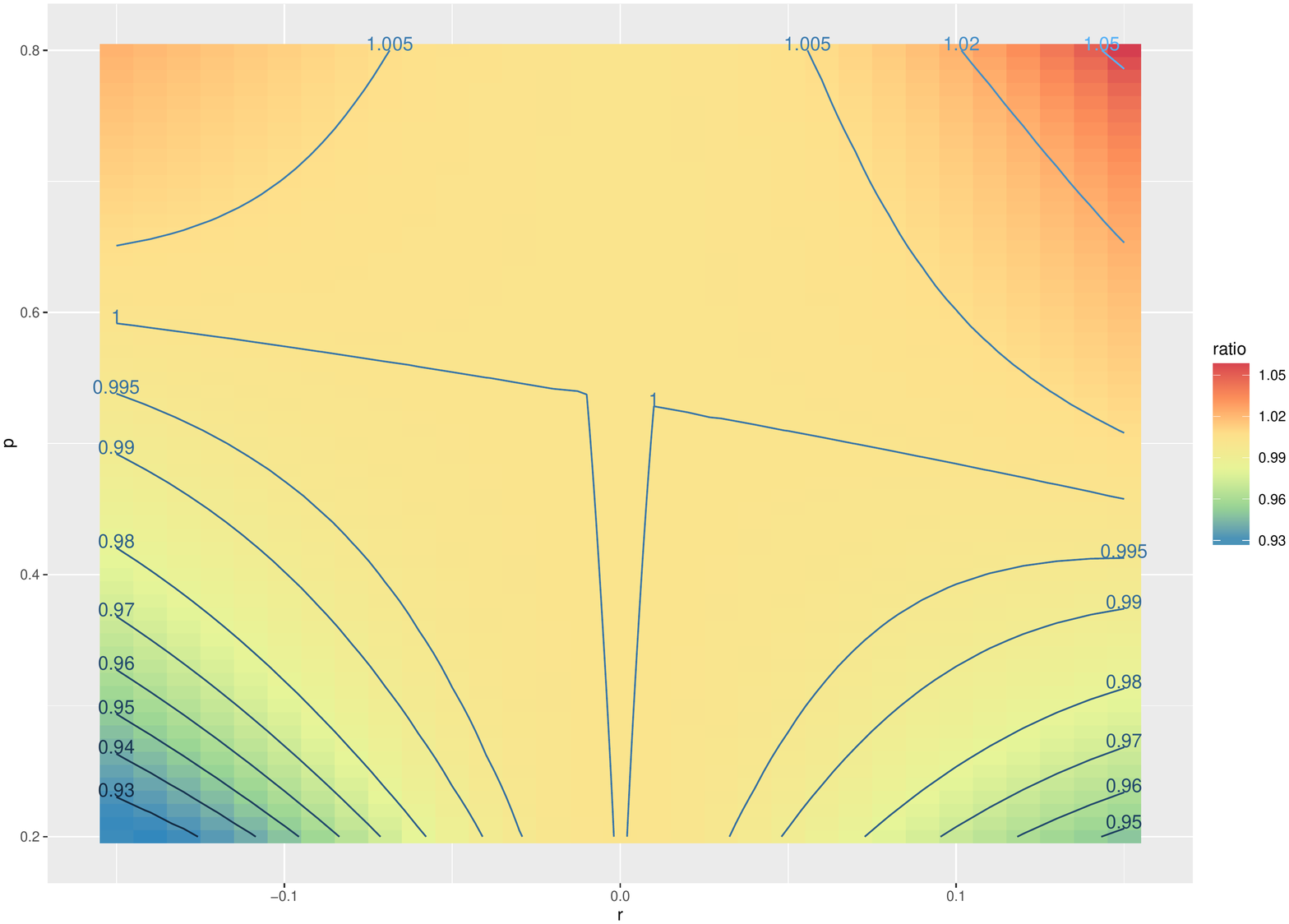}
\caption{The ratio $\rho_{\mathrm{A}}/\rho_{\mathrm{L}}$ displayed for various values of $p \in [0.2, 0.8]$ and $r = q - p \in [-0.15, 0.15]$. The labeled lines are the contour lines for $\rho_{\mathrm{A}}/\rho_{\mathrm{L}}$. 
}
\label{fig:ratio-plot}
\end{figure}

As an illustration of the ratio $\rho_{\mathrm{A}}/\rho_{\mathrm{L}}$, we first consider the collection of 2-block stochastic blockmodels where $\mathbf{B} = \Bigl[ \begin{smallmatrix} p^2 & pq \\ pq & q^2 \end{smallmatrix} \Bigr]$ for $p, q \in (0,1)$ and $\bm{\pi} = (\pi_1, \pi_2 )$ with $\pi_1 + \pi_2 = 1$. Then for sufficiently large $n$, $\rho_{\mathrm{A}}$ is approximately
$$ \rho_{\mathrm{A}} \approx \sup_{t \in (0,1)} \frac{nt(1 - t)}{2} (p - q)^{2} (t \sigma_1^{2} + (1 - t) \sigma_2^{2})^{-1}$$
where $\sigma_1$ and $\sigma_2$ are as specified in Eq.~\eqref{eq:er-p-q-ase1} and Eq.~\eqref{eq:er-p-q-ase2}, respectively. Simple calculations yield
$$ \rho_{\mathrm{A}} \approx \frac{n(p - q)^2 (\pi_1 p^2 + \pi_2 q^2)^2}{2\bigl(\sqrt{\pi_1 p^4 (1 - p^2) + \pi_2 p q^3(1 - pq) } + \sqrt{\pi_1 p^3 q(1 - pq) + \pi_2 q^4 (1 - q^2)}\bigr)^2}$$ for sufficiently large $n$. Similarly, denoting by $\tilde{\sigma}_1^{2}$ and $\tilde{\sigma}_2^2$ the variances specified in Eq.~\eqref{eq:er-p-q-lse1} and Eq.~\eqref{eq:er-p-q-lse2}, we have
\begin{equation*}
\begin{split}
 \rho_{\mathrm{L}} & \approx \sup_{t \in (0,1)} \frac{nt(1 - t)}{2} \Bigl(\frac{p}{\sqrt{0.6 \pi_1^2 + \pi_2 pq}} - \frac{q}{\sqrt{\pi_1 p q + \pi_2 q^2}}\Bigr)^{2} (t \tilde{\sigma}_1^{2} + (1 - t) \tilde{\sigma}_2^2)^{-1} \\
 & \approx \frac{2n(\sqrt{p} - \sqrt{q})^2 (\pi_1 p + \pi_2 q)^2}{\bigl(\sqrt{\pi_1 p (1 - p^2) + \pi_2 q (1 - pq)} + \sqrt{\pi_1 p (1 - pq) + \pi_2 q (1 - q^2)}\bigr)^2}
 \end{split}
\end{equation*}
for sufficiently large $n$. 
Fixing $\bm{\pi} = (0.6, 0.4)$,
we computed the ratio $\rho_{\mathrm{A}}/\rho_{\mathrm{L}}$ for a range of $p$ and $q$ values, with $p \in [0.2, 0.8]$ and $q = p + r$ where $r \in [-0.15, 0.15]$. The results are plotted in Figure~\ref{fig:ratio-plot}. The $y$-axis of Figure~\ref{fig:ratio-plot} denotes the values of $p$ and the $x$ axis are the values of $r$. 

We also generate instances of a stochastic blockmodel graph on $200$ vertices with parameters $p = 0.75$ and $q = 0.6$. For each graph we measure the error rate of the spectral embedding followed by the Gaussian mixture-model based clustering procedure in recovering the block assignments. The error rate for the $\mathrm{GMM} \circ \mathrm{ASE}$ procedure, averaged over $1000$ Monte Carlo replicates, is $0.079$ with a standard error of $6.6 \times 10^{-4}$; meanwhile the error rate for the $\mathrm{GMM} \circ \mathrm{LSE}$ procedure, also averaged over $1000$ Monte Carlo replicates, is $0.083$ with a standard error of $7.2 \times 10^{-6}$. The difference in the mean error rate is statistically significant at $\alpha = 0.001$. Conversely, when $p = 0.2$ and $q = 0.3$ and the graphs are on $400$ vertices, the mean error rate, over $1000$ Monte Carlo replicates, for the $\mathrm{GMM} \circ \mathrm{ASE}$ procedure is $0.161$ while the mean error rate for the $\mathrm{GMM} \circ \mathrm{LSE}$ procedure is $0.151$ and this difference is also statistically significant at $\alpha = 0.001$.

We next consider the collection of stochastic blockmodels with parameters $\bm{\pi}$ and $\mathbf{B}$ where 
\begin{equation}
\label{eq:3block-example}
 \mathbf{B} = \begin{bmatrix} p & q & q \\ q & p & q \\ q & q & p \\ \end{bmatrix}, \quad p, q \in (0,1), \,\, \text{and} \,\, \bm{\pi} = (0.8, 0.1, 0.1).
\end{equation}
First we compute the ratio $\rho_{\mathrm{A}}/\rho_{\mathrm{L}}$ for $p \in [0.3, 0.9]$ and $r = q - p$ with $r \in [- 0.2, -0.01]$. The results are plotted in Figure~\ref{fig:ratio_3blocks}, with the $y$-axis of Figure~\ref{fig:ratio_3blocks} being the values of $p$ and the $x$-axis being the values of $r$. We then generate instances of a stochastic blockmodel graph on $800$ vertices with $p = 0.9$ and $q = 0.72$ and estimate the error rate of the $\mathrm{GMM} \circ \mathrm{ASE}$ and the $\mathrm{GMM} \circ \mathrm{LSE}$ procedures in recovering the block assignments. The $\mathrm{GMM} \circ \mathrm{ASE}$ and $\mathrm{GMM} \circ \mathrm{LSE}$ error rates, averaged over $1000$ Monte Carlo replicates, are $0.29$ and $0.38$, respectively. For these choice of parameters, $\rho_{\mathrm{A}}/\rho_{\mathrm{L}} \approx 1.01$. We also generate instances of a stochastic blockmodel graph on $1600$ vertices with $p = 0.34$ and $q = 0.15$. The ratio $\rho_{\mathrm{A}}/\rho_{\mathrm{L}}$ in this case is $\approx 0.98$; the $\mathrm{GMM} \circ \mathrm{ASE}$ and $\mathrm{GMM} \circ \mathrm{LSE}$ error rates, averaged over $1000$ Monte Carlo replicates, are $0.18$ and $0.06$, respectively. 

\begin{figure}[htbp]
\center 
\includegraphics[width=0.8\textwidth]{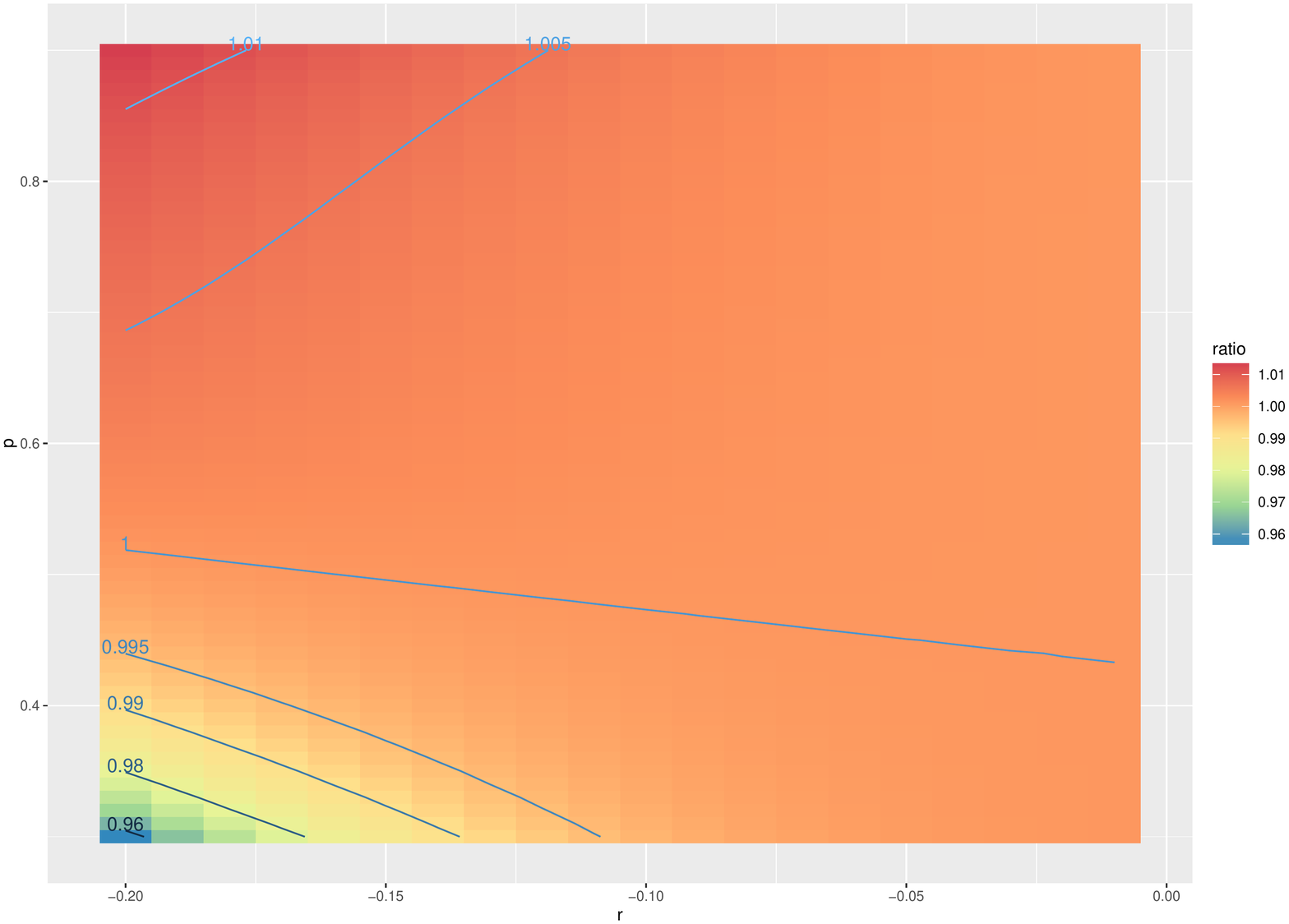}
\caption{The ratio $\rho_{A}/\rho_{L}$ displayed for various values of $p \in [0.2, 0.8]$ and $r = q - p \in [-0.2, -0.01]$ for the 3-block stochastic blockmodel of Eq.~\eqref{eq:3block-example}. The labeled lines are the contour lines for $\rho_{\mathrm{A}}/\rho_{\mathrm{L}}$. 
 }
\label{fig:ratio_3blocks}
\end{figure}

\section{Summary and Conclusions}
\label{sec:conclusions}
We shown in this paper several limit results for the eigenvectors corresponding to the largest eigenvalues of the normalized Laplacian matrix 
of random graphs. In particular, we show that for stochastic blockmodel graphs, conditioned on the block assignments, each row of the Laplacian spectral embedding converges to a multivariate normal distribution. We then discuss the relationship between spectral embeddings of the adjacency and normalized Laplacian matrices and subsequent inference. When the subsequent inference task is the problem of clustering the vertices of a graph, we show that the Chernoff information between the multivariate normals approximation of the embedding is a suitable measure for the large-sample optimal error rate, i.e., it characterizes the minimum error rate achievable by {\em any} clustering procedure that operates only on the spectral embedding. As a result, we are able to theoretically compare the use of spectral embedding of the adjacency matrix versus that of the normalized Laplacian for subsequent inference, thereby refining and extending the pioneering work of \cite{bickel_sarkar_2013}. 

We now mention several potential extensions of this work. 
The normalized Laplacian considered in this paper is just one example of possible normalization. In particular, given $\tau > 0$ one can define the $\tau$-regularized normalized Laplacian $\mathcal{L}_{\tau}$ via $\mathcal{L}_{\tau}(\mathbf{A}) = (\mathbf{D} + \tau \mathbf{I})^{-1/2} \mathbf{A} (\mathbf{D} + \tau \mathbf{I})^{-1/2}$ or $\mathcal{L}_{\tau}(\mathbf{A}) = (\mathbf{D} + \tau \mathbf{I})^{-1/2} (\mathbf{A} + \tau \bm{1} \bm{1}^{\top}) (\mathbf{D} + \tau \mathbf{I})^{-1/2}$\cite{chaudhuri12:_spect,qin2013dcsbm,Amini}. It had been shown that regularization is particularly useful for spectral clustering in sparse graphs. It will thus be of interest to derive limit results for the eigenvectors of $\mathcal{L}_{\tau}(\mathbf{A})$ analogous to those in this paper; such results can potentially allow one to choose the regularization parameter $\tau$. 
 
The limit results in this paper are for the spectral embedding of $(\mathbf{X}, \mathbf{A}) \sim \mathrm{RPDG}(F)$ into $\mathbb{R}^{d}$ when $d$, the rank of the matrix $\mathbb{E}[X X^{\top}]$ where $X \sim F$, is fixed and known. Similar results can be derived when the spectral embedding of $\mathbf{A}$ is into $\mathbb{R}^{d'}$ where $d' < d$. Limit results for spectral embedding of the adjacency matrix or Laplacian matrix into $\mathbb{R}^{d'}$ when $d' > d$ is, to the best of our knowledge, an open problem. A related inquiry is limit results for spectral embedding into $\mathbb{R}^{d'}$ when $d' < d$ but $d$ varies with $n$ and is not fixed, such as when the graph arises from a latent position model where the link function, viewed as an integral operator, has infinite rank. Since new results on stochastic blockmodels indicate that they can be regarded as a universal approximation to latent positions model graphs or graphons of exchangeable random graphs \cite{wolfe13:_nonpar,yang14:_nonpar}, limit results for the adjacency and Laplacian spectral embedding will be useful in further understanding of this approximation property. 

Finally, the Chernoff information used in this paper is a measure of the effect of spectral embedding on subsequent inference for a single graph. Recently, however, there has been interests in two-sample inference for graphs, e.g.,  
network comparisons or two-sample hypothesis testing for graphs \cite{Asta,tang14:_nonpar,tang14:_semipar}. As an example, given two distributions $F$ and $G$, the problem of testing whether $F = G$ given two random dot product graphs $\mathbf{A} \sim \mathrm{RDPG}(F)$ and $\mathbf{B} \sim \mathrm{RDPG}(G)$ was considered in 
\cite{tang14:_nonpar}; the proposed test statistic is a kernel-based distance measure between the spectral embedding $\hat{\mathbf{X}}$ of $\mathbf{A}$ and $\hat{\mathbf{Y}}$ of $\mathbf{B}$. Determining a measure that characterizes the effect of spectral embedding for two-sample graphs inference problems, akin to how the Chernoff information characterize the effect of spectral emebdding for single graph inference, is of significant interest.

\bibliography{../biblio}

\begin{thebibliography}{44}
\providecommand{\natexlab}[1]{#1}
\providecommand{\url}[1]{\texttt{#1}}
\expandafter\ifx\csname urlstyle\endcsname\relax
  \providecommand{\doi}[1]{doi: #1}\else
  \providecommand{\doi}{doi: \begingroup \urlstyle{rm}\Url}\fi

\bibitem[Ali and Shelvey(1966)]{Ali-Shelvey}
S.~M. Ali and S.~D. Shelvey.
\newblock A general class of coefficients of divergence of one distribution
  from another.
\newblock \emph{Journal of the Royal Statistical Society, Series B.},
  28:\penalty0 121--132, 1966.

\bibitem[Amini et~al.(2013)Amini, Chen, Bickel, and Levina]{Amini}
A.~Amini, A.~Chen, P.~Bickel, and E.~Levina.
\newblock Pseudo-likelihood methods for community detection in large sparse
  networks.
\newblock \emph{Annals of Statistics}, 41:\penalty0 2097--2122, 2013.

\bibitem[Asta and Shalizi(2014)]{Asta}
D.~Asta and C.~Shalizi.
\newblock Geometric network comparison.
\newblock Arxiv preprint at \url{http://arxiv.org/abs/1411.1350}, 2014.

\bibitem[Athreya et~al.(2016)Athreya, Lyzinski, Marchette, Priebe, Sussman, and
  Tang]{athreya2013limit}
A.~Athreya, V.~Lyzinski, D.~J. Marchette, C.~E. Priebe, D.~L. Sussman, and
  M.~Tang.
\newblock A limit theorem for scaled eigenvectors of random dot product graphs.
\newblock \emph{Sankhya A}, 78:\penalty0 1--18, 2016.

\bibitem[Belkin and Niyogi(2003)]{belkin03:_laplac}
M.~Belkin and P.~Niyogi.
\newblock Laplacian eigenmaps for dimensionality reduction and data
  representation.
\newblock \emph{Neural Computation}, 15:\penalty0 1373--1396, 2003.

\bibitem[Bickel and Sarkar()]{bickel_sarkar_2013}
P.~Bickel and P.~Sarkar.
\newblock Role of normalization for spectral clustering in stochastic
  blockmodels.
\newblock \emph{Annals of Statistics}, 43:\penalty0 962--990.

\bibitem[Bickel and Chen(2009)]{Bickel2009}
P.~J. Bickel and A.~Chen.
\newblock {A nonparametric view of network models and Newman-Girvan and other
  modularities.}
\newblock \emph{Proceedings of the National Academy of Sciences of the United
  States of America}, 106:\penalty0 21068--73, 2009.

\bibitem[Bollob\'{a}s et~al.(2007)Bollob\'{a}s, Janson, and
  Riordan]{bollobas07}
B.~Bollob\'{a}s, S.~Janson, and O.~Riordan.
\newblock The phase transition in inhomogeneous random graphs.
\newblock \emph{Random Structures and Algorithms}, 31:\penalty0 3--122, 2007.

\bibitem[Boucheron et~al.(2003)Boucheron, Lugosi, and Massart]{boucheron2003}
S.~Boucheron, G.~Lugosi, and P.~Massart.
\newblock Concentration inequalities using the entropy method.
\newblock \emph{Annals of Probability}, 31:\penalty0 1583--1614, 2003.

\bibitem[Chaudhuri et~al.(2012)Chaudhuri, Chung, and
  Tsiatas]{chaudhuri12:_spect}
K.~Chaudhuri, F.~Chung, and A.~Tsiatas.
\newblock Spectral partitioning of graphs with general degrees and the extended
  planted partition model.
\newblock In \emph{Proceedings of the 25th conference on learning theory},
  2012.

\bibitem[Chernoff(1952)]{chernoff_1952}
H.~Chernoff.
\newblock A measure of asymptotic efficiency for tests of a hypothesis based on
  the sum of observations.
\newblock \emph{Annals of Mathematical Statistics}, 23:\penalty0 493--507,
  1952.

\bibitem[Chernoff(1956)]{chernoff_1956}
H.~Chernoff.
\newblock Large sample theory: Parametric case.
\newblock \emph{Annals of Mathematical Statistics}, 27:\penalty0 1--22, 1956.

\bibitem[Choi et~al.(2012)Choi, Wolfe, and Airoldi]{Choi2010}
D.~S. Choi, P.~J. Wolfe, and E.~M. Airoldi.
\newblock Stochastic blockmodels with a growing number of classes.
\newblock \emph{Biometrika}, 99:\penalty0 273--284, 2012.

\bibitem[Chung(1997)]{chung1997spectral}
F.~R.~K. Chung.
\newblock \emph{Spectral Graph Teory}, volume~92.
\newblock American Mathematical Society, 1997.

\bibitem[Coifman and Lafon(2006)]{coifman06:_diffus_maps}
R.~Coifman and S.~Lafon.
\newblock Diffusion maps.
\newblock \emph{Applied and Computational Harmonic Analysis}, 21:\penalty0
  5--30, 2006.

\bibitem[Csiz\'{a}r(1967)]{Csizar}
I.~Csiz\'{a}r.
\newblock Information-type measures of difference of probability distributions
  and indirect observations.
\newblock \emph{Studia Scientiarum Mathematicarum Hungarica}, 2:\penalty0
  229--318, 1967.

\bibitem[Diaconis and
  Janson(2008)]{diaconis08:_graph_limit_exchan_random_graph}
P.~Diaconis and S.~Janson.
\newblock Graph limits and exchangeable random graphs.
\newblock \emph{Rendiconti di Matematica, Serie VII}, 28:\penalty0 33--61,
  2008.

\bibitem[Fraley and Raftery(1999)]{fraley99:_mclus}
C.~Fraley and A.~E. Raftery.
\newblock {MCLUST}: {S}oftware for model-based cluster analysis.
\newblock \emph{Journal of Classification}, 16:\penalty0 297--306, 1999.

\bibitem[Hoff et~al.(2002)Hoff, Raftery, and Handcock]{Hoff2002}
P.~D. Hoff, A.~E. Raftery, and M.~S. Handcock.
\newblock {Latent space approaches to social network analysis}.
\newblock \emph{Journal of the American Statistical Association}, 97\penalty0
  (460):\penalty0 1090--1098, 2002.

\bibitem[Holland et~al.(1983)Holland, Laskey, and Leinhardt]{holland}
P.~W Holland, K.~B. Laskey, and S.~Leinhardt.
\newblock Stochastic blockmodels: first steps.
\newblock \emph{Social Networks}, 5:\penalty0 109--137, 1983.

\bibitem[Karrer and Newman(2011)]{karrer2011stochastic}
B.~Karrer and M.~E.~J. Newman.
\newblock Stochastic blockmodels and community structure in networks.
\newblock \emph{Physical Review E}, 83:\penalty0 016107, 2011.

\bibitem[Leang and Johnson(1997)]{leang-johnson}
C.~C. Leang and D.~H. Johnson.
\newblock On the asymptotics of {M}-hypothesis bayesian detection.
\newblock \emph{IEEE Transactions on Information Theory}, 43:\penalty0
  280--282, 1997.

\bibitem[Lei and Rinaldo(2015)]{rinaldo_2013}
J.~Lei and A.~Rinaldo.
\newblock Consistency of spectral clustering in stochastic blockmodels.
\newblock \emph{Annals of Statistics}, 43:\penalty0 215--237, 2015.

\bibitem[Liese and Vadja(2006)]{Liese_Vadja}
F.~Liese and I.~Vadja.
\newblock On divergences and informations in statistics and information theory.
\newblock \emph{IEEE Transactions on Information Theory}, 52:\penalty0
  4394--4412, 2006.

\bibitem[Lu and Peng(2013)]{lu13:_spect}
L.~Lu and X.~Peng.
\newblock Spectra of edge-independent random graphs.
\newblock \emph{Electronic Journal of Combinatorics}, 20, 2013.

\bibitem[Luxburg(2007)]{von2007tutorial}
U.~Von Luxburg.
\newblock A tutorial on spectral clustering.
\newblock \emph{Statistics and Computing}, 17:\penalty0 395--416, 2007.

\bibitem[Lyzinski et~al.(2014)Lyzinski, Sussman, Tang, Athreya, and
  Priebe]{lyzinski13:_perfec}
V.~Lyzinski, D.~L. Sussman, M.~Tang, A.~Athreya, and C.~E. Priebe.
\newblock Perfect clustering for stochastic blockmodel graphs via adjacency
  spectral embedding.
\newblock \emph{Electronic Journal of Statistics}, 8:\penalty0 2905--2922,
  2014.

\bibitem[McSherry(2001)]{mcsherry}
F.~McSherry.
\newblock Spectral partitioning of random graphs.
\newblock In \emph{Proceedings of the 42nd IEEE Symposium on Foundations of
  Computer Science}, pages 529--537, 2001.

\bibitem[Merris(1994)]{merris1994}
R.~Merris.
\newblock Laplacian matrices of graphs: a survey.
\newblock \emph{Linear algebra and its applications}, 197:\penalty0 143--176,
  1994.

\bibitem[Mossel et~al.(In press.)Mossel, Neeman, and Sly]{mossel:ptrf}
E.~Mossel, J.~Neeman, and A.~Sly.
\newblock Stochastic block models and reconstruction.
\newblock \emph{Probab. Theory Related Fields}, In press.

\bibitem[Nickel(2006)]{nickel2006random}
C.~L.~M. Nickel.
\newblock \emph{Random dot product graphs: A model for social networks}.
\newblock PhD thesis, Johns Hopkins University, 2006.

\bibitem[Oliveira(2009)]{oliveira2009concentration}
R.~I. Oliveira.
\newblock Concentration of the adjacency matrix and of the {L}aplacian in
  random graphs with independent edges.
\newblock \url{http://arxiv.org/abs/0911.0600}, 2009.

\bibitem[Qin and Rohe(2013)]{qin2013dcsbm}
T.~Qin and K.~Rohe.
\newblock Regularized spectral clustering under the degree-corrected stochastic
  blockmodel.
\newblock \emph{NIPS}, 2013.

\bibitem[Rohe et~al.(2011)Rohe, Chatterjee, and Yu]{rohe2011spectral}
K.~Rohe, S.~Chatterjee, and B.~Yu.
\newblock Spectral clustering and the high-dimensional stochastic blockmodel.
\newblock \emph{Annals of Statistics}, 39:\penalty0 1878--1915, 2011.

\bibitem[Shi and Malik(2000)]{shi_malik}
J.~Shi and J.~Malik.
\newblock Normalized cuts and image segmentation.
\newblock \emph{IEEE Transactions on Pattern Analysis and Machine
  Intelligence}, 22:\penalty0 888--905, 2000.

\bibitem[Snijders and Nowicki(1997)]{Snijders1997Estimation}
T.~A.~B. Snijders and K.~Nowicki.
\newblock {Estimation and Prediction for Stochastic Blockmodels for Graphs with
  Latent Block Structure}.
\newblock \emph{Journal of Classification}, 14:\penalty0 75--100, 1997.

\bibitem[Stewart and Sun(1990)]{stewart90:_matrix}
G.~W. Stewart and J.~Sun.
\newblock \emph{Matrix pertubation theory}.
\newblock Academic Press, 1990.

\bibitem[Sussman(2014)]{sussman:_thesis}
D.~L. Sussman.
\newblock \emph{Foundations of Adjacency Spectral Embedding}.
\newblock PhD Thesis, Johns Hopkins University., 2014.

\bibitem[Sussman et~al.(2012)Sussman, Tang, Fishkind, and Priebe]{sussman12}
D.~L. Sussman, M.~Tang, D.~E. Fishkind, and C.~E. Priebe.
\newblock A consistent adjacency spectral embedding for stochastic blockmodel
  graphs.
\newblock \emph{Journal of the American Statistical Association}, 107:\penalty0
  1119--1128, 2012.

\bibitem[Tang et~al.(2016)Tang, Athreya, Sussman, Lyzinski, Park, and
  Priebe]{tang14:_semipar}
M.~Tang, A.~Athreya, D.~L. Sussman, V.~Lyzinski, Y.~Park, and C.~E. Priebe.
\newblock A semiparametric two-sample hypothesis testing problem for random dot
  product graphs.
\newblock \emph{Journal of Computational and Graphical Statistics}, 2016.
\newblock To appear.

\bibitem[Tang et~al.(In press.)Tang, Athreya, Sussman, Lyzinski, and
  Priebe]{tang14:_nonpar}
M.~Tang, A.~Athreya, D.~L. Sussman, V.~Lyzinski, and C.~E. Priebe.
\newblock A nonparametric two-sample hypothesis testing problem for random dot
  product graphs.
\newblock \emph{Bernoulli}, In press.

\bibitem[von Luxburg et~al.(2008)von Luxburg, Belkin, and
  Bousquet]{luxburg08:_consis}
U.~von Luxburg, M.~Belkin, and O.~Bousquet.
\newblock Consistency of spectral clustering.
\newblock \emph{Annals of Statistics}, 36:\penalty0 555--586, 2008.

\bibitem[Wolfe and Olhede(2013)]{wolfe13:_nonpar}
P.~J. Wolfe and S.~C. Olhede.
\newblock Nonparametric graphon estimation.
\newblock arXiv preprint at \url{http://arxiv.org/abs/1309/5936}, 2013.

\bibitem[Yang et~al.(2014)Yang, Han, and Airoldi]{yang14:_nonpar}
J.~J. Yang, Q.~Han, and E.~M. Airoldi.
\newblock Nonparametric estimation and testing of exchangeable graph models.
\newblock In \emph{Proceedings of the Seventeenth International Conference on
  Artificial Intelligence and Statistics}, pages 1060--1067, 2014.

\end{thebibliography}
\newpage
\appendix
\section{Proof of Theorem~\ref{THM:ASE} and Theorem~\ref{THM:NORMALITY-ASE}}
\label{sec:ase}

  We first present a sketch of the proof of Theorem~\ref{THM:ASE}, 
  noting that the main arguments are given in \cite{tang14:_semipar}. We also note that similar, albeit more involved, arguments are used in the proof of Theorem~\ref{THM:LSE}. Since the proof of Theorem~\ref{THM:LSE} will be presented in much greater detail in Section~\ref{sec:lse}, to avoid repetitions, we chose to omit the details in the current proof.
  Nevertheless, we emphasize that the statements of the results in \cite{tang14:_semipar} are slightly different from how they are stated in the current paper; these differences stem mainly from how sparseness in the graphs is incorporated. More specifically \cite{tang14:_semipar} considered a sequence of random dot product graphs where for each $n$, the matrix of latent positions $\mathbf{X}_n$ are fixed but unknown (see Definition~1 in \cite{tang14:_semipar}) and furthermore, there need not exist any relationship between $\mathbf{X}_n$ and $\mathbf{X}_{n'}$ for $n \not = n'$. Sparseness of the graphs is thus implicit (see for example the condition on the minimum vertex's degree in Assumption~1 in \cite{tang14:_semipar}). The current paper, however, assumes that the rows of $\mathbf{X}_n$ are independently sampled according to a distribution $F$. As such, sparseness needs to be made explicit through the sparsity factor $\rho_n$. 

  \begin{remark}
  For ease of exposition, henceforth we shall on many occasions remove the subscript $n$ from the matrices $\mathbf{X}_n, \hat{\mathbf{X}}_n, \mathbf{A}_n, \mathbf{P}_n$ and other related matrices such as $\mathbf{U}_{\mathbf{A}_n}$, $\mathbf{U}_{\mathbf{P}_n}$, etc. The subsequent statements are thus to be intepreted as holding for sufficient large $n$. Since we are concerned with limit results, this should lead to minimal confusion. 
  \end{remark}

  We first note that Eq.~\eqref{eq:6} follows from Theorem~A.5 in \cite{tang14:_semipar}. More specifically, if $(\mathbf{X}, \mathbf{A}) \sim \mathrm{RDPG}(F)$ with sparsity factor $\rho_n$, then Theorem~A.5 in \cite{tang14:_semipar} yields 
  $$ \|\hat{\mathbf{X}} - \rho_n^{1/2} \mathbf{X} \mathbf{W} \|_{F} = \|(\mathbf{A} - \mathbf{P}) \mathbf{U}_{\mathbf{P}} \mathbf{S}_{\mathbf{P}}^{-1/2} \|_{F} + O_{\mathbb{P}}((n \rho_n)^{-1/2}).$$ Since $\mathbf{P} = \rho_{n} \mathbf{X} \mathbf{X}^{\top}$ we have $\mathbf{U}_{\mathbf{P}} \mathbf{S}_{\mathbf{P}}^{1/2} \mathbf{W} = \rho_n^{1/2} \mathbf{X}$ for some orthogonal matrix $\mathbf{W}$. Therefore,
  \begin{equation*}
  \begin{split}
  \|(\mathbf{A} - \mathbf{P}) \mathbf{U}_{\mathbf{P}} \mathbf{S}_{\mathbf{P}}^{-1/2} \|_{F} &=  \|(\mathbf{A} - \mathbf{P}) \mathbf{U}_{\mathbf{P}} \mathbf{S}_{\mathbf{P}}^{1/2} \mathbf{W} \mathbf{W}^{\top} \mathbf{S}_{\mathbf{P}}^{-1} \mathbf{W} \|_{F} \\
  &= \|(\mathbf{A} - \mathbf{P}) \rho_n^{1/2} \mathbf{X} (\rho_n \mathbf{X}^{\top} \mathbf{X})^{-1} \|_{F} \\
  &= \rho_n^{-1/2} \|(\mathbf{A} - \mathbf{P}) \mathbf{X} (\mathbf{X}^{\top} \mathbf{X})^{-1} \|_{F}.
  \end{split}
  \end{equation*}
  Eq.~\eqref{eq:6} is thus established. We now show Eq.~\eqref{eq:5} and Eq.~\eqref{eq:rho_n_decreasing}. We shall use the convention that, unless stated otherwise, expectation of a random variable dependent on $\mathbf{A}$ is taken with respect to $\mathbf{A}$ conditional on $\mathbf{P}$. Let $\zeta = \rho_n \|(\mathbf{A} - \mathbf{P}) \mathbf{U}_{\mathbf{P}} \mathbf{S}_{\mathbf{P}}^{-1/2} \|_{F}^{2}$. Then, conditional on $\mathbf{P}$, $\zeta$ is a linear function of the {\em indepedent} random variables $\{a_{ij} - p_{ij}\}_{i < j}$. Lemma~A.5 in \cite{tang14:_semipar} shows that $\zeta$ is tightly concentrated around its expectation $\mathbb{E}[\zeta]$. We then have 
  \begin{equation*}
  \begin{split}
  \mathbb{E}[\zeta] &=\mathbb{E}[\|(\mathbf{A} - \mathbf{P}) \mathbf{U}_{\mathbf{P}} \mathbf{S}_{\mathbf{P}}^{-1/2} \|_{F}^{2}] 
  \\&= \rho_n^{-1} \mathbb{E}[\|(\mathbf{A} - \mathbf{P}) \mathbf{X} (\mathbf{X}^{\top} \mathbf{X})^{-1}\|_{F}^{2}] 
  \\ &= \mathrm{tr} \, n (\mathbf{X}^{\top} \mathbf{X})^{-1} \Bigl( n ^{-2} \rho_n^{-1} \mathbf{X}^{\top} \mathbb{E}[(\mathbf{A} - \mathbf{P})^{2}] \, \mathbf{X} \Bigr) n (\mathbf{X}^{\top} \mathbf{X})^{-1} 
  \end{split}
  \end{equation*}
  Now, the $ij$-th entry of $(\mathbf{A} - \mathbf{P})^{2}$ is of the form $\sum_{k} (a_{ik} - p_{ik}) (a_{kj} - p_{kj})$. As the upper diagonal entries of $\mathbf{A}$ are independent conditional on $\mathbf{P}$, we have
  $$\mathbb{E}\Bigl[\, \sum_{k} (a_{ik} - p_{ik})(a_{kj} - p_{kj}) \,\Bigr] = \begin{cases} 0 & \text{if $i \not = j$} \\ \sum_{k \not = i} p_{kj}(1 - p_{kj}) & \text{if $i = j$} \end{cases}$$

By the strong law of large numbers,
$n^{-1} \mathbf{X}^{T} \mathbf{X}_n = n^{-1} \sum_{i} X_i X_i^{\top}$ converges to $\Delta = \mathbb{E}[X_1 X_1^{\top}]$ almost surely as $n \rightarrow \infty$.
Hence $n (\mathbf{X}^{\top} \mathbf{X})^{-1}$ converges to $\Delta^{-1}$ almost
surely. In addition,
\begin{equation*}
\begin{split}
n^{-2} \rho_n^{-1} \mathbf{X}^{T} \mathbb{E}[(\mathbf{A} - \mathbf{P})^2] \mathbf{X} &= n^{-2} \rho_n^{-1}
\sum_{i=1}^{n} \sum_{j \not = i} X_i X_i^{\top} p_{ik} (1 - p_{ik})
 \\ &= n^{-2} \rho_n^{-1} \sum_{i=1}^{n} \sum_{k \not = i} X_i X_i^{\top} (\rho_n X_i^{\top} X_k
- \rho_n^{2} (X_i^{\top} X_k)^2) \\
&= n^{-2} \sum_{i=1}^{n} \sum_{k \not = i} X_i X_i^{\top} (X_i^{\top}X_k
- \rho_n X_i^{\top} X_k X_k^{\top} X_i)
\end{split}
\end{equation*}
If $\rho_n = 1$ for all $n$, the above term converges to $\mathbb{E}[X_1 X_1^{\top}
  (X_1^{\top} 
\mu_{F} - X_1^{\top} \Delta_{F} X_1)]$
almost surely. When $\rho_n \rightarrow 0$, the above term converges to $\mathbb{E}[X_1 X_1^{\top} X_1^{\top} \mu_{F}]$
almost surely. Eq.~\eqref{eq:5} and Eq.~\eqref{eq:rho_n_decreasing} is thus established.

We now sketch the proof of Theorem~\ref{THM:NORMALITY-ASE}. We emphasize that Theorem~\ref{THM:NORMALITY-ASE} is a generalization of the corresponding result in \cite{athreya2013limit,sussman:_thesis}, the generalization being that Theorem~\ref{THM:NORMALITY-ASE} does not assume distinct eigenvalues of the matrix $\mathbb{E}[X X^{\top}]$ where $X \sim F$; distinct eigenvalues is a necessary assumption for the proof given in \cite{athreya2013limit,sussman:_thesis}. 

Let $a_{ij}$ and $p_{ij}$ denote the $ij$-th entry of $\mathbf{A}$ and $\mathbf{P}$. 
From Eq.~\eqref{eq:6}, by exchangeability of the collection $\{\mathbf{W}_n \hat{X}_j - \rho_n^{1/2} X_j\}_{j=1}^{n}$, for any fixed index $i$ we have
\begin{equation*}
\begin{split}
\sqrt{n}(\mathbf{W}_n \hat{X}_i - \rho_n^{1/2} X_i) &= \sqrt{n} \rho_n^{-1/2} (\mathbf{X}^{\top} \mathbf{X})^{-1} \sum_{j \not = i} (a_{ij} - p_{ij}) X_j + o_{\mathbb{P}}(1) \\
&= \rho_n^{-1/2} (n^{-1} \mathbf{X}^{\top} \mathbf{X})^{-1} \sum_{j \not = i} \frac{(a_{ij} - p_{ij})}{\sqrt{n}} X_j + o_{\mathbb{P}}(1) \\
&= (n^{-1} \mathbf{X}^{\top} \mathbf{X})^{-1} \sum_{j \not = i} \frac{(a_{ij} - \rho_n X_i^{\top} X_j)}{\sqrt{n \rho_n}} X_j + o_{\mathbb{P}}(1).
\end{split}
\end{equation*}
Now conditional on $X_i$, the quantity $\sum_{j \not = i} \tfrac{(a_{ij} - \rho_n X_i^{\top} X_j)}{\sqrt{n \rho_n}} X_j$ is a sum of independent and identically distributed mean $0$ random variables. Thus by the multivariate central limit theorem, conditioning on $X_i = x$ yields
$$ \sum_{j \not = i} \frac{(a_{ij} - \rho_n x^{\top} X_j)}{\sqrt{n \rho_n}} X_j \overset{\mathrm{d}}{\longrightarrow} \mathcal{N}(0, \mathbb{E}[X_1 X_1^{\top} (x^{\top} X_1 - \rho_n x^{\top} X_1 X_1^{\top} x)]). $$ 
Furthermore, since $n^{-1} \mathbf{X}^{\top} \mathbf{X} = n^{-1} \sum X_i X_i^{\top} \overset{\mathrm{a.s.}}{\longrightarrow} \Delta$ as $n \rightarrow \infty$, we have by Slutsky's theorem that
$$ \sqrt{n}(\mathbf{W}_n \hat{X}_i - \rho_n^{1/2} X_i) 
\overset{\mathrm{d}}{\longrightarrow} \mathcal{N}(0, \Delta^{-1} \mathbb{E}[X_1 X_1^{\top} (x^{\top} X_1 - \rho_n x^{\top} X_1 X_1^{\top} x)] \Delta^{-1}), $$
thereby establishing Theorem~\ref{THM:NORMALITY-ASE}.

\section{Proof of Theorem~\ref{THM:LSE} and Theorem~\ref{THM:NORMALITY-LSE}}
\label{sec:lse}
For ease of exposition, we present in Section~\ref{sec:proof-LSE-1} a proof of Theorem~\ref{THM:NORMALITY-LSE}, assuming Eq.~\eqref{eq:LSE-main1} in Theorem~\ref{THM:LSE} holds. We next derive, in Section~\ref{sec:proof-LSE-2}, Eq.~\eqref{eq:LSE-main1} in Theorem~\ref{THM:LSE}. We then show, in Section~\ref{sec:proof-LSE-3} that the Frobenius norms in Eq.~\eqref{eq:4} and Eq.~\eqref{eq:14} are tightly concentrated around their expectations. We complete the proof of Theorem~\ref{THM:LSE} by computing these expectations explicitly when $\rho_n \equiv 1$ and when $\rho_n \rightarrow 0$. 

\subsection{Proof of Theorem~\ref{THM:NORMALITY-LSE}}
\label{sec:proof-LSE-1}
Recall that we suppress the dependency on $n$ in the subscript of the matrices $\mathbf{A}_n, \mathbf{X}_n, \mathbf{P}_n$ and other related matrices. In addition, recall that $\tilde{\mathbf{X}} = \rho_n^{1/2} \mathbf{T}^{-1/2} \mathbf{X} = \mathrm{diag}(\mathbf{X} \mathbf{X}^{\top} \bm{1})^{-1/2} \mathbf{X}$. Eq.~\eqref{eq:LSE-main1} from Theorem~\ref{THM:LSE} then implies
   \begin{equation*}
     \breve{\mathbf{X}} \mathbf{W} - \tilde{\mathbf{X}} = \mathbf{T}^{-1/2}(\mathbf{A}
    - \mathbf{P}) \mathbf{T}^{-1/2} \tilde{\mathbf{X}} (\tilde{\mathbf{X}}^{\top} \tilde{\mathbf{X}})^{-1} 
    + \frac{1}{2}\mathbf{T}^{-1}(\mathbf{T} - \mathbf{D}) \tilde{\mathbf{X}} 
    + \mathbf{R}
   \end{equation*}
   for some orthogonal matrix $\mathbf{W}$ and $n \times d$ matrix $\mathbf{R}$ with $\|\mathbf{R}\|_{F} = O_{\mathbb{P}}((n \rho_n)^{-1}).$ 
For a fixed index $i$, let $\zeta_i$ denotes the $i$-th row of
   $n \rho_n^{1/2} (\breve{\mathbf{X}} \mathbf{W} - \tilde{\mathbf{X}})$. Also let $r_i$ denote the $i$-th row of $\mathbf{R}$. 
   Now exchangeability of the $\{X_j\}_{j=1}^{n}$ implies exchangeability of the $\{\breve{X}_i\}_{j=1}^{n}$ and exchangeability of the $\{\tilde{X}_i\}_{j=1}^{n}$. This also implies exchangeability of the $\{\zeta_j\}_{j=1}^{n}$ and thus exchangeability of the $\{r_j\}_{j=1}^{n}$. Now, for any fixed index $i$, by exchangeability of the $\{r_j\}_{j=1}^{n}$, we have 
   $$n^{2} \rho_n \mathbb{E}[\|r_i\|^2] = n^{2} \rho_n \frac{1}{n} \mathbb{E}[\sum_{j} \|r_j\|^2] = n \rho_n \mathbb{E}[\|\mathbf{R}\|_{F}^2]$$
   Now, with probability at least $1 - n^{-3}$, $\|\mathbf{R}\|_{F} \leq C_0 (n \rho_n)^{-1})$ for some constant $C_0$. In addition, $\|\mathbf{R}\|_{F} \leq n$ almost surely. Thus $ \mathbb{E}[\|\mathbf{R}\|_{F}^2] \leq C_0^2 (n \rho_n)^{-2} (1 - n^{-3}) + n \times n^{-3} = O((n \rho_n)^{-2})$. Therefore $n^{2} \rho_n \mathbb{E}[\|r_i\|^2] = O((n \rho_n)^{-1})$. Since $n \rho_n = \omega(\log^{4}(n))$, we therefore have $n^{2} \rho_n \mathbb{E}[\|r_i\|^2] \rightarrow 0$ as $n \rightarrow \infty$, i.e., $n \rho_n^{1/2} r_i \overset{\mathrm{d}}{\rightarrow} 0$ as $n \rightarrow \infty$. 

   Let $a_{ij}$ and $p_{ij}$ denote the $ij$-th entry of $\mathbf{A}$ and $\mathbf{P}$, respectively. The above reasoning implies that for a fixed index $i$, $\zeta_i$ is of the form
   \begin{equation*}
     \begin{split}
     \zeta_i &= (\tilde{\mathbf{X}}^{\top} \tilde{\mathbf{X}})^{-1} \frac{n \rho_n^{1/2}}{\sqrt{t_i}} \Bigl( \sum_{j} \frac{a_{ij} -
     p_{ij}}{\sqrt{t_j}} \tilde{X}_j \Bigr) + \frac{n \rho_n^{1/2} (t_i - d_i)}{2t_i} \tilde{X}_i + o_{\mathbb{P}}(1) \\
   &= (\tilde{\mathbf{X}}^{\top} \tilde{\mathbf{X}})^{-1}
   \frac{\sqrt{n \rho_n}}{\sqrt{t_i}} \Bigl( \sum_{j \not =
   i} \frac{\sqrt{n \rho_n} (a_{ij} - p_{ij})X_j}{ t_j} \Bigr) - \frac{(n \rho_n)^{3/2} X_i}{2 t_i^{3/2}} \sum_{j \not =
   i} \frac{(a_{ij} - p_{ij})}{\sqrt{n \rho_n}} + o_{\mathbb{P}}(1).
     \end{split}
   \end{equation*}
   We first note that $\tilde{\mathbf{X}}^{\top} \tilde{\mathbf{X}} = \mathbf{X}^{\top} \mathrm{diag}(\mathbf{X} \mathbf{X} \bm{1})^{-1} \mathbf{X}$ converges almost surely to
   $\tilde{\Delta}$ as $n \rightarrow \infty$. This can be seen as follows. Denoting $\mu = \mathbb{E}[X_1]$, we have
\begin{equation*}
  \begin{split}
  \tilde{\mathbf{X}}^{\top} \tilde{\mathbf{X}} &= \sum_{i=1}^{n}
  \frac{X_i X_i^{\top}}{\sum_{j} X_i^{\top} X_j} 
  \\ &= \Bigl(\sum_{i=1}^{n} \frac{X_i X_i^{\top}}{n
    X_i^{\top} \mu}\Bigr) + \sum_{i=1}^{n} X_i X_i^{\top}
  \Bigl(\frac{1}{\sum_{j} X_i^{\top} X_j} - \frac{1}{n X_i^{\top} \mu}
  \Bigr) \\
  & = \Bigl(\sum_{i=1}^{n} \frac{X_i X_i^{\top}}{n
    X_i^{\top} \mu}\Bigr) + \sum_{i=1}^{n} \frac{X_i X_i^{\top}}{n
    X_i^{\top} \mu} \Bigl(\frac{n X_i^{\top} \mu - \sum_{j} X_i^{\top}
    X_j}{\sum_{j} X_i^{\top} X_j}\Bigr).
  \end{split}
\end{equation*}
Now, for any index $i$, let $c_i = |(n X_i^{\top} \mu - \sum_{j} X_i^{\top} X_j)/
(\sum_{j} X_i^{\top} X_j)|$. Then by Hoeffding's inequality, $c_i = O_{\mathbb{P}}(n^{-1/2})$. As $X_i X_i^{\top}$ is positive semidefinite for each
index $i$, we thus have
\begin{equation*}
  - c_i \frac{X_i X_i^{\top}}{n X_i^{\top} \mu} \preceq \frac{X_i X_i^{\top}}{n
    X_i^{\top} \mu} \Bigl(\frac{n X_i^{\top} \mu - \sum_{j} X_i^{\top}
    X_j}{\sum_{j} X_i^{\top} X_j}\Bigr) \preceq c_i \frac{X_i X_i^{\top}}{n X_i^{\top} \mu} 
\end{equation*}
where $\preceq$ denotes the positive semidefinite ordering of matrices.
Hence \begin{equation*}
\begin{split}
  - (\sup_{j \in [n]} c_j) \sum_{i} \frac{X_i X_i^{\top}}{n
    X_i^{\top} \mu} & \preceq \sum_{i} \frac{X_i X_i^{\top}}{n
    X_i^{\top} \mu} \Bigl(\frac{n X_i^{\top} \mu - \sum_{j} X_i^{\top}
    X_j}{\sum_{j} X_i^{\top} X_j}\Bigr) \\ & \preceq (\sup_{j \in [n]} c_j)
  \sum_{i} \frac{X_i X_i^{\top}}{n
    X_i^{\top} \mu}.
    \end{split}
\end{equation*}
We then have by a union bound that $\sup_{i \in [n]} c_i = O_{\mathbb{P}}(\sqrt{n^{-1} \log{n}})$
and hence $\sup_{i \in [n]} c_i \overset{\mathrm{a.s.}}{\rightarrow} 0$ as $n
\rightarrow \infty$. In addition, by the strong law of large numbers
\begin{equation}
  \label{eq:11}
  \sum_{i} \frac{X_i X_i^{\top}}{n
    X_i^{\top} \mu} \overset{\mathrm{a.s.}}{\longrightarrow} \mathbb{E}\Bigl[ \frac{X_1
    X_1^{\top}}{X_1^{\top} \mu} \Bigr]
\end{equation}
as $n \rightarrow \infty$. Thus,
$$\sum_{i} \frac{X_i X_i^{\top}}{n
    X_i^{\top} \mu} \Bigl(\frac{n X_i^{\top} \mu - \sum_{j} X_i^{\top}
    X_j}{\sum_{j} X_i^{\top} X_j}\Bigr) \overset{\mathrm{a.s.}}{\longrightarrow} 0
$$ as $n \rightarrow \infty$. We thus conclude that 
\begin{equation}
\label{eq:tildeX-outer}
   \tilde{\mathbf{X}}^{\top} \tilde{\mathbf{X}} = \Bigl(\sum_{i=1}^{n} \frac{X_i X_i^{\top}}{n
    X_i^{\top} \mu}\Bigr) \, + \sum_{i=1}^{n} \frac{X_i X_i^{\top}}{n
    X_i^{\top} \mu} \frac{n X_i^{\top} \mu - \sum_{j} X_i^{\top}
    X_j}{\sum_{j} X_i^{\top} X_j}
  \overset{\mathrm{a.s.}}{\longrightarrow} \mathbb{E} \Bigl[ \frac{X_1
    X_1^{\top}}{X_1^{\top} \mu} \Bigr]  
\end{equation}
as $n \rightarrow \infty$.

Therefore $(\tilde{\mathbf{X}}^{\top} \tilde{\mathbf{X}})^{-1}$ converges almost
   surely to $\tilde{\Delta}^{-1}$ as $n \rightarrow \infty$. In
   addition, $t_i/(n \rho_n) \rightarrow X_i^{\top} \mu$ as $n \rightarrow
   \infty$ and hence $\sqrt{n \rho_n /t_i} \rightarrow (X_i^{\top} \mu)^{-1/2}$
   as $n \rightarrow \infty$. 
   We next consider the term
   \begin{equation*}
     \begin{split}
   \sum_{j \not = i}
     \frac{\sqrt{n \rho_n} (a_{ij} - p_{ij})X_j}{t_j} &=
     \sum_{j \not = i} \frac{(a_{ij} -
       p_{ij}) X_j}{\sqrt{n \rho_n} X_j^{\top} \mu} + 
     \sum_{j \not = i} \frac{(a_{ij} - p_{ij}) X_j}{
       \sqrt{n \rho_n} X_j^{\top} \mu} \frac{n \rho_n X_j^{\top} \mu - t_j}{t_j}
    \end{split}
   \end{equation*}
   The second sum on the right hand side of the above display is,
   conditioned on $\mathbf{P}$, a sum of mean $0$ random
   variables. Hoeffding's inequality implies that the event 
     \begin{equation*}
      \bigl \| \sum_{j \not = i} \frac{(a_{ij} - p_{ij}) X_j}{
       \sqrt{n \rho_n} X_j^{\top} \mu} \frac{n \rho_n X_j^{\top} \mu - t_j}{t_j} \bigr \| \geq s 
       \end{equation*}
       occurs with probability at most
       $$
      2 \exp\Bigl(\frac{-C n \rho_n s^{2}}{\sum_{j \not = i}
       \|X_j\|^{2}(n \rho_n X_j^{\top} \mu
       - t_j)^{2}(X_j^{\top}\mu)^{-2} t_j^{-2}}\Bigr)
     $$
     for some constant $C > 0$. 
    Therefore,
     \begin{equation*}
     \sum_{j \not = i} \frac{(a_{ij} - p_{ij}) X_j}{\sqrt{n \rho_n}
       X_j^{\top} \mu} \frac{n \rho_n X_j^{\top} \mu - t_j}{t_j}
     \overset{\mathrm{a.s.}}{\longrightarrow} 0
     \end{equation*}
     as $n \rightarrow \infty$. 
     We thus have
     \begin{equation}
     \label{eq:clt-lse-proof-1}
     \zeta_i = (\tilde{\mathbf{X}}^{\top} \tilde{\mathbf{X}})^{-1}
   \frac{\sqrt{n \rho_n}}{\sqrt{t_i}} \Bigl( \sum_{j \not =
   i} \frac{(a_{ij} - p_{ij})X_j}{\sqrt{n \rho_n} X_j^{\top} \mu} \Bigr) - \frac{n \rho_n \sqrt{n \rho_n} X_i}{2 t_i \sqrt{t_i}} \sum_{j \not =
   i} \frac{(a_{ij} - p_{ij})}{\sqrt{n \rho_n}} + o_{\mathbb{P}}(1).
   \end{equation}
   We now show that 
   \begin{equation} 
   \label{eq:clt-lse-proof-2}
    \frac{n \rho_n \sqrt{n \rho_n} X_i}{2 t_i \sqrt{t_i}} \sum_{j \not =
   i} \frac{(a_{ij} - p_{ij})}{\sqrt{n \rho_n}} = \frac{\sqrt{n \rho_n}(\tilde{\mathbf{X}}^{\top} \tilde{\mathbf{X}})^{-1} \tilde{\Delta} X_i}{2 \sqrt{t_i} X_i^{\top} \mu} \sum_{j \not =
   i} \frac{(a_{ij} - p_{ij})}{\sqrt{n \rho_n}} + o_{\mathbb{P}}(1).
   \end{equation}
   This can be done as follows. We first consider the term 
   \begin{equation*}
   \frac{n \rho_n \sqrt{n \rho_n} X_i}{2 t_i \sqrt{t_i}} \sum_{j \not =
   i} \frac{(a_{ij} - p_{ij})}{\sqrt{n \rho_n}} = \frac{\sqrt{n \rho_n} }{2 \sqrt{t_i}} \Bigl(\sum_{j \not =
   i} \frac{(a_{ij} - p_{ij})}{\sqrt{n \rho_n}} \Bigr) \Bigl(\frac{X_i}{X_i^{\top} \mu} + \frac{n \rho_n X_i}{t_i} - \frac{X_i}{X_i^{\top} \mu}\Bigr)
   \end{equation*}
   Once again, conditional on $\mathbf{P}$, 
   $$ \frac{\sqrt{n \rho_n}}{2\sqrt{t_i}} \Bigl(\sum_{j \not =
   i} \frac{(a_{ij} - p_{ij})}{\sqrt{n \rho_n}} \Bigr) \Bigl(\frac{n \rho_n X_i}{t_i} - \frac{X_i}{X_i^{\top} \mu}\Bigr)$$
   is a sum of mean $0$ random variable. Hence, by Hoeffding's inequality, we also have that
   $$ \frac{\sqrt{n} \rho_n}{2 \sqrt{t_i}} \Bigl( \sum_{j \not =
   i} \frac{(a_{ij} - p_{ij})}{\sqrt{n \rho_n}} \Bigr) \Bigl(\frac{n \rho_n X_i}{t_i} - \frac{X_i}{X_i^{\top} \mu}\Bigr) \overset{\mathrm{a.s.}}{\longrightarrow} 0$$
   as $n \rightarrow \infty$. We thus have
   \begin{equation}
   \begin{split}
   \label{eq:clt-lse-proof-3}
   \frac{n \rho_n \sqrt{n \rho_n} X_i}{2 t_i \sqrt{t_i}}  \sum_{j \not =
   i} \frac{(a_{ij} - p_{ij})}{\sqrt{n \rho_n}} &= \frac{\sqrt{n \rho_n}}{2 \sqrt{t_i}} \sum_{j \not =
   i} \frac{(a_{ij} - p_{ij})X_i}{\sqrt{n \rho_n}X_i^{\top} \mu}  + o_{\mathbb{P}}(1). 
   \end{split}
   \end{equation}
   We next write
   $$ \frac{\sqrt{n \rho_n}}{2 \sqrt{t_i}} \sum_{j \not =
   i} \frac{(a_{ij} - p_{ij})X_i}{\sqrt{n \rho_n}X_i^{\top} \mu} = 
\frac{\sqrt{n \rho_n}(\tilde{\mathbf{X}}^{\top} \tilde{\mathbf{X}})^{-1}}{2 \sqrt{t_i}} \sum_{j \not = i}
   \frac{(a_{ij} - p_{ij})(\tilde{\Delta} + \tilde{\mathbf{X}}^{\top} \tilde{\mathbf{X}} - \tilde{\Delta})X_i}{\sqrt{n \rho_n} X_i^{\top} \mu}.$$

   We again evoke Hoeffding's inequality conditionally on $\mathbf{P}$ to conclude that
   \begin{equation}
   \label{eq:clt-lse-proof-4}
    \frac{\sqrt{n \rho_n}(\tilde{\mathbf{X}}^{\top} \tilde{\mathbf{X}})^{-1}}{2 \sqrt{t_i}} \sum_{j \not =
   i} \frac{(a_{ij} - p_{ij})(\tilde{\mathbf{X}}^{\top} \tilde{\mathbf{X}} - \tilde{\Delta})X_i}{\sqrt{n \rho_n} X_i^{\top} \mu} \overset{\mathrm{a.s.}}{\longrightarrow} 0
   \end{equation} as $n \rightarrow \infty$. Eq.~\eqref{eq:clt-lse-proof-2} then follows from Eq.~\eqref{eq:clt-lse-proof-3} and Eq.~\eqref{eq:clt-lse-proof-4}. 

   Combining Eq.~\eqref{eq:clt-lse-proof-1} and Eq.~\eqref{eq:clt-lse-proof-2}, we arrive at
   \begin{equation*}
   \begin{split}
   \zeta_i &= (\tilde{\mathbf{X}}^{\top} \tilde{\mathbf{X}})^{-1} \frac{\sqrt{n \rho_n}}{\sqrt{t_i}} \Bigl(\sum_{j \not =
   i} \frac{(a_{ij} - p_{ij})}{\sqrt{n \rho_n}} \Bigl(\frac{X_j}{X_j^{\top} \mu} - \frac{\tilde{\Delta} X_i}{2 X_i^{\top} \mu} \Bigr)\Bigr) + o_{\mathbb{P}}(1) \\
   &= (\tilde{\mathbf{X}}^{\top} \tilde{\mathbf{X}})^{-1} \frac{\sqrt{n \rho_n}}{\sqrt{t_i}} \Bigl(\sum_{j \not =
   i} \frac{(a_{ij} - \rho_n X_i^{\top} X_j)}{\sqrt{n \rho_n}} \Bigl(\frac{X_j}{X_j^{\top} \mu} - \frac{\tilde{\Delta} X_i}{2 X_i^{\top} \mu} \Bigr)\Bigr) + o_{\mathbb{P}}(1).
   \end{split} 
   \end{equation*}
   Now, for each fixed index $i$, conditioning on $X_i = x$, the quantity
   \begin{equation}
   \label{eq:clt_lse_tmp1}
    \frac{1}{\sqrt{n \rho_n}} \sum_{j \not = i} (a_{ij} - \rho_n X_i^{\top} X_j) \Bigl(\frac{X_j}{X_j^{\top} \mu} - \frac{\tilde{\Delta} X_i}{2 X_i^{\top} \mu} \Bigr) 
    \end{equation}
   is a sum of independent and identically distributed mean $0$ random variables. Therefore, by the multivariate central limit theorem, we have that conditional on $X_i = x$, the term in Eq.~\eqref{eq:clt_lse_tmp1} converges in distribution to
   $$ \mathcal{N}\Bigl(0, \mathbb{E}\Bigl[\Bigl(\frac{X_j}{X_j^{\top} \mu} - \frac{\tilde{\Delta} x}{2 x^{\top} \mu}\Bigr) (x^{\top} X_j - \rho_n x^{\top} X_j X_j^{\top} x) \Bigl(\frac{X_j}{X_j^{\top} \mu} - \frac{\tilde{\Delta} x}{2 x^{\top} \mu}\Bigr)^{\top} \Bigr]\Bigr). $$ 
   Finally, recall that $(\tilde{\mathbf{X}}^{\top} \tilde{\mathbf{X}})^{-1}$ and $\sqrt{n \rho_n/t_i}$ converge almost
   surely to $\tilde{\Delta}^{-1}$ and $(X_i^{\top} \mu)^{-1/2}$ as $n \rightarrow \infty$. Therefore, by Slutsky's theorem, conditional on $X_i = x$, $\zeta_i = n \rho_n^{1/2} (\mathbf{W}
   \breve{X}_i - \tilde{X}_i)$ converges in distribution to  
      \begin{equation*}
       \begin{split}
    \mathcal{N}\Bigl(0, \mathbb{E}\Bigl[\Bigl(\frac{\tilde{\Delta}^{-1} X_j}{X_j^{\top} \mu} - \frac{x}{2 x^{\top} \mu}\Bigr) \Bigl(\frac{x^{\top} X_j - \rho_n x^{\top} X_j X_j^{\top} x}{x^{\top} \mu}\Bigr) \Bigl(\frac{\tilde{\Delta}^{-1} X_j}{X_j^{\top} \mu} - \frac{x}{2 x^{\top} \mu} \Bigr)^{\top} \Bigr] \Bigr)
       \end{split}
     \end{equation*}
     as desired.

\subsection{Proof of Eq.~\eqref{eq:LSE-main1}}
\label{sec:proof-LSE-2}
We start with a concentration inequality for the spectral norm of $\mathbf{A} - \mathbf{P}$ and $\mathcal{L}(\mathbf{A}) - \mathcal{L}(\mathbf{P})$ in the case when $\mathbf{A}$ is an edge-independent inhomogenous random graph. 

\begin{lemma}[\cite{oliveira2009concentration,lu13:_spect}]
  \label{lem:lu-peng}
  Let $\mathbf{A} \sim \mathrm{Bernoulli}(\mathbf{P})$, i.e., $\mathbf{A}$ is a symmetric matrix whose upper triangular entries are independent Bernoulli random variables with $\mathbb{P}[a_{ij} = 1] = p_{ij}$.
  Let $\Delta = \max_{i} \sum_{j \not = i} p_{ij}$ and $\delta = \min_{i}
  \sum_{j \not = i} p_{ij}$ denotes the maximum and minimum row sums of $\mathbf{P}$. Suppose $\delta$ satisfies $\delta \gg
  \log^{4}(n)$. Then
  \begin{gather*}
  \|\mathbf{A} - \mathbf{P}\| = O_{\mathbb{P}}(\sqrt{\Delta}), \\
    \|\mathcal{L}(\mathbf{A}) - \mathcal{L}(\mathbf{P}) \| = O_{\mathbb{P}}(\delta^{-1/2}).
  \end{gather*}
\end{lemma}

When $\mathbf{P} = \rho_n \mathbf{X}
\mathbf{X}^{\top}$ then $\delta$ and $\Delta$ are both of order $\Theta(n
\rho_n)$. Furthermore, the non-zero eigenvalues of $\mathbf{P}$ are
all of order $\Theta(n \rho_n)$ while the non-zero eigenvalues of
$\mathcal{L}(\mathbf{P})$ are all of order $\Theta(1)$. 
In light of Lemma~\ref{lem:lu-peng}, for our subsequent
derivation, we shall assume that $\rho_n = \omega(\log^{k}(n))$ for
some positive integer $k \geq 4$. 

Lemma~\ref{lem:lu-peng} implies the following proposition. 
\begin{proposition}
  \label{prop:2}
  Let $(\mathbf{A}, \mathbf{X}) \sim \mathrm{RDPG}(F)$ with sparsity factor $\rho_n$. Let
  $\mathbf{W}_1 \bm{\Sigma} \mathbf{W}_2^{\top}$ be the singular value
  decomposition of $\tilde{\mathbf{U}}_{\mathbf{P}}^{\top}
  \tilde{\mathbf{U}}_{\mathbf{A}}$. 
  Then
  \begin{equation*}
    \|\tilde{\mathbf{U}}_{\mathbf{P}}^{\top} \tilde{\mathbf{U}}_{\mathbf{A}} -
    \mathbf{W}_1 \mathbf{W}_2^{\top} \|_{F} = O_{\mathbb{P}}((n \rho_n)^{-1}).
  \end{equation*}
\end{proposition}
\begin{proof}
  Let $\sigma_1, \sigma_2, \dots, \sigma_d$ denote the singular values of
  $\tilde{\mathbf{U}}_{\mathbf{P}}^{\top} \tilde{\mathbf{U}}_{\mathbf{A}}$ (the diagonal
  entries of $\bm{\Sigma}$). Then $\sigma_i = \cos(\theta_i)$ where
  the $\theta_i$ are the principal angles between the subspaces
  spanned by $\tilde{\mathbf{U}}_{\mathbf{A}}$ and
  $\mathbf{U}_{\mathbf{P}}$. Furthermore, by the Davis-Kahan
  $\sin(\Theta)$ theorem (see e.g., Theorem 3.6 in \cite{stewart90:_matrix}),
  \begin{equation*}
    \| \tilde{\mathbf{U}}_{\mathbf{A}} \tilde{\mathbf{U}}_{\mathbf{A}}^{\top} -
    \tilde{\mathbf{U}}_{\mathbf{P}} \tilde{\mathbf{U}}_{\mathbf{P}}^{\top} \| =
    \max_{i} | \sin(\theta_i) | \leq
    \frac{\|\mathcal{L}(\mathbf{A}) - \mathcal{L}(\mathbf{P})\|}{\lambda_{d}(\mathcal{L}(\mathbf{P}))} = O_{\mathbb{P}}((n \rho_n)^{-1/2}).
  \end{equation*}
  Here $\lambda_{d}(\mathcal{L}(\mathbf{P}))$ denotes 
  the $d$ largest eigenvalue of $\mathcal{L}(\mathbf{P})$.  We thus have
  \begin{equation*}
    \begin{split}
    \|\tilde{\mathbf{U}}_{\mathbf{P}}^{\top} \tilde{\mathbf{U}}_{\mathbf{A}} -
    \mathbf{W}_1 \mathbf{W}_2^{\top} \|_{F} = \|\bm{\Sigma} -
    \mathbf{I} \|_{F} &= \Bigl(\sum_{i=1}^{d} (1 - \sigma_i)^2\Bigr)^{1/2}  \\ & \leq \sum_{i=1}^{d} (1 -
    \sigma_i^{2}) = \sum_{i=1}^{d} \sin^{2}(\theta_i).
  \end{split}
  \end{equation*}
  Thererfore $\|\tilde{\mathbf{U}}_{\mathbf{P}}^{\top} \tilde{\mathbf{U}}_{\mathbf{A}} -
    \mathbf{W}_1 \mathbf{W}_2^{\top} \|_{F}
     = O_{\mathbb{P}}((n \rho_n)^{-1})
  $ as desired. 
\end{proof}
From now on, we shall denote by $\mathbf{W}^{*}$
the orthogonal matrix $\mathbf{W}_1 \mathbf{W}_2^{\top}$ as defined in the above
proposition. Next, we state the following lemma.
\begin{lemma}
  \label{lem:2}
     Let $(\mathbf{A}, \mathbf{X}) \sim \mathrm{RDPG}(F)$ with sparsity factor $\rho_n$. Then 
  \begin{gather}
  \label{eq:lem2a}
    n \rho_n \|\tilde{\mathbf{U}}_{\mathbf{P}}^{\top}
    \tilde{\mathbf{U}}_{\mathbf{A}} \tilde{\mathbf{S}}_{\mathbf{A}} -
    \tilde{\mathbf{S}}_{\mathbf{P}} \tilde{\mathbf{U}}_{\mathbf{P}}^{\top} \tilde{\mathbf{U}}_{\mathbf{A}} \| = O_{\mathbb{P}}(1), \\
    \label{eq:lem2b}
    n \rho_n \| \tilde{\mathbf{U}}_{\mathbf{P}}^{\top}
    \tilde{\mathbf{U}}_{\mathbf{A}} \tilde{\mathbf{S}}_{\mathbf{A}}^{1/2} -
    \tilde{\mathbf{S}}_{\mathbf{P}}^{1/2}
    \tilde{\mathbf{U}}_{\mathbf{P}}^{\top}
    \tilde{\mathbf{U}}_{\mathbf{A}}   \| = O_{\mathbb{P}}(1), \\
    \label{eq:lem2c}
    n \rho_n \| \tilde{\mathbf{U}}_{\mathbf{P}}^{\top}
    \tilde{\mathbf{U}}_{\mathbf{A}} \tilde{\mathbf{S}}_{\mathbf{A}}^{-1/2} -
    \tilde{\mathbf{S}}_{\mathbf{P}}^{-1/2}
    \tilde{\mathbf{U}}_{\mathbf{P}}^{\top}
    \tilde{\mathbf{U}}_{\mathbf{A}}   \| = O_{\mathbb{P}}(1). 
  \end{gather}
\end{lemma}
In proving Lemma~\ref{lem:2}, we need the following technical result. Lemma~\ref{lem:2} and Lemma~\ref{lemma:quadratic-up-ua} are the key technical lemmas of this paper. Roughly speaking, Lemma~\ref{lem:2} along with Proposition~\ref{prop:2} allows us to interchange the order of the orthogonal transformation $\mathbf{W}^{*}$ with the diagonal scaling matrices $\mathbf{S}_{\mathbf{A}}$ or $\mathbf{S}_{\mathbf{A}}$; Lemma~\ref{lemma:quadratic-up-ua} simplifies various expressions involving $\mathbf{A}, \mathbf{D}, \mathcal{L}(\mathbf{A})$ and $\tilde{\mathbf{U}}_{\mathbf{A}}$. 
\begin{lemma}
  \label{lemma:quadratic-up-ua}
  Let $(\mathbf{A}, \mathbf{X}) \sim \mathrm{RDPG}(F)$ with sparsity factor $\rho_n$. Then the following holds simultaneously
  \begin{gather}
  \label{eq:uLA-LPu}
  \mathbf{D}^{-1/2} - \mathbf{T}^{-1/2} = \tfrac{1}{2} \mathbf{T}^{-3/2} (\mathbf{T} - \mathbf{D}) + O_{\mathbb{P}}((n \rho_n)^{-3/2}), \\
  \label{eq:uLA-LPu1b}
  \mathcal{L}(\mathbf{A}) = \mathbf{T}^{-1/2}(\mathbf{A} - \mathbf{P}) \mathbf{T}^{-1/2} + \mathbf{D}^{-1/2} \mathbf{P} \mathbf{D}^{-1/2} + O_{\mathbb{P}}((n \rho_n)^{-1}),
  \end{gather}
  \begin{equation}
  \begin{split}
  \label{eq:uLA-LPu2}
    \mathbf{D}^{-1/2}
      \mathbf{P} \mathbf{D}^{-1/2} - \mathcal{L}(\mathbf{P}) &= 
      \tfrac{1}{2} \mathbf{T}^{-3/2} (\mathbf{T} - \mathbf{D})
      \mathbf{P} \mathbf{D}^{-1/2} \\ & + \tfrac{1}{2} \mathbf{T}^{-1/2}
       \mathbf{P} \mathbf{T}^{-3/2}(\mathbf{T} - \mathbf{D}) + O_{\mathbb{P}}((n \rho_n)^{-1}).
    \end{split}
  \end{equation}
  \begin{gather}
  \label{eq:uLA-LPu3a}
    \tilde{\mathbf{U}}_{\mathbf{A}} - \tilde{\mathbf{U}}_{\mathbf{P}} \tilde{\mathbf{U}}_{\mathbf{P}}^{\top} \tilde{\mathbf{U}}_{\mathbf{A}} = O_{\mathbb{P}}((n \rho_n)^{-1/2}). \\ 
      \label{eq:uLA-LPu3b}
       \mathbf{T}^{-1/2} \mathbf{P}
      \mathbf{T}^{-3/2}(\mathbf{T} - \mathbf{D})
      \tilde{\mathbf{U}}_{\mathbf{P}} = O_{\mathbb{P}}((n \rho_n)^{-1}), \\
        \label{eq:uLA-LPu3c}
      \tilde{\mathbf{U}}_{\mathbf{P}}^{\top} \mathbf{T}^{-3/2} (\mathbf{T} - \mathbf{D})
      \mathbf{P} \mathbf{D}^{-1/2} = O_{\mathbb{P}}((n \rho_n)^{-1}), \\
        \label{eq:uLA-LPu3d}
          \tilde{\mathbf{U}}_{\mathbf{P}}^{\top} (\mathcal{L}(\mathbf{A}) -
    \mathcal{L}(\mathbf{P})) \tilde{\mathbf{U}}_{\mathbf{P}} = O_{\mathbb{P}}((n \rho_n)^{-1}). 
\end{gather}
\end{lemma}

We continue with the proof of Eq.~\eqref{eq:LSE-main1}. Let $\bm{\Pi} = \tilde{\mathbf{U}}_{\mathbf{P}} \tilde{\mathbf{U}}_{\mathbf{P}}^{\top}$ and $\bm{\Pi}^{\perp} = \mathbf{I} - \bm{\Pi}$.
Proposition~\ref{prop:2} and Lemma~\ref{lem:2} then yield
  \begin{equation*}
    \begin{split}
      \tilde{\mathbf{U}}_{\mathbf{A}} \tilde{\mathbf{S}}_{\mathbf{A}}^{1/2}
     - \tilde{\mathbf{U}}_{\mathbf{P}} \tilde{\mathbf{S}}_{\mathbf{P}}^{1/2}
     \mathbf{W}^{*} & = 
     \tilde{\mathbf{U}}_{\mathbf{A}} \tilde{\mathbf{S}}_{\mathbf{A}}^{1/2}
     - \tilde{\mathbf{U}}_{\mathbf{P}} \tilde{\mathbf{S}}_{\mathbf{P}}^{1/2}
     \tilde{\mathbf{U}}_{\mathbf{P}}^{\top}
     \tilde{\mathbf{U}}_{\mathbf{A}}  + O_{\mathbb{P}}( (n \rho_n)^{-1}) \\
     &= \tilde{\mathbf{U}}_{\mathbf{A}} \tilde{\mathbf{S}}_{\mathbf{A}}^{1/2}
     - \tilde{\mathbf{U}}_{\mathbf{P}} 
     \tilde{\mathbf{U}}_{\mathbf{P}}^{\top}
     \tilde{\mathbf{U}}_{\mathbf{A}}
     \tilde{\mathbf{S}}_{\mathbf{A}}^{1/2}  + O_{\mathbb{P}}((n \rho_n)^{-1})
     \\
     &= \bm{\Pi}^{\perp} \mathcal{L}(\mathbf{A})
     \tilde{\mathbf{U}}_{\mathbf{A}} \tilde{\mathbf{S}}_{\mathbf{A}}^{-1/2}
     + O_{\mathbb{P}}((n \rho_n)^{-1}). 
     \end{split}
\end{equation*}
Since $\mathcal{L}(\mathbf{P}) = \tilde{\mathbf{U}}_\mathbf{P} \tilde{\mathbf{S}}_{\mathbf{P}} \tilde{\mathbf{U}}_{\mathbf{P}}^{\top}$, 
 $\bm{\Pi}^{\perp} \mathcal{L}(\mathbf{P})
     = \bm{0}$ and hence
     \begin{equation}
       \label{eq:7}
      \tilde{\mathbf{U}}_{\mathbf{A}} \tilde{\mathbf{S}}_{\mathbf{A}}^{1/2}
     - \tilde{\mathbf{U}}_{\mathbf{P}} \tilde{\mathbf{S}}_{\mathbf{P}}^{1/2}
     \mathbf{W}^{*} = \bm{\Pi}^{\perp} (\mathcal{L}(\mathbf{A})
     - \mathcal{L}(\mathbf{P}))
     \tilde{\mathbf{U}}_{\mathbf{A}}
     \tilde{\mathbf{S}}_{\mathbf{A}}^{-1/2} + O_{\mathbb{P}}((n \rho_n)^{-1}).
     \end{equation}
   In addition,
    \begin{equation*}
    \begin{split}
(\mathcal{L}(\mathbf{A})
     - \mathcal{L}(\mathbf{P}))
     \tilde{\mathbf{U}}_{\mathbf{A}} \tilde{\mathbf{S}}_{\mathbf{A}}^{-1/2} 
     & = (\mathcal{L}(\mathbf{A})
     - \mathcal{L}(\mathbf{P})) \bm{\Pi}^{\perp}
     \tilde{\mathbf{U}}_{\mathbf{A}}\tilde{\mathbf{S}}_{\mathbf{A}}^{-1/2}
     + (\mathcal{L}(\mathbf{A})
     - \mathcal{L}(\mathbf{P})) \bm{\Pi}
     \tilde{\mathbf{U}}_{\mathbf{A}} \tilde{\mathbf{S}}_{\mathbf{A}}^{-1/2}
      \\
     &= O_{\mathbb{P}}((n \rho_n)^{-1}) + (\mathcal{L}(\mathbf{A})
     - \mathcal{L}(\mathbf{P})) \bm{\Pi}
     \tilde{\mathbf{U}}_{\mathbf{A}} \tilde{\mathbf{S}}_{\mathbf{A}}^{-1/2},
     \end{split}
\end{equation*}
where we bound $(\mathcal{L}(\mathbf{A})
     - \mathcal{L}(\mathbf{P})) \bm{\Pi}^{\perp}
     \tilde{\mathbf{U}}_{\mathbf{A}} \tilde{\mathbf{S}}_{\mathbf{A}}^{-1/2}$ using Eq.~\eqref{eq:uLA-LPu3a} and the submultiplicatity of the spectral norm. 
Eq.~\eqref{eq:7} then implies
\begin{equation}
\label{eq:7b}
\begin{split}
      \tilde{\mathbf{U}}_{\mathbf{A}} \tilde{\mathbf{S}}_{\mathbf{A}}^{1/2}
     - \tilde{\mathbf{U}}_{\mathbf{P}} \tilde{\mathbf{S}}_{\mathbf{P}}^{1/2}
     \mathbf{W}^{*} &= 
\bm{\Pi}^{\perp} (\mathcal{L}(\mathbf{A})
     - \mathcal{L}(\mathbf{P}))      \tilde{\mathbf{U}}_{\mathbf{A}}
     \tilde{\mathbf{S}}_{\mathbf{A}}^{-1/2} + O_{\mathbb{P}}((n \rho_n)^{-1})
     \\ &= \bm{\Pi}^{\perp} (\mathcal{L}(\mathbf{A})
     - \mathcal{L}(\mathbf{P})) \bm{\Pi}
     \tilde{\mathbf{U}}_{\mathbf{A}}
     \tilde{\mathbf{S}}_{\mathbf{A}}^{-1/2} +  O_{\mathbb{P}}((n \rho_n)^{-1}).
     \end{split}
     \end{equation}
     By Eq.~\eqref{eq:uLA-LPu3d} and sub-multiplicativity of the Frobenius norm, we also have
     \begin{equation*}
  \begin{split}
    \bm{\Pi}^{\perp} (\mathcal{L}(\mathbf{A})
     - \mathcal{L}(\mathbf{P}))  \bm{\Pi} \tilde{\mathbf{U}}_{\mathbf{A}}
     \tilde{\mathbf{S}}_{\mathbf{A}}^{-1/2}   
     & = 
     (\mathcal{L}(\mathbf{A})
     - \mathcal{L}(\mathbf{P})) \bm{\Pi} \tilde{\mathbf{U}}_{\mathbf{A}}
     \tilde{\mathbf{S}}_{\mathbf{A}}^{-1/2}
     - \bm{\Pi} (\mathcal{L}(\mathbf{A})
     - \mathcal{L}(\mathbf{P})) \bm{\Pi} \tilde{\mathbf{U}}_{\mathbf{A}}
     \tilde{\mathbf{S}}_{\mathbf{A}}^{-1/2}\\
     &= (\mathcal{L}(\mathbf{A}) - 
      \mathcal{L}(\mathbf{P})) \bm{\Pi} \tilde{\mathbf{U}}_{\mathbf{A}}
     \tilde{\mathbf{S}}_{\mathbf{A}}^{-1/2}
     + O_{\mathbb{P}}((n \rho_n)^{-1}).
   \end{split}
  \end{equation*}
  Eq.~\eqref{eq:7b} then becomes 
\begin{equation}
\label{eq:7c}
  \begin{split}
  \tilde{\mathbf{U}}_{\mathbf{A}} \tilde{\mathbf{S}}_{\mathbf{A}}^{1/2}
     - \tilde{\mathbf{U}}_{\mathbf{P}} \tilde{\mathbf{S}}_{\mathbf{P}}^{1/2}
     \mathbf{W}^{*} &=
    \bm{\Pi}^{\perp} (\mathcal{L}(\mathbf{A})
     - \mathcal{L}(\mathbf{P}))  \bm{\Pi}    \tilde{\mathbf{U}}_{\mathbf{A}}
     \tilde{\mathbf{S}}_{\mathbf{A}}^{-1/2} + O_{\mathbb{P}}((n \rho_n)^{-1})
     \\ &= (\mathcal{L}(\mathbf{A})
     - \mathcal{L}(\mathbf{P})) \bm{\Pi}
     \tilde{\mathbf{U}}_{\mathbf{A}}
     \tilde{\mathbf{S}}_{\mathbf{A}}^{-1/2} + O_{\mathbb{P}}((n \rho_n)^{-1}) \\
     &= (\mathcal{L}(\mathbf{A})
     - \mathcal{L}(\mathbf{P})) \tilde{\mathbf{U}}_{\mathbf{P}}
     \tilde{\mathbf{U}}_{\mathbf{P}}^{\top} 
     \tilde{\mathbf{U}}_{\mathbf{A}}
     \tilde{\mathbf{S}}_{\mathbf{A}}^{-1/2} + O_{\mathbb{P}}((n \rho_n)^{-1})\\
     &= (\mathcal{L}(\mathbf{A})
     - \mathcal{L}(\mathbf{P})) \tilde{\mathbf{U}}_{\mathbf{P}}
     \mathbf{W}^{*}     \tilde{\mathbf{S}}_{\mathbf{A}}^{-1/2} + O_{\mathbb{P}}((n \rho_n)^{-1}) \\
     &= (\mathcal{L}(\mathbf{A})
     - \mathcal{L}(\mathbf{P})) \tilde{\mathbf{U}}_{\mathbf{P}}
     \tilde{\mathbf{S}}_{\mathbf{P}}^{-1/2} \mathbf{W}^{*} + O_{\mathbb{P}}((n \rho_n)^{-1}).
   \end{split}
  \end{equation}
  Recall from Eq.~\eqref{eq:uLA-LPu1b} the decomposition $$\mathcal{L}(\mathbf{A}) =
    \mathbf{T}^{-1/2}(\mathbf{A}
    - \mathbf{P}) \mathbf{T}^{-1/2} + \mathbf{D}^{-1/2}
    \mathbf{P} \mathbf{D}^{-1/2} + O_{\mathbb{P}}((n \rho_n)^{-1}).$$ Therefore, from Eq.~\eqref{eq:7c}, we have
    \begin{equation}
    \label{eq:7d}
    \begin{split}
            \tilde{\mathbf{U}}_{\mathbf{A}} \tilde{\mathbf{S}}_{\mathbf{A}}^{1/2}
     - \tilde{\mathbf{U}}_{\mathbf{P}} \tilde{\mathbf{S}}_{\mathbf{P}}^{1/2}
     \tilde{\mathbf{W}}^{*} = O_{\mathbb{P}}((n \rho_n)^{-1}) &+ \mathbf{T}^{-1/2}(\mathbf{A}
    - \mathbf{P}) \mathbf{T}^{-1/2} \tilde{\mathbf{U}}_{\mathbf{P}}
     \tilde{\mathbf{S}}_{\mathbf{P}}^{-1/2} \mathbf{W}^{*} \\ &+  (\mathbf{D}^{-1/2}
    \mathbf{P} \mathbf{D}^{-1/2} - \mathbf{T}^{-1/2} \mathbf{P}
    \mathbf{T}^{-1/2}) \tilde{\mathbf{U}}_{\mathbf{P}}
     \tilde{\mathbf{S}}_{\mathbf{P}}^{-1/2} \mathbf{W}^{*}.
     \end{split}
    \end{equation}
    We next recall from Eq.~\eqref{eq:uLA-LPu2} the decomposition
    \begin{equation*}
    \begin{split}
      \mathbf{D}^{-1/2}
      \mathbf{P} \mathbf{D}^{-1/2} - \mathbf{T}^{-1/2} \mathbf{P}
      \mathbf{T}^{-1/2} =
      \tfrac{1}{2} \mathbf{T}^{-3/2} (\mathbf{T} - \mathbf{D})
      \mathbf{P} \mathbf{D}^{-1/2} &+ \tfrac{1}{2} \mathbf{T}^{-1/2}
      \mathbf{P} \mathbf{T}^{-3/2}(\mathbf{T} - \mathbf{D}) \\ &+
      O_{\mathbb{P}}((n \rho_n)^{-3/2}). 
      \end{split}
    \end{equation*}
    In addition, we recall from Eq.~\eqref{eq:uLA-LPu3b} that
    \begin{equation*}
    \begin{split}
      \mathbf{T}^{-1/2} \mathbf{P}
      \mathbf{T}^{-3/2}(\mathbf{T} - \mathbf{D})
      \tilde{\mathbf{U}}_{\mathbf{P}} &= O_{\mathbb{P}}((n \rho_n)^{-1}).
      \end{split}
    \end{equation*}
    Eq.~\eqref{eq:7d} therefore reduces to
    \begin{equation}
    \label{eq:7e}
      \begin{split}
      \tilde{\mathbf{U}}_{\mathbf{A}} \tilde{\mathbf{S}}_{\mathbf{A}}^{1/2}
     - \tilde{\mathbf{U}}_{\mathbf{P}} \tilde{\mathbf{S}}_{\mathbf{P}}^{1/2}
    \mathbf{W}^{*}  = O_{\mathbb{P}}((n \rho_n)^{-1}) &+ \mathbf{T}^{-1/2}(\mathbf{A}
    - \mathbf{P}) \mathbf{T}^{-1/2} \tilde{\mathbf{U}}_{\mathbf{P}}
     \tilde{\mathbf{S}}_{\mathbf{P}}^{-1/2} \mathbf{W}^{*} \\ &+       \tfrac{1}{2} \mathbf{T}^{-3/2} (\mathbf{T} - \mathbf{D})
      \mathbf{P} \mathbf{D}^{-1/2} \tilde{\mathbf{U}}_{\mathbf{P}}
     \tilde{\mathbf{S}}_{\mathbf{P}}^{-1/2} \mathbf{W}^{*}. 
      \end{split} 
    \end{equation}
    Now \begin{equation*}
    \begin{split}
    \mathbf{T}^{-3/2} (\mathbf{T} - \mathbf{D})
      \mathbf{P} \mathbf{D}^{-1/2} &= \mathbf{T}^{-3/2} (\mathbf{T} - \mathbf{D})
      \mathbf{P} (\mathbf{D}^{-1/2} - \mathbf{T}^{-1/2} +
      \mathbf{T}^{-1/2}) \\ &= \mathbf{T}^{-3/2} (\mathbf{T} - \mathbf{D})
      \mathbf{P} \mathbf{T}^{-1/2} + O_{\mathbb{P}}((n \rho_n)^{-1}),
      \end{split}
      \end{equation*}
      and thus Eq.~\eqref{eq:7e} further simplifies to
      \begin{equation}
        \begin{split}
        \label{eq:7f}
         \tilde{\mathbf{U}}_{\mathbf{A}} \tilde{\mathbf{S}}_{\mathbf{A}}^{1/2}
     - \tilde{\mathbf{U}}_{\mathbf{P}} \tilde{\mathbf{S}}_{\mathbf{P}}^{1/2}
     \mathbf{W}^{*}  = O_{\mathbb{P}}((n \rho_n)^{-1}) &+ \mathbf{T}^{-1/2}(\mathbf{A}
    - \mathbf{P}) \mathbf{T}^{-1/2} \tilde{\mathbf{U}}_{\mathbf{P}}
     \tilde{\mathbf{S}}_{\mathbf{P}}^{-1/2} \mathbf{W}^{*} \\ &+ \tfrac{1}{2} \mathbf{T}^{-3/2} (\mathbf{T} - \mathbf{D})
      \mathbf{P} \mathbf{T}^{-1/2} \tilde{\mathbf{U}}_{\mathbf{P}}
     \tilde{\mathbf{S}}_{\mathbf{P}}^{-1/2} \mathbf{W}^{*}.
      \end{split}
      \end{equation}
      Since $\mathbf{T}$ and $\mathbf{D}$ are diagonal matrices, we note that
      \begin{equation*}
      \begin{split}
      \mathbf{T}^{-3/2} (\mathbf{T} - \mathbf{D}) \mathbf{P} \mathbf{T}^{-1/2} \tilde{\mathbf{U}}_{\mathbf{P}} \tilde{\mathbf{S}}_{\mathbf{P}}^{-1/2} \mathbf{W}^{*} & = \mathbf{T}^{-1} (\mathbf{T} - \mathbf{D}) \mathbf{T}^{-1/2} \mathbf{P} \mathbf{T}^{-1/2}
      \tilde{\mathbf{U}}_{\mathbf{P}} \tilde{\mathbf{S}}_{\mathbf{P}}^{-1/2} \mathbf{W}^{*} \\ & = 
\mathbf{T}^{-1} (\mathbf{T} - \mathbf{D}) \mathcal{L}(\mathbf{P}) \tilde{\mathbf{U}}_{\mathbf{P}} \tilde{\mathbf{S}}_{\mathbf{P}}^{-1/2} \mathbf{W}^{*}
     \\ & =  \mathbf{T}^{-1} (\mathbf{T} - \mathbf{D}) \tilde{\mathbf{U}}_{\mathbf{P}} \tilde{\mathbf{S}}_{\mathbf{P}} \tilde{\mathbf{S}}_{\mathbf{P}}^{-1/2} \mathbf{W}^{*} 
      \\
      &= \mathbf{T}^{-1} (\mathbf{T} - \mathbf{D}) \tilde{\mathbf{U}}_{\mathbf{P}} \tilde{\mathbf{S}}_{\mathbf{P}}^{1/2} \mathbf{W}^{*}.
      \end{split}
      \end{equation*}
      We therefore arrive at
      \begin{equation}
      \label{eq:LSE-main0}
        \begin{split}
         \tilde{\mathbf{U}}_{\mathbf{A}} \tilde{\mathbf{S}}_{\mathbf{A}}^{1/2}
     - \tilde{\mathbf{U}}_{\mathbf{P}} \tilde{\mathbf{S}}_{\mathbf{P}}^{1/2}
     \mathbf{W}^{*}  = O_{\mathbb{P}}((n \rho_n)^{-1}) &+  \mathbf{T}^{-1/2}(\mathbf{A}
    - \mathbf{P}) \mathbf{T}^{-1/2} \tilde{\mathbf{U}}_{\mathbf{P}}
     \tilde{\mathbf{S}}_{\mathbf{P}}^{-1/2} \mathbf{W}^{*} \\ &+       \tfrac{1}{2} \mathbf{T}^{-1} (\mathbf{T} - \mathbf{D})
     \tilde{\mathbf{U}}_{\mathbf{P}}
     \tilde{\mathbf{S}}_{\mathbf{P}}^{1/2} \mathbf{W}^{*}. 
        \end{split}
      \end{equation}
   To conclude the proof of Eq.~\eqref{eq:LSE-main1}, we recall that $\tilde{\mathbf{X}}
   \tilde{\mathbf{X}}^{\top} = \mathcal{L}(\mathbf{P}) =
   \tilde{\mathbf{U}}_{\mathbf{P}} \tilde{\mathbf{S}}_{\mathbf{P}}
   \tilde{\mathbf{U}}_{\mathbf{P}}^{\top}$; hence
   $\tilde{\mathbf{X}} = \tilde{\mathbf{U}}_{\mathbf{P}}
   \tilde{\mathbf{S}}_{\mathbf{P}}^{1/2} \tilde{\mathbf{W}}$ for some orthogonal
   matrix $\tilde{\mathbf{W}}$. Therefore
   \begin{gather*}
   \tilde{\mathbf{U}}_{\mathbf{P}}
     \tilde{\mathbf{S}}_{\mathbf{P}}^{1/2} \mathbf{W}^{*} = \tilde{\mathbf{U}}_{\mathbf{P}}
   \tilde{\mathbf{S}}_{\mathbf{P}}^{1/2} \tilde{\mathbf{W}} \tilde{\mathbf{W}}^{\top} \mathbf{W}^{*} = \tilde{\mathbf{X}} \tilde{\mathbf{W}}^{\top} \mathbf{W}^{*} \\ \tilde{\mathbf{U}}_{\mathbf{P}}
     \tilde{\mathbf{S}}_{\mathbf{P}}^{-1/2} \mathbf{W}^{*} =  \tilde{\mathbf{U}}_{\mathbf{P}}
     \tilde{\mathbf{S}}_{\mathbf{P}}^{1/2} \tilde{\mathbf{W}} \tilde{\mathbf{W}}^{\top} \tilde{\mathbf{S}}_{\mathbf{P}}^{-1} \tilde{\mathbf{W}} \tilde{\mathbf{W}}^{\top} \mathbf{W}^{*} = \tilde{\mathbf{X}} (\tilde{\mathbf{X}}^{\top} \tilde{\mathbf{X}})^{-1} \tilde{\mathbf{W}}^{\top} \mathbf{W}^{*}.
   \end{gather*}
   Substituting the above equations into Eq.~\eqref{eq:LSE-main0} yields 
   \begin{equation*}
   \begin{split}
     \tilde{\mathbf{U}}_{\mathbf{A}} \tilde{\mathbf{S}}_{\mathbf{A}}^{1/2}
     - \tilde{\mathbf{X}} \tilde{\mathbf{W}}^{\top}
     \mathbf{W}^{*}  = O_{\mathbb{P}}((n \rho_n)^{-1}) &+  \mathbf{T}^{-1/2}(\mathbf{A}
    - \mathbf{P}) \mathbf{T}^{-1/2} \tilde{\mathbf{X}} (\tilde{\mathbf{X}}^{\top} \tilde{\mathbf{X}})^{-1} \tilde{\mathbf{W}}^{\top} \mathbf{W}^{*}. \\ &+       \tfrac{1}{2} \mathbf{T}^{-1} (\mathbf{T} - \mathbf{D}) \tilde{\mathbf{X}} \tilde{\mathbf{W}}^{\top} \mathbf{W}^{*}
    \end{split}
   \end{equation*}
   Equivalently, 
   \begin{equation*}
   \begin{split}
     \tilde{\mathbf{U}}_{\mathbf{A}} \tilde{\mathbf{S}}_{\mathbf{A}}^{1/2} (\mathbf{W}^{*})^{\top} \tilde{\mathbf{W}} 
     - \tilde{\mathbf{X}}   = O_{\mathbb{P}}((n \rho_n)^{-1}) &+  \mathbf{T}^{-1/2}(\mathbf{A}
    - \mathbf{P}) \mathbf{T}^{-1/2} \tilde{\mathbf{X}} (\tilde{\mathbf{X}}^{\top} \tilde{\mathbf{X}})^{-1}  \\ &+       \tfrac{1}{2} \mathbf{T}^{-1} (\mathbf{T} - \mathbf{D}) \tilde{\mathbf{X}}. 
    \end{split}
   \end{equation*}
   Eq.~\eqref{eq:LSE-main1} is thereby established. 
   
   \subsection{Proof of Lemma~\ref{lem:2} and Lemma~\ref{lemma:quadratic-up-ua}}
   We first present the proof of Lemma~\ref{lemma:quadratic-up-ua}. We recall the notations $\mathbf{D} = \mathrm{diag}(\mathbf{A} \bm{1})$ and
  $\mathbf{T} = \mathrm{diag}(\mathbf{P} \bm{1})$. 
  Denote by $d_i$ and $t_i$ the $i$-th diagonal elements of
  $\mathbf{D}$ and $\mathbf{T}$. The $i$-th diagonal
  element of $\mathbf{D}^{-1/2} - \mathbf{T}^{-1/2}$ can be written as
  \begin{equation*}
    \begin{split}
    \frac{1}{\sqrt{d_i}} - \frac{1}{\sqrt{t_i}} &= \frac{t_i -
      d_i}{(\sqrt{d_i} + \sqrt{t_i})\sqrt{d_i} \sqrt{t_i}} \\ &= \frac{t_i
      - d_i}{2 t_i^{3/2}} + (t_i - d_i) \Bigl(\frac{1}{d_i
        \sqrt{t_i} + t_{i} \sqrt{d_i}} - \frac{1}{2t_i^{3/2}} \Bigr) \\
      &= \frac{t_i - d_i}{2 t_i^{3/2}} + (t_i - d_i)
      \frac{t_i(\sqrt{t_i} - \sqrt{d_i}) + (t_i - d_i) \sqrt{t_i})}{2
        t_i^{3/2}(d_i \sqrt{t_i} + t_i \sqrt{d_i})}.
    \end{split}
  \end{equation*}
  We have, by Chernoff's bound, 
  that $|t_i - d_i| = O_{\mathbb{P}}(\sqrt{n \rho_n})$ for any given index $i$, and hence
  $|\sqrt{t_i} - \sqrt{d_i}| = O_{\mathbb{P}}(1)$. Therefore, 
  \begin{equation*}
    (t_i - d_i)
      \frac{t_i(\sqrt{t_i} - \sqrt{d_i}) + (t_i - d_i) \sqrt{t_i})}{2
        t_i^{3/2}(d_i \sqrt{t_i} + t_i \sqrt{d_i})} = O_{\mathbb{P}}(\sqrt{n
        \rho_n}) \frac{O_{\mathbb{P}}(n \rho_n)}{\Omega_{\mathbb{P}}(n^{3} \rho_n^3)} = O_{\mathbb{P}}((n \rho_n)^{-3/2}).
  \end{equation*}
  Upon taking an union bound over all indices $i =
  1,2,\dots, n$, we have
  \begin{equation}
  \label{eq:sqrtD - sqrtT}
    \mathbf{D}^{-1/2} - \mathbf{T}^{-1/2} = \frac{1}{2}
    \mathbf{T}^{-3/2}(\mathbf{T} - \mathbf{D}) + O_{\mathbb{P}}((n \rho_n)^{-3/2} \log{n}).
  \end{equation}
  Eq.~\eqref{eq:uLA-LPu} is thereby established. Eq.~\eqref{eq:uLA-LPu2} follows directly from Eq.~\eqref{eq:uLA-LPu} and the definition of $\mathcal{L}(\mathbf{P}) = \mathbf{T}^{-1/2} \mathbf{P} \mathbf{T}^{-1/2}$.
  We next show Eq.~\eqref{eq:uLA-LPu1b}. Consider the following decomposition of $\mathcal{L}(\mathbf{A})$
  \begin{equation*}
    \begin{split}
    \mathcal{L}(\mathbf{A}) &= \mathbf{D}^{-1/2} (\mathbf{A} - \mathbf{P})
    \mathbf{D}^{-1/2} + \mathbf{D}^{-1/2} \mathbf{P} \mathbf{D}^{-1/2}  
    \\ &= \mathbf{T}^{-1/2} (\mathbf{A} - \mathbf{P}) \mathbf{T}^{-1/2} 
    + \mathbf{T}^{-1/2} (\mathbf{A} - \mathbf{P}) (\mathbf{D}^{-1/2} - \mathbf{T}^{-1/2}) \\ &+ (\mathbf{D}^{-1/2} - \mathbf{T}^{-1/2}) (\mathbf{A} - \mathbf{P}) \mathbf{D}^{-1/2} 
     + \mathbf{D}^{-1/2} \mathbf{P} \mathbf{D}^{-1/2}.
    \end{split}
  \end{equation*}
  By Lemma~\ref{lem:lu-peng}, we have
  \begin{equation}
  \label{eq:A-Ppart2}
    \|(\mathbf{A} -
    \mathbf{P}) \mathbf{T}^{-1/2} \| \leq \|\mathbf{A} -
    \mathbf{P} \| \times \|\mathbf{T}^{-1/2} \| = O_{\mathbb{P}}(1).
  \end{equation}
  Similarly, Lemma~\ref{lem:lu-peng} and Chernoff bound yield
  \begin{equation}
  \label{eq:A-P}
  \|(\mathbf{A} -
    \mathbf{P}) \mathbf{D}^{-1/2} \| \leq \|(\mathbf{A} -
    \mathbf{P}) \| \times \|\mathbf{D}^{-1/2} \| = O_{\mathbb{P}}(1).
 \end{equation}
  Combining Eq.~\eqref{eq:sqrtD - sqrtT} and Eq.~\eqref{eq:A-P}, we have
  \begin{equation*}
    \begin{split}
    \|(\mathbf{D}^{-1/2} - \mathbf{T}^{-1/2}) (\mathbf{A} -
     \mathbf{P}) \mathbf{D}^{-1/2} \| & \leq 
   (\|\mathbf{T}^{-3/2}
    (\mathbf{D} - \mathbf{T}) \|/2 +  O_{\mathbb{P}}((n \rho_n)^{-3/2})) \times O_{\mathbb{P}}(1)
    \\ & = O_{\mathbb{P}}((n \rho_n)^{-1}).
    \end{split}
  \end{equation*}
  Similarly, Eq.~\eqref{eq:sqrtD - sqrtT} and Eq.~\eqref{eq:A-Ppart2} implies
     $$ \|\mathbf{T}^{-1/2} (\mathbf{A} -
     \mathbf{P})( \mathbf{D}^{-1/2} - \mathbf{T}^{-1/2})\| = O_{\mathbb{P}}((n \rho_n)^{-1}).$$
  We thus have 
  \begin{equation}
  \label{eq:lA_error}
 \mathcal{L}(\mathbf{A})=
    \mathbf{T}^{-1/2}(\mathbf{A}
    - \mathbf{P}) \mathbf{T}^{-1/2} + \mathbf{D}^{-1/2}
    \mathbf{P} \mathbf{D}^{-1/2} + O_{\mathbb{P}}((n \rho_n)^{-1}).
  \end{equation}
  Eq.~\eqref{eq:uLA-LPu1b} is thereby established. 
  
  We next derive Eq.~\eqref{eq:uLA-LPu3b} through Eq.~\eqref{eq:uLA-LPu3d}. From Eq.~\eqref{eq:lA_error}, we have
  \begin{equation}
  \label{eq:U_p_decomp1}
  \begin{split}
     \tilde{\mathbf{U}}_{\mathbf{P}}^{\top} (\mathcal{L}(\mathbf{A}) -
    \mathcal{L}(\mathbf{P})) \tilde{\mathbf{U}}_{\mathbf{P}} &=
    \tilde{\mathbf{U}}_{\mathbf{P}}^{\top} \mathbf{T}^{-1/2}(\mathbf{A}
    - \mathbf{P}) \mathbf{T}^{-1/2} \tilde{\mathbf{U}}_{\mathbf{P}} \\ &+ \tilde{\mathbf{U}}_{\mathbf{P}}^{\top} (\mathbf{D}^{-1/2}
    \mathbf{P} \mathbf{D}^{-1/2} - \mathbf{T}^{-1/2} \mathbf{P}
    \mathbf{T}^{-1/2}) \tilde{\mathbf{U}}_{\mathbf{P}} \\ &+ O_{\mathbb{P}}((n \rho_n)^{-1}).
    \end{split}
  \end{equation}
  We first bound the spectral norm of $\tilde{\mathbf{U}}_{\mathbf{P}}^{\top} \mathbf{T}^{-1/2}(\mathbf{A}
    - \mathbf{P}) \mathbf{T}^{-1/2}
    \tilde{\mathbf{U}}_{\mathbf{P}}$. Let $\tilde{\bm{u}}_{i}$ be the
    $i$-th column of $\tilde{\mathbf{U}}_{\mathbf{P}} \mathbf{T}^{-1/2}$; the $ij$-th
    entry of $\tilde{\mathbf{U}}_{\mathbf{P}}^{\top} \mathbf{T}^{-1/2}(\mathbf{A}
    - \mathbf{P}) \mathbf{T}^{-1/2}
    \tilde{\mathbf{U}}_{\mathbf{P}}$ is then of the form
    \begin{equation*}
      \tilde{\bm{u}}_i^{\top} (\mathbf{A} -
      \mathbf{P}) \tilde{\bm{u}}_j = \sum_{k < l} 2
      \tilde{\bm{u}}_{ik} (a_{kl} - p_{kl})
      \tilde{\bm{u}}_{jl} + \sum_{k} \tilde{\bm{u}}_{ik}
      p_{kk}  \tilde{\bm{u}}_{jk}
    \end{equation*}
    where $\tilde{\bm{u}}_{ik}$ is the $k$-th element of the vector $\tilde{\bm{u}}_i$. We note that
    $$|\sum_{k} \tilde{\bm{u}}_{ik}
      p_{kk}  \tilde{\bm{u}}_{jk}| \leq \rho_n \|\tilde{\bm{u}}_i \|
      \times \|\tilde{\bm{u}}_j \| \leq \rho_n \delta_n^{-1} =
      O_{\mathbb{P}}(\rho_n (n \rho_n)^{-1}).$$  
      In addtion, $\sum_{k < l} 2
      \tilde{\bm{u}}_{ik} (a_{kl} - p_{kl})
      \tilde{\bm{u}}_{jl}$ is, conditioned on $\mathbf{P}$, a sum of mean $0$ random variables. 
    Hoeffding's inequality then implies
    \begin{equation*}
      \begin{split}
      \mathbb{P}\,\, \Bigl[ \Bigl| \sum_{k < l} 2
      \tilde{\bm{u}}_{ik} (a_{kl} - p_{kl})
      \tilde{\bm{u}}_{lj} \Bigr| \geq t \,\, \Bigr] & \leq \exp\Bigl(-
      \frac{t^2}{2(\sum_{k < l} \tilde{\bm{u}}_{ik}^{2}
        \tilde{\bm{u}}_{jl}^{2})} \Bigr) \\ & \leq
      \exp \Bigl( -      \frac{t^2}{2 \sum_{k} \sum_{l} \tilde{\bm{u}}_{ik}^{2}
        \tilde{\bm{u}}_{jl}^{2}}\Bigr) \\ & \leq 
      \exp \Bigl( -      \frac{t^2}{2 \delta^{-2}} \Bigr). 
      \end{split}
    \end{equation*}
    Hence $\tilde{\bm{u}}_i^{\top} (\mathbf{A} -
      \mathbf{P}) \tilde{\bm{u}}_j = O_{\mathbb{P}}(\delta^{-1})$. As $ \tilde{\mathbf{U}}_{\mathbf{P}}^{\top} \mathbf{T}^{-1/2}(\mathbf{A}
    - \mathbf{P}) \mathbf{T}^{-1/2} \tilde{\mathbf{U}}_{\mathbf{P}}$ is a $d \times d$ matrix, 
    a union bound then implies
      \begin{equation}
      \label{eq:sqrtT_A-P}
        \tilde{\mathbf{U}}_{\mathbf{P}}^{\top} \mathbf{T}^{-1/2}(\mathbf{A}
    - \mathbf{P}) \mathbf{T}^{-1/2}
    \tilde{\mathbf{U}}_{\mathbf{P}} = O_{\mathbb{P}}(\delta^{-1}) = O_{\mathbb{P}}((n \rho_n)^{-1}).
      \end{equation}
  

We next bound the spectral norm of $\tilde{\mathbf{U}}_{\mathbf{P}}^{\top} (\mathbf{D}^{-1/2}
    \mathbf{P} \mathbf{D}^{-1/2} - \mathbf{T}^{-1/2} \mathbf{P}
    \mathbf{T}^{-1/2}) \tilde{\mathbf{U}}_{\mathbf{P}}$. Let $\zeta_{ij}$ denote the $ij$-th entry of $\tilde{\mathbf{U}}_{\mathbf{P}}^{\top} (\mathbf{D}^{-1/2}
    \mathbf{P} \mathbf{D}^{-1/2} - \mathbf{T}^{-1/2} \mathbf{P}
    \mathbf{T}^{-1/2}) \tilde{\mathbf{U}}_{\mathbf{P}}$. From Eq.~\eqref{eq:sqrtD - sqrtT}, we have
    \begin{equation*}
    \begin{split}
    \zeta_{ij} &= \tilde{\bm{u}}_i^{\top} \Bigl((\mathbf{D}^{-1/2} - \mathbf{T}^{-1/2})
      \mathbf{P} \mathbf{D}^{-1/2} + \mathbf{T}^{-1/2} \mathbf{P}
      (\mathbf{D}^{-1/2} - \mathbf{T}^{-1/2})\Bigr) \tilde{\bm{u}}_{j} \\
      &=
      \frac{1}{2} \tilde{\bm{u}}_{i} \Bigl(\mathbf{T}^{-3/2} (\mathbf{T} - \mathbf{D})
      \mathbf{P} \mathbf{D}^{-1/2}  + \mathbf{T}^{-1/2}
      \mathbf{P} \mathbf{T}^{-3/2}(\mathbf{T} - \mathbf{D})\Bigr) \tilde{\bm{u}}_{j} +
      O_{\mathbb{P}}((n \rho_n)^{-3/2}). 
      \end{split} 
    \end{equation*}
   
      Now  let $\zeta_{ij}^{(1)}$ and $\zeta_{ij}^{(2)}$ denote the quantities
      \begin{gather*}
      \zeta_{ij}^{(1)} = \frac{1}{2} \tilde{\bm{u}}_{i}^{\top} \mathbf{T}^{-3/2} (\mathbf{T} - \mathbf{D})
      \mathbf{P} \mathbf{D}^{-1/2} \tilde{\bm{u}}_{j}, \\
       \zeta_{ij}^{(2)} = \frac{1}{2} \tilde{\bm{u}}_{i}^{\top} \mathbf{T}^{-1/2}
      \mathbf{P} \mathbf{T}^{-3/2}(\mathbf{T} - \mathbf{D}) \tilde{\bm{u}}_{j}.
      \end{gather*} 

      Because $\mathbf{P} = \rho_n \mathbf{X}
      \mathbf{X}^{\top}$, we have
      \begin{align*}
      \zeta_{ij}^{(1)}
      & \leq \frac{1}{2} \| \rho_n^{1/2} \tilde{\bm{u}}_i^{\top} \mathbf{T}^{-3/2} (\mathbf{T} -
      \mathbf{D}) \mathbf{X} \| \times \| \rho_n^{1/2}
      \mathbf{X}^{\top} \mathbf{D}^{-1/2} \tilde{\bm{u}}_{j} \|, \\
      \zeta_{ij}^{(2)} & \leq \frac{1}{2}  \| \rho_n^{1/2} \tilde{\bm{u}}_{i}^{\top} \mathbf{T}^{-1/2} \mathbf{X} \| \times \|\rho_n^{1/2} \mathbf{X}^{\top} \mathbf{T}^{-3/2} (\mathbf{T} - \mathbf{D}) \tilde{\bm{u}}_j \|.
      \end{align*}
      For $k \in \{1,2,\dots, d\}$, let $\bm{x}_k$ denote the $k$-th column of $\mathbf{X}$. Furthermore, for $l \in \{1,2,\dots,n\}$, let $x_{kl}$ denote the $l$-th entry of $\bm{x}_k$ -- equivalently the $k$-th entry of $X_l$ (recall that $\mathbf{X} = [X_1 \mid \cdots \mid X_n]^{\top}$). Also let $\tilde{u}_{kl}$ denotes the $l$-th entry of $\tilde{\bm{u}}_k$. Then $\rho_n^{1/2} \tilde{\bm{u}}_{i}^{\top} \mathbf{T}^{-3/2} (\mathbf{T} -
      \mathbf{D}) \mathbf{X}$ is a vector in $\mathbb{R}^{d}$ whose
      $k$-th element is of the form
\begin{equation*}
  \begin{split}
\rho_n^{1/2} \tilde{\bm{u}}_{i}^{\top}
\mathbf{T}^{-3/2} (\mathbf{D} - \mathbf{T}) \bm{x}_k  & = 
  \rho_n^{1/2} \sum_{l} \frac{\tilde{u}_{il}}{t_l^{3/2}} (d_l -
  t_l) x_{kl} \\ &=
  \rho_n^{1/2} \sum_{l} \sum_{m} \frac{\tilde{u}_{il}}{t_{l}^{3/2}}
  (a_{lm} - p_{lm}) x_{kl} \\ &= 2 \rho_n^{1/2}  \sum_{l < m}
    \frac{\tilde{u}_{il}}{t_l^{3/2}} (a_{lm} -
    p_{lm}) x_{kl} + \rho_n^{1/2} \sum_{l} \frac{\tilde{u}_{il}}{
      t_l^{3/2}} p_{ll} x_{kl}
  \end{split}
\end{equation*}
Conditioned on $\mathbf{P}$, the above is a sum of
mean $0$ random variables and a term of order $O((n
\rho_n)^{-3/2})$. Hoeffding's inequality then yields 
\begin{equation*}
  \begin{split}
\mathbb{P} \Bigl[ \Bigl| 2 \rho_n^{1/2} \sum_{l < m}
    \frac{\tilde{u}_{il}}{ t_l^{3/2}} (a_{lm} - p_{lm}) x_{kl} \Bigr| \geq s \Bigr]  & \leq 2 \exp
    \Bigl( - \frac{s^2}{ 2 \rho_n \sum_{l < m} t_l^{-3} \tilde{u}_{il}^{2}
      x_{kl}^{2} } \Bigr) \\ & \leq 2 \exp
    \Bigl( - \frac{s^2}{ 2 n \rho_n \sum_{l} t_l^{-3} \tilde{u}_{il}^{2}
      x_{kl}^{2} } \Bigr) \\ & \leq 2 \exp \Bigl( - \frac{s^2}{ 2 n
      \rho_n \delta_n^{-3} \sum_{l} \tilde{u}_{il}^{2}} \Bigr) 
    \\ & \leq 2 \exp \Bigl( - \frac{s^2 \delta_n^{3}}{2 n \rho_n \|\tilde{\bm{u}}_i\|^2} \Bigr) \\
    & \leq 2 \exp \Bigl( - \frac{s^2 \delta^{3}}{2 n \rho_n} \Bigr)
  \end{split}
\end{equation*}
where we used the fact that $x_{kl}^{2} \leq 1$ for all indices $k$ and $l$ (as $(\mathbf{A}, \mathbf{X})
\sim \mathrm{RDPG}(F)$). We thus have 
\begin{equation}
\label{eq:zeta_ij1}
\rho_n^{1/2} \tilde{\bm{u}}_{i}^{\top}
\mathbf{T}^{-3/2} (\mathbf{D} - \mathbf{T}) \bm{x}_k = O_{\mathbb{P}}((n \rho_n)^{-1})
\end{equation} 
A union bound over the $d$ entries of $\rho_n^{1/2} \tilde{\bm{u}}_i^{\top} \mathbf{T}^{-3/2} (\mathbf{T} -
      \mathbf{D}) \mathbf{X}$ along with the bound $\|\rho_n^{1/2}
      \mathbf{T}^{-1/2} \mathbf{X} \| = O_{\mathbb{P}}(1)$ yield that $\zeta_{ij}^{(1)} =
      O_{\mathbb{P}}((n \rho_n)^{-1})$. An identical argument also yield that $\zeta_{ij}^{(2)} = O_{\mathbb{P}}((n \rho_n)^{-1})$. Therefore, $\zeta_{ij} = O((n \rho_n)^{-1})$. A union bound over the indices $i,j \in \{1,2,\dots,d\}$ also implies
      \begin{gather}
      \label{eq:union-zeta_ij1}
        \mathbf{T}^{-1/2} \mathbf{P}
      \mathbf{T}^{-3/2}(\mathbf{T} - \mathbf{D})
      \tilde{\mathbf{U}}_{\mathbf{P}} = O_{\mathbb{P}}((n \rho_n)^{-1}), \\
        \label{eq:union-zeta_ij2}
      \tilde{\mathbf{U}}_{\mathbf{P}}^{\top} \mathbf{T}^{-3/2} (\mathbf{T} - \mathbf{D})
      \mathbf{P} \mathbf{D}^{-1/2} = O_{\mathbb{P}}((n \rho_n)^{-1}), \\
      \label{eq:union-zeta_ij}
      \|\tilde{\mathbf{U}}_{\mathbf{P}}^{\top} (\mathbf{D}^{-1/2}
    \mathbf{P} \mathbf{D}^{-1/2} - \mathbf{T}^{-1/2} \mathbf{P}
    \mathbf{T}^{-1/2}) \tilde{\mathbf{U}}_{\mathbf{P}}\| = O_{\mathbb{P}}((n \rho_n)^{-1}).
      \end{gather}
     We thus derive Eq.~\eqref{eq:uLA-LPu3b} and Eq.~\eqref{eq:uLA-LPu3c}. Eq.~\eqref{eq:uLA-LPu3d} follows from Eq.~\eqref{eq:U_p_decomp1}, Eq.~\eqref{eq:sqrtT_A-P} and Eq.~\eqref{eq:union-zeta_ij}. Lemma~\ref{lemma:quadratic-up-ua} is thereby established.
 
Lemma~\ref{lem:2} now follows directly from Lemma~\ref{lemma:quadratic-up-ua}. Indeed, by Eq.~\eqref{eq:uLA-LPu3a} and Eq.~\eqref{eq:uLA-LPu3d}, we have
\begin{equation} 
\label{eq:lem:2-part1}
\begin{split}
\tilde{\mathbf{U}}_{\mathbf{P}}^{\top}
    \tilde{\mathbf{U}}_{\mathbf{A}} \tilde{\mathbf{S}}_{\mathbf{A}} -
    \tilde{\mathbf{S}}_{\mathbf{P}}
    \tilde{\mathbf{U}}_{\mathbf{P}}^{\top} \tilde{\mathbf{U}}_{\mathbf{A}} &=
    \tilde{\mathbf{U}}_{\mathbf{P}}^{\top} \mathcal{L}({\mathbf{A}})
    \tilde{\mathbf{U}}_{\mathbf{A}} - \tilde{\mathbf{U}}_{\mathbf{P}}^{\top}
    \mathcal{L}({\mathbf{P}}) \tilde{\mathbf{U}}_{\mathbf{A}} \\ &=
    \tilde{\mathbf{U}}_{\mathbf{P}}^{\top} (\mathcal{L}(\mathbf{A}) -
    \mathcal{L}(\mathbf{P})) (\mathbf{R} + \tilde{\mathbf{U}}_{\mathbf{P}}
    \tilde{\mathbf{U}}_{\mathbf{P}}^{\top}
    \tilde{\mathbf{U}}_{\mathbf{A}}) \\ &=
    O_{\mathbb{P}}((n \rho_n)^{-1}) + \tilde{\mathbf{U}}_{\mathbf{P}} (\mathcal{L}(\mathbf{A}) -
    \mathcal{L}(\mathbf{P})) \tilde{\mathbf{U}}_{\mathbf{P}}
    \tilde{\mathbf{U}}_{\mathbf{P}}^{\top}
    \tilde{\mathbf{U}}_{\mathbf{A}} \\ &= O_{\mathbb{P}}((n \rho_n)^{-1}).
    \end{split}
\end{equation}
Eq.~\eqref{eq:lem2a} is thereby established. We now establish Eq.~\eqref{eq:lem2b}, noting that the same argument applies also to Eq.~\eqref{eq:lem2c}. For $i,j \in \{1,2,\dots, d\}$, let $r_{ij}$ denote the $ij$-th entry of $\tilde{\mathbf{U}}_{\mathbf{P}}^{\top}
    \tilde{\mathbf{U}}_{\mathbf{A}}$. Also, for $i \in \{1,2,\dots, d\}$, let $\tilde{\lambda}_i(\mathbf{A})$ and $\tilde{\lambda}_{j}(\mathbf{P})$ denote the $i$-th eigenvalue of $\mathcal{L}(\mathbf{A})$ and $\mathcal{L}(\mathbf{P})$, respectively. 
Then the $ij$-th entry of $\tilde{\mathbf{U}}_{\mathbf{P}}^{\top}
    \tilde{\mathbf{U}}_{\mathbf{A}} \tilde{\mathbf{S}}_{\mathbf{A}}^{1/2} - \tilde{\mathbf{S}}_{\mathbf{P}}^{1/2} \tilde{\mathbf{U}}_{\mathbf{P}}^{\top} \tilde{\mathbf{U}}_{\mathbf{A}}$ is of the form
    $$ r_{ij} (\tilde{\lambda}_{j}^{1/2}(\mathbf{A}) - \tilde{\lambda}_{i}^{1/2}(\mathbf{P})) = \frac{r_{ij}(\tilde{\lambda}_{j}(\mathbf{A}) - \tilde{\lambda}_{i}(\mathbf{P}))}{\tilde{\lambda}_{j}^{1/2}(\mathbf{A}) + \tilde{\lambda}_{i}^{1/2}(\mathbf{P}))}.$$
    Since $\tilde{\lambda}_i(\mathbf{A}) = \Theta_{\mathbb{P}}(1)$ and $\tilde{\lambda}_{j}(\mathbf{P}) = \Theta_{\mathbb{P}}(1)$, the previous expression and Eq.~\eqref{eq:lem:2-part1} yield
    $$ r_{ij} (\tilde{\lambda}_{j}^{1/2}(\mathbf{A}) - \tilde{\lambda}_{i}^{1/2}(\mathbf{P})) = O_{\mathbb{P}}((n \rho_n)^{-1}). $$
    A union bound over $i,j$ then implies Eq.~\eqref{eq:lem2b}. 

   \subsection{Proof of Eq.~\eqref{eq:4} and Eq.~\eqref{eq:14}}
   \label{sec:proof-LSE-3}
    Recall Eq.~\eqref{eq:LSE-main1}, i.e., with $\zeta = (\breve{\mathbf{X}} \mathbf{W} - \tilde{\mathbf{X}} )$, we have
    $$ \|\zeta\|_{F} = \|\mathbf{T}^{-1/2}(\mathbf{A} - \mathbf{P}) \mathbf{T}^{-1/2} \tilde{\mathbf{X}} (\tilde{\mathbf{X}}^{\top} \tilde{\mathbf{X}})^{-1} + \tfrac{1}{2} \mathbf{T}^{-1}(\mathbf{T} - \mathbf{D}) \tilde{\mathbf{X}} \|_{F} + O_{\mathbb{P}}((n \rho_n)^{-1}).$$
    The above implies, 
    \begin{equation*}
    \begin{split} \|\zeta\|_{F}^{2} &= \|\mathbf{T}^{-1/2}(\mathbf{A} - \mathbf{P}) \mathbf{T}^{-1/2} \tilde{\mathbf{X}} (\tilde{\mathbf{X}}^{\top} \tilde{\mathbf{X}})^{-1}\|_{F}^{2} + \tfrac{1}{4} \|\mathbf{T}^{-1}(\mathbf{T} - \mathbf{D}) \tilde{\mathbf{X}} \|_{F}^{2} \\ 
    &+ \mathrm{tr} \,\, \tilde{\mathbf{X}}^{\top} \mathbf{T}^{-1}(\mathbf{T} - \mathbf{D}) \mathbf{T}^{-1/2}(\mathbf{A} - \mathbf{P}) \mathbf{T}^{-1/2} \tilde{\mathbf{X}} (\tilde{\mathbf{X}}^{\top} \tilde{\mathbf{X}})^{-1}
    + O_{\mathbb{P}}((n \rho_n)^{-3/2}).
    \end{split}
    \end{equation*}
    We show Eq.~\eqref{eq:4} and Eq.~\eqref{eq:14} by analyzing each term in the right hand side of the above display. In particular, we shall show that these terms are concentrated around their expected values; evaluation of these expected values, in the limit as $n \rightarrow \infty$, yield Eq.~\eqref{eq:4} and Eq.~\eqref{eq:14}. 

    We first consider the term $Z = \|\mathbf{T}^{-1/2}(\mathbf{A} - \mathbf{P}) \mathbf{T}^{-1/2} \tilde{\mathbf{X}} (\tilde{\mathbf{X}}^{\top} \tilde{\mathbf{X}})^{-1}\|_{F}^{2}$. We note that conditional on $\mathbf{P}$, $Z$ is a function of the $n(n-1)/2$
independent random variables $\{a_{ij}\}_{i < j}$. It is therefore expected that $Z$ will be concentrated around its expectation $\mathbb{E}[Z]$ where the expectation is taken with respect to $\mathbf{A}$, conditional on $\mathbf{P}$. We verify this below. 

Let $\mathbf{A}' = (a'_{ij})$ be an independent copy of $\mathbf{A}$, i.e., the upper triangular entries of $\mathbf{A}'$ are independent Bernoulli random variables with mean parameters $\{p_{ij} \}_{i < j}$.
Let $\mathbf{A}^{(ij)}$ be the matrix obtained by replacing the $(i,j)$ and $(j,i)$ entries of
$\mathbf{A}$ by $a'_{ij}$ and let $Z^{(ij)} = \|\mathbf{T}^{-1/2}(\mathbf{A}^{(ij)} - \mathbf{P}) \mathbf{T}^{-1/2} \tilde{\mathbf{X}} (\tilde{\mathbf{X}}^{\top} \tilde{\mathbf{X}})^{-1}\|_{F}^{2}$. We show concentration of $Z$ around $\mathbb{E}[Z]$ using 
the following concentration inequality from \cite[Theorem~5 and Theorem~6]{boucheron2003}.

\begin{theorem}
\label{THM:log_Sobolev}
Assume that there exists positive constants $a$ and $b$ such that 
$$ \sum_{i < j}(Z - Z^{(ij)})^2 \leq aZ + b. $$
Then for all $t > 0$,
\begin{gather} \mathbb{P}[Z - \mathbb{E}[Z] \geq t] \leq \exp\Bigl(\frac{-t^2}{4 a \mathbb{E}[Z] + 4b + 2at}\Bigr), \\
\mathbb{P}[Z - \mathbb{E}[Z] \leq -t] \leq \exp\Bigl(\frac{-t^2}{4 a \mathbb{E}[Z]} \Bigr).
\end{gather}
\end{theorem}
We now bound $\sum_{i < j}(Z - Z^{(ij)})^2$. For notational convenience, we denote
the $i$-th row of $\tilde{\mathbf{X}}(\tilde{\mathbf{X}}^{\top} \tilde{\mathbf{X}})^{-1}$ by $\zeta_i$ and the $i$-th row of 
$\mathbf{T}^{-1/2}(\mathbf{A} - \mathbf{P}) \mathbf{T}^{-1/2} \tilde{\mathbf{X}}(\tilde{\mathbf{X}}^{\top} \tilde{\mathbf{X}})^{-1}$ by $\xi_i$.
We shall also denote the
inner product between vectors in Euclidean space by $\langle \cdot,
\cdot \rangle$. For each $i$, $\xi_i = \sum_{j=1}^{n} \tfrac{a_{ij} - p_{ij}}{\sqrt{t_i t_j}} \zeta_j$ and hence
\begin{equation*}
  \begin{split}
  Z &= \sum_{k=1}^{n} \xi_k^{2}
  = \sum_{k=1}^{n} \sum_{\ell=1}^{n} \sum_{\ell'=1}^{n} \frac{(a_{k\ell} - p_{k\ell})(a_{k \ell'} - p_{k \ell'})}{t_k \sqrt{t_{\ell} t_{\ell'}}} \langle \zeta_{\ell}, \zeta_{\ell'} \rangle.
  \end{split}
\end{equation*}
Now $\mathbf{A}$ and $\mathbf{A}^{(ij)}$ differs possibly only in the $(i,j)$
and $(j,i)$ entries; furthermore, the $\{t_i\}$ do not depend on the entries of $\mathbf{A}$ and $\mathbf{A}^{(ij)}$. 
We thus have, upon considering the cases where $k = i$ and $\ell =j$, $k = j$ and $\ell = i$, $k = i$ and $\ell' = j$, and $k = j$ and $\ell' = i$, that
\begin{equation*}
  \begin{split}
  Z - Z^{(ij)} &= \sum_{\ell'=1}^{n} \frac{(a_{ij} - a'_{ij}) (a_{i \ell'} - p_{i \ell'})}{t_i \sqrt{t_{j} t_{\ell'}}}
  \langle \zeta_j, \zeta_{\ell'} \rangle + \sum_{\ell'=1}^{n} \frac{(a_{ji} - a'_{ji})(a_{j \ell'} - p_{j \ell'})}{t_j \sqrt{t_{i} t_{\ell'}}}
  \langle \zeta_i, \zeta_{\ell'} \rangle \\
  &+ \sum_{\ell=1}^{n} \frac{(a_{i \ell} - p_{i \ell}) (a_{ij} - a'_{ij})}{t_i \sqrt{t_{j} t_{\ell}}}
  \langle \zeta_j, \zeta_{\ell} \rangle + \sum_{\ell=1}^{n} \frac{(a_{j \ell} - p_{j \ell})(a_{ji} - a'_{ji})}{t_j \sqrt{t_{i} t_{\ell}}}
  \langle \zeta_i, \zeta_{\ell} \rangle.
  \end{split}
\end{equation*}
Since $a_{ij} = a_{ji}$ and $a'_{ij} = a'_{ji}$, the above simplifies to 
$$ Z - Z^{(ij)} = 2(a_{ij} - a'_{ij}) \sum_{\ell=1}^{n} \Bigl( \frac{a_{i \ell} - p_{i \ell}}{t_i \sqrt{t_j t_{\ell}}} \langle \zeta_j, \zeta_{\ell} \rangle + 
\frac{a_{j \ell} - p_{j \ell}}{t_j \sqrt{t_i t_{\ell}}} \langle \zeta_i, \zeta_{\ell} \rangle \Bigr). $$
We then have, since $a_{ij}$ and $a'_{ij}$ are binary variables, i.e., $|a_{ij} - a'_{ij}| \leq 1$, that 
  \begin{equation*}
    (Z - Z^{(ij)})^{2} \leq 4 \Bigl(\sum_{\ell=1}^{n} \frac{a_{i \ell} - p_{i \ell}}{t_i \sqrt{t_j t_{\ell}}} \langle \zeta_j, \zeta_{\ell} \rangle \Bigr)^{2} + 4 
    \Bigl(\sum_{\ell=1}^{n} \frac{a_{j \ell} - p_{j \ell}}{t_j \sqrt{t_i t_{\ell}}} \langle \zeta_i, \zeta_{\ell} \rangle \Bigr)^{2}
  \end{equation*}
Now $(t_j t_l)^{-1/2} \langle \zeta_j, \zeta_l \rangle$ is the $(l,j)$-th entry of $\mathbf{T}^{-1/2} \tilde{\mathbf{X}}(\tilde{\mathbf{X}}^{\top} \tilde{\mathbf{X}})^{-1} (\tilde{\mathbf{X}}^{\top} \tilde{\mathbf{X}})^{-1} \tilde{\mathbf{X}} \mathbf{T}^{-1/2}$. 
Thus, $\sum_{\ell=1}^{n} \frac{a_{i \ell} - p_{i \ell}}{t_i \sqrt{t_j t_{\ell}}} \langle \zeta_j, \zeta_{\ell} \rangle$ is the $(i,j)$-th entry of
$\mathbf{T}^{-1} (\mathbf{A} - \mathbf{P}) \mathbf{T}^{-1/2} \tilde{\mathbf{X}}(\tilde{\mathbf{X}}^{\top} \tilde{\mathbf{X}})^{-2} \tilde{\mathbf{X}} \mathbf{T}^{-1/2}$. We therefore have,
\begin{equation*}
\begin{split}
    \sum_{i < j} (Z - Z^{(ij)})^{2} & \leq 4 \sum_{i < j}  \Bigl(\sum_{\ell=1}^{n} \frac{a_{i \ell} - p_{i \ell}}{t_i \sqrt{t_j t_{\ell}}} \langle \zeta_j, \zeta_{\ell} \rangle \Bigr)^{2} + 4 \sum_{i < j}
    \Bigl(\sum_{\ell=1}^{n} \frac{a_{j \ell} - p_{j \ell}}{t_j \sqrt{t_i t_{\ell}}} \langle \zeta_i, \zeta_{\ell} \rangle \Bigr)^{2} \\
    &\leq 8 \sum_{i=1}^{n} \sum_{j=1}^{n} \Bigl(\sum_{\ell=1}^{n} \frac{a_{i \ell} - p_{i \ell}}{t_i \sqrt{t_j t_{\ell}}} \langle \zeta_j, \zeta_{\ell} \rangle \Bigr)^{2} 
    \\ & \leq 8 \|\mathbf{T}^{-1} (\mathbf{A} - \mathbf{P}) \mathbf{T}^{-1/2} \tilde{\mathbf{X}}(\tilde{\mathbf{X}}^{\top} \tilde{\mathbf{X}})^{-2} \tilde{\mathbf{X}} \mathbf{T}^{-1/2} \|_{F}^{2} \\
    & \leq 8 \|\mathbf{T}^{-1/2} (\mathbf{A} - \mathbf{P}) \mathbf{T}^{-1/2} \tilde{\mathbf{X}}(\tilde{\mathbf{X}}^{\top} \tilde{\mathbf{X}})^{-1} \|_{F}^{2}  \|(\tilde{\mathbf{X}}^{\top} \tilde{\mathbf{X}})^{-1} \tilde{\mathbf{X}} \mathbf{T}^{-1/2} \|^{2} \|\mathbf{T}^{-1/2} \|^2 
    \\ & \leq 8 Z \|(\tilde{\mathbf{X}}^{\top} \tilde{\mathbf{X}})^{-1} \tilde{\mathbf{X}} \mathbf{T}^{-1/2} \|^{2} \|\mathbf{T}^{-1/2} \|^2 \\
    &\leq 8 Z \| \rho_n^{-1/2} (\mathbf{X}^{\top} \mathbf{T}^{-1} \mathbf{X})^{-1} \mathbf{X} \mathbf{T}^{-1} \|^{2} \|\mathbf{T}^{-1/2} \|^2 \\
    &\leq 8 \rho_n^{-1} Z \|(\mathbf{X}^{\top} \mathbf{T}^{-1} \mathbf{X})^{-1} \|^{2} \|\mathbf{X} \|^{2} \|\mathbf{T}^{-1} \|^2 \|\mathbf{T}^{-1/2} \|^2 \\
    & \leq 8 \rho_n^{-1} Z \|(\mathbf{X}^{\top} \mathbf{T}^{-1} \mathbf{X})^{-1} \|^{2} n \delta^{-3} \\
    & \leq C \rho_n^{-1} Z \|(\mathbf{X}^{\top} \mathbf{T}^{-1} \mathbf{X})^{-1} \|^{2} n (n \rho_n)^{-3} \\
    & \leq C (n \rho_n)^{-2} Z.
  \end{split}
\end{equation*}
for some constant $C$; note that $C$ denote a generic constant, not depending on $Z$, in the above display and could change from line to line. In the above derivation, we have used the fact that $C_0 \sqrt{n} \leq \|\mathbf{X} \| \leq \sqrt{n}$ for some constant $C_0 > 0$ and $\|\mathbf{T}\| \geq \delta \geq C_1 n \rho_n$ for some constant $C_1 >0$. 

We then have, by Theorem~\ref{THM:log_Sobolev}, that for all $t > 0$, 
\begin{gather}
  \label{eq:log-sobolev-part1b}
  \mathbb{P}[Z - \mathbb{E}[Z] > t] \leq \exp\Bigl( \frac{-Ct^2}{(n \rho_n)^{-2} \mathbb{E}[Z] + 2 (n \rho_n)^{-2} t} \Bigr) \\
  \mathbb{P}[Z - \mathbb{E}[Z] > -t] \leq \exp\Bigl( \frac{-Ct^2}{(n \rho_n)^{-2} \mathbb{E}[Z]} \Bigr).
\end{gather}
In addition, it is straightforward to see that $\mathbb{E}[Z] \leq C_3 (n \rho_n)^{-1}$, for some constant $C_3 > 0$; here the expectation is taken with respect to $\mathbf{A}$ conditional on $\mathbf{P}$. 
We therefore have that there exists a constant $C > 0$ such that $t = C (n \rho_n)^{-3/2} \log^{1/2}{n}$ yield
\begin{equation}
\label{eq:LSE-main-part1}
\begin{split}
 Z &= \mathbb{E}[Z] + O_{\mathbb{P}}((n \rho_n)^{-3/2} \log^{1/2}) \\
 &= \mathbb{E} \|\mathbf{T}^{-1/2} (\mathbf{A} - \mathbf{P}) \mathbf{T}^{-1/2} \tilde{\mathbf{X}}(\tilde{\mathbf{X}}^{\top} \tilde{\mathbf{X}})^{-1} \|_{F}^{2}  + O_{\mathbb{P}}((n \rho_n)^{-3/2} \log^{1/2})
 \end{split}
\end{equation} 

We now evaluate $\mathbb{E}[Z]$. We have 
   \begin{equation*}
     \begin{split}
     \mathbb{E}[Z] &= \mathbb{E} \Bigl[\mathrm{tr} \,\,  
(\tilde{\mathbf{X}}^{\top} \tilde{\mathbf{X}})^{-1} \tilde{\mathbf{X}}^{\top}
     \mathbf{T}^{-1/2}
   (\mathbf{A} - \mathbf{P}) \mathbf{T}^{-1} (\mathbf{A} -
   \mathbf{P}) \mathbf{T}^{-1/2} \tilde{\mathbf{X}}(\tilde{\mathbf{X}}^{\top} \tilde{\mathbf{X}})^{-1}
   \Bigr] \\
   &= \mathrm{tr} \,\, 
(\tilde{\mathbf{X}}^{\top} \tilde{\mathbf{X}})^{-1} \tilde{\mathbf{X}}^{\top}
   \mathbb{E} \Bigl[\mathbf{T}^{-1/2} (\mathbf{A} - \mathbf{P}) \mathbf{T}^{-1} (\mathbf{A} -
   \mathbf{P}) \mathbf{T}^{-1/2} \Bigr] \tilde{\mathbf{X}}(\tilde{\mathbf{X}}^{\top} \tilde{\mathbf{X}})^{-1}.
     \end{split}
   \end{equation*}  

   We note that $\mathbf{T}^{-1/2} (\mathbf{A} - \mathbf{P}) \mathbf{T}^{-1}
   (\mathbf{A} - \mathbf{P}) \mathbf{T}^{-1/2}$ is a $n \times n$ matrix whose
   $ij$-th entry $\xi_{ij}$ is of the form
   \begin{equation*}
     \xi_{ij} := \sum_{k} t_i^{-1/2} t_k^{-1} t_{j}^{-1/2} (a_{ik} - p_{ik})(a_{kj} - p_{kj}) 
   \end{equation*}
   and hence
   \begin{equation*}
     \mathbb{E}[\xi_{ij}]
     = \begin{cases} 0 & \text{if $i \not = j$} \\
       \sum_{k} t_i^{-1} t_k^{-1} p_{ik} (1 - p_{ik}) & \text{if $i = j$} \end{cases} 
   \end{equation*}
   We shall denote by $\tilde{\mathbf{M}}$ the diagonal matrix
   $(\mathbb{E}[\xi_{ij}])$ as given above. Then 
   \begin{equation*}
      n \rho_n \mathbb{E}[Z] = 
n \rho_n \mathrm{tr} \, \, (\tilde{\mathbf{X}}^{\top}
\tilde{\mathbf{X}})^{-1} \tilde{\mathbf{X}}^{\top} 
   \tilde{\mathbf{M}} \tilde{\mathbf{X}} (\tilde{\mathbf{X}}^{\top}
\tilde{\mathbf{X}})^{-1}
\end{equation*}
We first recall from Eq.~\eqref{eq:tildeX-outer} that $\tilde{\mathbf{X}}^{\top} \tilde{\mathbf{X}} \overset{\mathrm{a.s.}}{\rightarrow} \tilde{\Delta}$ and $(\tilde{\mathbf{X}}^{\top} \tilde{\mathbf{X}})^{-1} \overset{\mathrm{a.s.}}{\rightarrow} \tilde{\Delta}^{-1}$ as $n \rightarrow \infty$. 
We next consider $n \rho_n \tilde{\mathbf{X}}^{\top} \tilde{\mathbf{M}}
\tilde{\mathbf{X}}$. Let $\tilde{m}_i$ denote the $i$-th diagonal element of $\tilde{\mathbf{M}}$. We have
\begin{equation*}
  \begin{split}
  n \rho_n \tilde{\mathbf{X}}^{\top} \tilde{\mathbf{M}}
\tilde{\mathbf{X}} &= n \rho_n \sum_{i} \frac{\rho_n X_i
    X_i^{\top}\tilde{m}_i}{t_i} \\&= n \rho_n \sum_{i}
  \frac{ \rho_n X_i X_i^{\top}
    \tilde{m}_i}{\rho_n \sum_{j} X_i^{\top} X_j} \\ &= 
  \sum_{i} \frac{X_i X_i^{\top} n \rho_n \tilde{m}_i}{n X_i^{\top} \mu} + \sum_{i}
  \frac{X_i X_i^{\top} n \rho_n \tilde{m}_i}{n X_i^{\top} \mu} \Bigl(\frac{ n X_i^{\top} \mu - \sum_{j}
    X_i^{\top} X_j}{\sum_{j} X_i^{\top} X_j} \Bigr).
  \end{split}
\end{equation*}
Similar to our derivation of Eq.~\eqref{eq:tildeX-outer}, we have
\begin{equation*}
\begin{split}
  - (\sup_{j \in [n]}{ n \rho_n \tilde{m}_j c_j}) \sum_{i} \frac{X_i X_i^{\top}}{n
    X_i^{\top} \mu} & \preceq \sum_{i} \frac{X_i X_i^{\top} n \rho_n \tilde{m}_i}{n
    X_i^{\top} \mu} \Bigl(\frac{n X_i^{\top} \mu - \sum_{j} X_i^{\top}
    X_j}{\sum_{j} X_i^{\top} X_J} 
   \Bigr) \\ &\preceq (\sup_{j \in [n]} n \rho_n \tilde{m}_j c_j) \sum_{i} \frac{X_i X_i^{\top}}{n
    X_i^{\top} \mu}.
    \end{split}
\end{equation*}
In addition, for each index $i$, 
\begin{gather*}
n \rho_n \tilde{m}_i =
n \rho_n \sum_{k} t_i^{-1} t_k^{-1} p_{ik} (1 - p_{ik}) = O_{\mathbb{P}}(1)
\end{gather*} 
and hence $\sup_{i \in [n]} n \rho_n \tilde{m}_i c_i \overset{\mathrm{a.s.}}{\longrightarrow} 0$
as $n \rightarrow \infty$. Therefore
\begin{equation}
  \label{eq:12}
  \sum_{i} \frac{X_i X_i^{\top} n \rho_n \tilde{m}_i}{n
    X_i^{\top} \mu} \Bigl(\frac{n X_i^{\top} \mu - \sum_{j} X_i^{\top}
    X_j}{\sum_{j} X_i^{\top} X_J} 
   \Bigr) \overset{\mathrm{a.s}}{\longrightarrow} \bm{0} 
\end{equation}
as $n \rightarrow \infty$. We thus only need to consider
\begin{equation*}
  \begin{split}
  \sum_{i} \tfrac{X_i X_i^{\top} n \rho_n \tilde{m}_i}{n X_i^{\top}
    \mu} & = \sum_{i} \sum_{k} \tfrac{\rho_n X_i X_i^{\top} 
  p_{ik} ( 1- p_{ik})}{(X_i^{\top} \mu) t_i t_k} \\ &=
\sum_{i} \sum_{k} \tfrac{ \rho_n X_i X_i^{\top} \rho_n X_i^{\top} X_k (1 - \rho_n
     X_i^{\top} X_k)}{(X_i^{\top} \mu) \sum_{j} \rho_n X_i^{\top} X_j
     \sum_{l} \rho_n X_k^{\top} X_l} \\
   &= \sum_{i} \sum_{k}   \tfrac{X_i X_i^{\top} (X_i^{\top} X_k - \rho_n X_i^{\top}
     X_k X_k^{\top} X_i)}{(X_i^{\top} \mu) \sum_{j} X_i^{\top} X_j
     \sum_{l} X_k^{\top} X_l} \\
   &= \sum_{i} \sum_{k} \tfrac{X_i X_i^{\top} (X_i^{\top} X_k - \rho_n X_i^{\top}
     X_k X_k^{\top} X_i)}{n^{2} (X_i^{\top} \mu)^{2} (X_k
     ^{\top} \mu) } \\ &+ \sum_{i} \sum_{k} \tfrac{X_i X_i^{\top} (X_i^{\top} X_k - \rho_n X_i^{\top}
     X_k X_k^{\top} X_i)}{n^{2} (X_i^{\top} \mu)^{2} (X_k^{\top} \mu)} \Bigl(\tfrac{n^{2} (X_i^{\top}
     \mu)(X_k^{\top} \mu)  - \sum_{j} X_i^{\top} X_j \sum_{l}
   X_k^{\top} X_l}{ \sum_{j} X_i^{\top} X_j \sum_{l}
   X_k^{\top} X_l} \Bigr)
  \end{split}
\end{equation*}
An analogous argument to that used in deriving Eq.~\eqref{eq:12} yield
\begin{equation*}
  \sum_{i} \sum_{k} \tfrac{X_i X_i^{\top} (X_i^{\top} X_k - \rho_n X_i^{\top}
     X_k X_k^{\top} X_i)}{n^{2} (X_i^{\top} \mu)^{2} (X_k^{\top} \mu)} \Bigl(\tfrac{n^{2} (X_i^{\top}
     \mu)(X_k^{\top} \mu)  - \sum_{j} X_i^{\top} X_j \sum_{l}
   X_k^{\top} X_l}{ \sum_{j} X_i^{\top} X_j \sum_{l}
   X_k^{\top} X_l} \Bigr) \overset{\mathrm{a.s.}}{\longrightarrow} \bm{0}
\end{equation*}
as $n \rightarrow \infty$. It thus remains to evaluate
$$ \sum_{i} \sum_{k} \frac{X_i X_i^{\top} (X_i^{\top} X_k - \rho_n X_i^{\top}
     X_k X_k^{\top} X_i)}{n^{2} (X_i^{\top} \mu)^{2} (X_k
     ^{\top} \mu) }.$$

The strong law of large numbers implies
\begin{gather*}
  \sum_{i} \sum_{k} \frac{X_i X_i^{\top} X_i^{\top} X_k}{n^{2} (X_i^{\top} \mu)^{2} (X_k
     ^{\top} \mu) } \overset{\mathrm{a.s.}}{\longrightarrow} \mathbb{E} \Bigl[\frac{ X_1 X_1^{\top}
     X_1^{\top} \tilde{\mu}}{ (X_1^{\top} \mu)^2 } \Bigr] \\
\rho_n \sum_{i} \sum_{k} \frac{X_i X_i^{\top} X_i^{\top}
     X_k X_k^{\top} X_i}{n^{2} (X_i^{\top} \mu)^{2} (X_k
     ^{\top} \mu) } \rightarrow \rho_n \mathbb{E} \Bigl[ \frac{X_1
   X_1^{\top} X_1^{\top} \tilde{\Delta} X_1}{(X_1^{\top} \mu)^{2}} \Bigr].
 \end{gather*}
We invoke Slutsky's theorem and conclude that
\begin{equation}
\label{eq:slutsky-lse-part1}
  \begin{split}
     n \rho_n Z &= n \rho_n \|\mathbf{T}^{-1/2}(\mathbf{A} - \mathbf{P}) \mathbf{T}^{-1/2} \tilde{\mathbf{X}} (\tilde{\mathbf{X}}^{\top} \tilde{\mathbf{X}})^{-1}\|_{F}^{2} \\ &=
     n \rho_n
\mathrm{tr} \, \, (\tilde{\mathbf{X}}^{\top}
\tilde{\mathbf{X}})^{-1} \tilde{\mathbf{X}}^{\top} 
   \tilde{\mathbf{M}} \tilde{\mathbf{X}} (\tilde{\mathbf{X}}^{\top}
\tilde{\mathbf{X}})^{-1} + O_{\mathbb{P}}((n \rho_n)^{-1/2} \log^{1/2}{n})  \\ & \rightarrow \mathrm{tr}
\tilde{\Delta}^{-1} \mathbb{E}\Bigl[\frac{X_1 X_1^{\top} (X_1^{\top}
    \tilde{\mu} - \rho_n X_1^{\top} \tilde{\Delta}
    X_1)}{(X_1^{\top} \mu)^{2}} \Bigr] \tilde{\Delta}^{-1}.
  \end{split}
\end{equation}

We next bound $Z := \|(\mathbf{T} - \mathbf{D}) \mathbf{T}^{-1} \tilde{\mathbf{X}} \|^{2}_{F}$. $Z$ is again a function of the $n(n-1)/2$
independent random variables $\{a_{ij} \}_{i < j}$. Let $Z^{(ij)} = \|(\mathbf{T} - \mathbf{D}^{(ij)}) \mathbf{T}^{-1} \tilde{\mathbf{X}} \|$ where $\mathbf{D}^{(ij)}$ is the diagonal matrix whose diagonal entries are the degrees of $\mathbf{A}^{(ij)}$; we recall that $\mathbf{A}^{(ij)}$ is obtained by replacing the $(i,j)$ and $(j,i)$ entries of $\mathbf{A}$ with an independent copy $a'_{ij}$ of $a_{ij}$. We now bound $\sum_{i < j} (Z - Z^{(ij)})^2$. Let $\tilde{X}_i$ denote the $i$-th row of $\tilde{\mathbf{X}}$. Then
\begin{equation*}
Z = \sum_{k} \frac{(t_k - d_k)^{2}}{t_k^2} \|\tilde{X}_k\|^2,
\end{equation*}
and hence (with $d^{(ij)}_{k}$ denoting the degree of vertex $k$ in $\mathbf{A}^{(ij)}$)
\begin{equation*}
\begin{split}
Z - Z^{(ij)} &= \sum_{k} \bigl((t_k - d_k)^{2} - (t_k - d^{(ij)}_k)^2\bigr) \frac{\|\tilde{X}_k\|^2}{t_k^2} \\ &= \sum_{k} (d^{(ij)}_k - d_k) (2 t_k - d_k - d^{(ij)}_k) \frac{ \|\tilde{X}_k\|^2}{t_k^2}
\\ &= (a'_{ij} - a_{ij}) \Bigl((2 t_i - 2 d_i + a_{ij} - a'_{ij}) \frac{\|\tilde{X}_i\|^2}{t_i^2}  + (2 t_j - 2 d_j + a_{ij} - a'_{ij}) \frac{\| \tilde{X}_j \|^2}{t_j^2} \Bigr).
\end{split}
\end{equation*}
Using the fact that $(b + c)^2 \leq 2b^2 + 2c^2$ and that $a_{ij} = a_{ji}$, $a'_{ij} = a'_{ji}$ we have
\begin{equation*}
(Z - Z^{(ij)})^{2} \leq 2 (a'_{ij} - a_{ij})^{2} (2 t_i - 2 d_i + a_{ij} - a'_{ij})^{2} \frac{\|\tilde{X}_i \|^{4}}{t_i^4} + 2 (a'_{ji} - a_{ji})^{2} (2 t_j - 2 d_j + a_{ji} - a'_{ji})^{2}
 \frac{\|\tilde{X}_j\|^{4}}{t_j^4},
\end{equation*}
from which we derive
\begin{equation*}
\begin{split}
\allowdisplaybreaks[2]
 \sum_{i < j} (Z - Z^{(ij)})^{2} 
 & \leq \sum_{i=1}^{n} \sum_{j=1}^{n} (a'_{ij} - a_{ij})^{2} \bigl(16( t_i - d_i)^{2} + 4 \bigr) \frac{\|\tilde{X}_i\|^4}{t_i^4} \\
 & \leq \sum_{i=1}^{n} \sum_{j=1}^{n} \bigl(16 (t_i - d_i)^2 + 4 \bigr) \frac{\|\tilde{X}_i\|^4}{t_i^4} \\
 & \leq \sum_{i=1}^{n} \sum_{j=1}^{n} \bigl(16 (t_i - d_i)^2 + 4 \bigr) \frac{\|\tilde{X}_i\|^2}{t_i^2} \frac{\rho_n t_i^{-1} \|X_i\|^2}{t_i^2} \\
 & \leq C \sum_{i=1}^{n} \sum_{j=1}^{n} \bigl(16 (t_i - d_i)^2 + 4 \bigr) \frac{\|\tilde{X}_i\|^2}{t_i^2} n^{-3} \rho_n^{-2} \\
 & \leq C_1 (n \rho_n)^{-2} Z + C_2 (n \rho_n)^{-4} \leq C_3 (n \rho_n)^{-2} Z
 \end{split}
\end{equation*}
for some constants $C_1, C_2, C_3 > 0$. Once again, we apply Theorem~\ref{THM:log_Sobolev} to conclude
\begin{gather}
  \label{eq:log-sobolev-part1a}
  \mathbb{P}[Z - \mathbb{E}[Z] > t] \leq \exp\Bigl( \frac{-Ct^2}{(n \rho_n)^{-2} \mathbb{E}[Z] + 2 (n \rho_n)^{-2} t} \Bigr) \\
  \mathbb{P}[Z - \mathbb{E}[Z] > -t] \leq \exp\Bigl( \frac{-Ct^2}{(n \rho_n)^{-2} \mathbb{E}[Z]} \Bigr).
\end{gather}
In addition, $\mathbb{E}[Z] = \mathbb{E}[\sum_{k} (t_k - d_k)^2 t_k^{-2} \|\tilde{X}_k\|^2] \leq C (n \rho_n)^{-1}$ for some constant $C > 0$; here the expectation is taken with respect to $\mathbf{D}$ conditional on $\mathbf{P}$. 
We thus conclude 
\begin{equation}
\label{eq:Z-Z2}
\begin{split}
Z &= \|(\mathbf{T} - \mathbf{D}) \mathbf{T}^{-1} \tilde{\mathbf{X}} \|_{F}^2 \\ &= \mathbb{E}[\|\mathbf{T} - \mathbf{D}) \mathbf{T}^{-1} \tilde{\mathbf{X}} \|_{F}^2 + O_{\mathbb{P}}((n \rho_n)^{-3/2} \log^{1/2}(n)). 
\end{split}
\end{equation}
We now evaluate $n \rho_n \mathbb{E}[\|\tfrac{1}{2} \mathbf{T}^{-1} (\mathbf{T} - \mathbf{D}) \tilde{\mathbf{X}} \|_{F}^2]$. We only sketch the argument, noting that the details follow in a similar manner to that used in deriving Eq.~\eqref{eq:slutsky-lse-part1}. 
We have that
\begin{equation*}
\begin{split}
n \rho_n \mathbb{E}[\|\frac{1}{2} \mathbf{T}^{-1} (\mathbf{T} - \mathbf{D}) \tilde{\mathbf{X}} \|_{F}^2] &= n \rho_n \frac{1}{4} \mathrm{tr} \tilde{\mathbf{X}}^{\top}
\mathbf{T}^{-1} \mathbb{E}[(\mathbf{T} - \mathbf{D})^2] \mathbf{T}^{-1} \tilde{\mathbf{X}}.
\end{split}
\end{equation*} 
Now $\mathbf{T}^{-1} \mathbb{E}[(\mathbf{T} - \mathbf{D})^{2}] \mathbf{T}^{-1}$ is a diagonal matrix whose $i$-th diagonal entry is of the form
$ t_i^{-2} \sum_{j} p_{ij} (1 - p_{ij})$. Hence, 
\begin{equation*}
\begin{split}
 n \rho_n \tilde{\mathbf{X}}^{\top} \mathbf{T}^{-1} \mathbb{E}[(\mathbf{T} - \mathbf{D})^{2}] \mathbf{T}^{-1} \tilde{\mathbf{X}} &= n \rho_n^2 \sum_{i} t_i^{-3} X_i X_i^{\top} \sum_{j} p_{ij} (1 - p_{ij})  \\
 &= \sum_{i} \frac{n \rho_n^2 X_i X_i^{\top}}{(n \rho_n X_i^{\top} \mu)^{-3}} \sum_{j} \rho_n X_i^{\top} X_j (1 - \rho_n X_i^{\top} X_j) + o_{\mathbb{P}}(1)\\
 &= \sum_{i} n^{-1} \frac{X_i X_i^{\top}}{(X_i^{\top} \mu)^3} \sum_{j} n^{-1} X_i^{\top} X_j (1 - \rho_n X_i^{\top} X_j) + o_{\mathbb{P}}(1)
 \\ &= \sum_{i} n^{-1} \frac{X_i X_i^{\top}}{(X_i^{\top} \mu)^3} \sum_{j} n^{-1} X_i^{\top} X_j (1 - \rho_n X_j^{\top} X_i) + o_{\mathbb{P}}(1)
 \end{split}
\end{equation*}
We therefore have
\begin{equation}
\label{eq:slutsky-lse-part2}
\begin{split}
 n \rho_n \mathbb{E}[\|\frac{1}{2} \mathbf{T}^{-1} (\mathbf{T} - \mathbf{D}) \tilde{\mathbf{X}} \|_{F}^2] & \overset{\mathrm{a.s.}}{\longrightarrow} \frac{1}{4} \mathrm{tr} \Bigl( \mathbb{E}\Bigl[\frac{X_i X_i^{\top}}{(X_i^{\top} \mu)^2} \Bigl(1 - \frac{\rho_n X_i^{\top} \Delta X_i}{(X_i^{\top} \mu)^3}\Bigr) \Bigr]
 \end{split}
 \end{equation}
as $n \rightarrow \infty$. 

Finally we consider $Z := n \rho_n \mathrm{tr} \,\, \tilde{\mathbf{X}}^{\top} \mathbf{T}^{-1} (\mathbf{T} - \mathbf{D}) \mathbf{T}^{-1/2}(\mathbf{A} - \mathbf{P}) \mathbf{T}^{-1/2} \tilde{\mathbf{X}} (\tilde{\mathbf{X}}^{\top} \tilde{\mathbf{X}})^{-1}$. A similar, albeit slightly more tedious, argument to that used in deriving Eq.~\eqref{eq:LSE-main-part1} and Eq.~\eqref{eq:Z-Z2} yields
\begin{equation*}
\begin{split}
Z &= \mathrm{tr} \,\, \tilde{\mathbf{X}}^{\top} \mathbf{T}^{-1} (\mathbf{T} - \mathbf{D}) \mathbf{T}^{-1/2}(\mathbf{A} - \mathbf{P}) \mathbf{T}^{-1/2} \tilde{\mathbf{X}} (\tilde{\mathbf{X}}^{\top} \tilde{\mathbf{X}})^{-1} \\
&= \mathrm{tr} \,\, \mathbb{E} \Bigl[ \tilde{\mathbf{X}}^{\top} \mathbf{T}^{-1} (\mathbf{T} - \mathbf{D}) \mathbf{T}^{-1/2}(\mathbf{A} - \mathbf{P}) \mathbf{T}^{-1/2} \tilde{\mathbf{X}} (\tilde{\mathbf{X}}^{\top} \tilde{\mathbf{X}})^{-1} \Bigr] + O_{\mathbb{P}}((n \rho_n)^{-3/2} \log^{1/2}{n}).
\end{split}
\end{equation*}
We now evaluate $\mathbb{E}[Z]$. We have
$$ \mathbb{E}[Z] = \mathrm{tr} \,\, \tilde{\mathbf{X}}^{\top} \mathbf{T}^{-3/2} \mathbb{E} \Bigl[(\mathbf{T} - \mathbf{D}) (\mathbf{A} - \mathbf{P}) \Bigr] \mathbf{T}^{-1/2} \tilde{\mathbf{X}} (\tilde{\mathbf{X}}^{\top} \tilde{\mathbf{X}})^{-1} $$
Now the $ij$-th entry of $\mathbb{E}[(\mathbf{A} - \mathbf{P})(\mathbf{T} - \mathbf{D})]$ is of the form
\begin{equation*}
\begin{split}
\mathbb{E}[(a_{ij} - p_{ij})(t_j - d_j)] &= \mathbb{E}[(a_{ij} - p_{ij}) \sum_{k} (p_{jk} - a_{jk})] \\ &= \sum_{k} \mathbb{E}[(a_{ij} - p_{ij})(p_{kj} - a_{kj})] = - p_{ij}(1 - p_{ij}),
\end{split}
\end{equation*} 
and hence, with $\circ$ denoting the Hadamard product of matrices, 
\begin{equation}
\begin{split}
 n \rho_n \mathbb{E}[Z] 
&= - n \rho_n \mathrm{tr} \,\, \tilde{\mathbf{X}}^{\top}  \mathbf{T}^{-3/2} (\mathbf{P} - \mathbf{P} \circ \mathbf{P}) \mathbf{T}^{-1/2} \tilde{\mathbf{X}} (\tilde{\mathbf{X}}^{\top} \tilde{\mathbf{X}})^{-1} \\
&= - n \rho_n \mathrm{tr} \,\, \tilde{\mathbf{X}}^{\top} \mathbf{T}^{-1} \bigl(\mathbf{T}^{-1/2} \mathbf{P} \mathbf{T}^{-1/2} - \mathbf{T}^{-1/2} (\mathbf{P} \circ \mathbf{P}) \mathbf{T}^{-1/2} \bigr) \tilde{\mathbf{X}} (\tilde{\mathbf{X}}^{\top} \tilde{\mathbf{X}})^{-1} \\
&= - n \rho_n \mathrm{tr} \,\, \tilde{\mathbf{X}}^{\top} \mathbf{T}^{-1} \bigl(\tilde{\mathbf{X}} \tilde{\mathbf{X}}^{\top} - \mathbf{T}^{-1/2} (\mathbf{P} \circ \mathbf{P}) \mathbf{T}^{-1/2} \bigr) \tilde{\mathbf{X}} (\tilde{\mathbf{X}}^{\top} \tilde{\mathbf{X}})^{-1} \\
&= - n \rho_n \mathrm{tr} \,\, \tilde{\mathbf{X}}^{\top} \mathbf{T}^{-1} \tilde{\mathbf{X}} + n \rho_n \mathrm{tr} \,\, \tilde{\mathbf{X}}^{\top} \mathbf{T}^{-3/2} (\mathbf{P} \circ \mathbf{P}) \mathbf{T}^{-1/2} \tilde{\mathbf{X}} (\tilde{\mathbf{X}}^{\top} \tilde{\mathbf{X}})^{-1}.
\end{split}
\end{equation}
We first consider the term $n \rho_n \mathrm{tr} \,\,  \tilde{\mathbf{X}}^{\top} \mathbf{T}^{-1} \tilde{\mathbf{X}}$. We have
\begin{equation*}
\begin{split}
 n \rho_n \mathrm{tr} \,\,  \tilde{\mathbf{X}}^{\top} \mathbf{T}^{-1} \tilde{\mathbf{X}} &= - n \rho_n \sum_{i} \frac{\rho_n X_i X_i^{\top}}{t_i^2} = - \frac{1}{n} \sum_{i} \frac{X_i X_i^{\top}}{(X_i^{\top} \mu)^2} + o_{\mathbb{P}}(1), 
 \end{split}
 \end{equation*}
 and hence
 \begin{equation}
 \label{eq:slutsky-lse-part3a}
 - n \rho_n \mathrm{tr} \,\,  \tilde{\mathbf{X}}^{\top} \mathbf{T}^{-1} \tilde{\mathbf{X}}
  \overset{\mathrm{a.s.}}{\longrightarrow} - \mathrm{tr} \,\, \mathbb{E}\Bigl[ \frac{X_1 X_1^{\top}}{(X_1^{\top} \mu)^2}\Bigr]. 
 \end{equation}
 Finally, we consider the term $n \rho_n \mathrm{tr} \,\, \tilde{\mathbf{X}}^{\top} \mathbf{T}^{-3/2} (\mathbf{P} \circ \mathbf{P}) \mathbf{T}^{-1/2} \tilde{\mathbf{X}} (\tilde{\mathbf{X}}^{\top} \tilde{\mathbf{X}})^{-1}$.
 We recall that 
 $(\tilde{\mathbf{X}}^{\top} \tilde{\mathbf{X}})^{-1} \overset{\mathrm{a.s.}}{\longrightarrow} \tilde{\Delta}^{-1}$ as $n \rightarrow \infty$. In addition,
 \begin{equation*} 
 \begin{split}
 n \rho_n \tilde{\mathbf{X}}^{\top} \mathbf{T}^{-3/2} (\mathbf{P} \circ \mathbf{P}) \mathbf{T}^{-1/2} \tilde{\mathbf{X}} &=
 n \rho_n^2 \mathbf{X}^{\top} \mathbf{T}^{-2} (\mathbf{P} \circ \mathbf{P}) \mathbf{T}^{-1} \mathbf{X} \\
 &= n \rho_n^2 \sum_{i} \sum_{j} \frac{p_{ij}^2}{t_i^2 t_j} X_i X_j^{\top} \\
 &= n \rho_n ^2 \sum_{i} \sum_{j} \frac{p_{ij}^2}{(n \rho_n)^3 (X_i^{\top} \mu)^2 X_j^{\top} \mu} X_i X_j^{\top} + o_{\mathbb{P}}(1) \\
 &= n \rho_n^2 \sum_{i} \sum_{j} \frac{ \rho_n^2 (X_i^{\top} X_j)^2}{(n \rho_n)^3 (X_i^{\top} \mu)^2 X_j^{\top} \mu} X_i X_j^{\top} + o_{\mathbb{P}}(1) \\
 &= \rho_n \sum_{i} \frac{1}{n} \sum_{j} \frac{1}{n} \frac{X_i^{\top} X_j X_j^{\top} X_i}{(X_i^{\top} \mu)^2 X_j^{\top} \mu} X_i X_j^{\top} + o_{\mathbb{P}}(1). 
\end{split}
\end{equation*}
We thus conclude 
\begin{equation}
\label{eq:slutksy-lse-part3b}
n \rho_n \tilde{\mathbf{X}}^{\top} \mathbf{T}^{-3/2} (\mathbf{P} \circ \mathbf{P}) \mathbf{T}^{-1/2} \tilde{\mathbf{X}} (\tilde{\mathbf{X}}^{\top} \tilde{\mathbf{X}})^{-1} \overset{\mathrm{a.s.}}{\longrightarrow} \rho_n \mathrm{tr} \,\, \mathbb{E}\Bigl[\frac{X_1^{\top} X_2 X_2^{\top} X_1}{(X_1^{\top} \mu)^2 X_2^{\top} \mu} X_1 X_2^{\top}\Bigr] \tilde{\Delta}^{-1}
\end{equation}
 where the expectation is taken with respect to $X_1, X_2$ being i.i.d drawn from $F$. 
 Combining Eq.~\eqref{eq:slutsky-lse-part3a} and Eq.~\eqref{eq:slutksy-lse-part3b} yield
 \begin{equation}
 \begin{split}
 \label{eq:slutsky-lse-part3}
 n \rho_n \mathrm{tr} \,\, \mathbb{E} \Bigl[ \tilde{\mathbf{X}}^{\top} \mathbf{T}^{-1} (\mathbf{T} - \mathbf{D}) \mathbf{T}^{-1/2}(\mathbf{A} - \mathbf{P}) \mathbf{T}^{-1/2} \tilde{\mathbf{X}} (\tilde{\mathbf{X}}^{\top} \tilde{\mathbf{X}})^{-1} \Bigr] \\
 \overset{\mathrm{a.s.}}{\longrightarrow} \rho_n \mathrm{tr} \,\, \mathbb{E}\Bigl[\frac{X_1^{\top} X_2 X_2^{\top} X_1}{(X_1^{\top} \mu)^2 X_2^{\top} \mu} X_1 X_2^{\top}\Bigr] \tilde{\Delta}^{-1} - \mathrm{tr} \,\, \mathbb{E}\Bigl[ \frac{X_1 X_1^{\top}}{(X_1^{\top} \mu)^2}\Bigr]. 
 \end{split}
 \end{equation}
Eq.~\eqref{eq:4} and Eq.~\eqref{eq:14} then follows directly from Eq.~\eqref{eq:slutsky-lse-part1}, Eq.~\eqref{eq:slutsky-lse-part2} and 
Eq.~\eqref{eq:slutsky-lse-part3}.

\section{Within-block variances}
\label{sec:within-variance-appendix}
We now verify that Theorem~\ref{THM:between-ASE} and Theorem~\ref{THM:between-LSE} are indeed generalizations of Theorem~3.1 and Theorem~3.2 from \cite{bickel_sarkar_2013}. Suppose that $K = d$, i.e., that $\mathbf{B}$ is invertible. Then denoting by $\bm{\nu}$ the $d \times d$ matrix $\bm{\nu} = \bigl[ \nu_1 \mid \nu_2 \mid \cdots \mid \nu_d \bigr]$, we have that $\bm{\nu}$ is also invertible and that $\mathbf{B} = \bm{\nu}^{\top} \bm{\nu}$ and $\Delta = \bm{\nu} \mathrm{diag}(\bm{\pi}) \bm{\nu}^{\top}$. Let $\bm{z}_k = (\nu_k^{\top} \nu_1 (1 - \nu_k^{\top} \nu_1), \cdots, \nu_{k}^{\top} \nu_d (1 - \nu_{k}^{\top} \nu_d))$. Then
$$ \mathbb{E}[X_1 X_1^{\top} (\nu_k^{\top} X_1 - \nu_k^{\top} X_1 X_1^{\top} \nu_k)] = \bm{\nu} \bigl(\mathrm{diag}(\bm{\pi}) \mathrm{diag}(\bm{z}_k) \bigr) \bm{\nu}^{\top}. $$
Then Eq.~\eqref{eq:d_kk1} in Theorem~\ref{THM:between-ASE} simplifies to
\begin{equation}
\label{eq:d_kk_invertible}
\begin{split}
 n^2 \hat{d}_{kk} & \overset{\mathrm{a.s.}}{\longrightarrow} \mathrm{tr} \,\, \Delta^{-3} \mathbb{E}[X_1 X_1^{\top} (\nu_k^{\top} X_1 - \nu_k^{\top} X_1 X_1^{\top} \nu_k)] \\ &= \mathrm{\tr} \,\, (\bm{\nu} \mathrm{diag}(\bm{\pi}) \bm{\nu}^{\top})^{-3} \bm{\nu} \bigl(\mathrm{diag}(\bm{\pi}) \mathrm{diag}(\bm{z}_k) \bigr) \bm{\nu}^{\top} \\
 &= \mathrm{\tr} \,\, ((\bm{\nu}^{\top})^{-1} \mathrm{diag}(\bm{\pi})^{-1} \bm{\nu}^{-1})^{3} \bm{\nu} \bigl(\mathrm{diag}(\bm{\pi}) \mathrm{diag}(\bm{z}_k) \bigr) \bm{\nu}^{\top} \\
 &= \mathrm{\tr} \,\, \bigl(\mathrm{diag}(\bm{\pi})^{-1} \bm{\nu}^{-1} (\bm{\nu}^{\top})^{-1}\bigr)^{2} \mathrm{diag}(\bm{z}_k) 
 \\ &= \mathrm{tr} \,\, \bigl(\mathrm{diag}(\bm{\pi})^{-1} \mathbf{B}^{-1}\bigr)^{2} \mathrm{diag}(\bm{z}_k) \\
 &= \mathrm{tr} \,\, \bigl(\mathrm{diag}(\bm{\pi})^{-1/2} \mathbf{B}^{-1} \mathrm{diag}(\bm{\pi})^{-1/2}\bigr)^{2} \mathrm{diag}(\bm{z}_k)
 \\ &= \sum_{l} \sum_{l'} \frac{\nu_k^{\top} \nu_l (1 - \nu_k^{\top} \nu_l)(\mathbf{B}^{-1}_{ll'})^{2}}{\pi_{l} \pi_{l'}} \\ &= \sum_{l} \sum_{l'} \frac{\mathbf{B}_{kl} (1 - \mathbf{B}_{kl}) (\mathbf{B}^{-1}_{ll'})^{2}}{\pi_{l} \pi_{l'}}
 \end{split}
\end{equation}
where $\mathbf{B}^{-1}_{ll'}$ is the $ll'$-th entry of $\mathbf{B}^{-1}$. We note that the above expression for $\hat{d}_{kk}$ can be written purely in terms of the entries of $\mathbf{B}$ and $\bm{\pi}$ without the need to find the $\nu_1, \nu_2, \dots, \nu_d$ explicitly.

We compare Eq.~\eqref{eq:d_kk_invertible} with Theorem~3.1 in \cite{bickel_sarkar_2013}. Let $\mathbf{A}$ be sampled from a stochastic blockmodel with parameters $\mathbf{B} = \Bigl[\begin{matrix} \alpha_n & \beta_n \\ \beta_n & \gamma_n \end{matrix} \Bigr]$ and $\bm{\pi} = (\pi_1, \pi_2)$ with $\alpha_n \beta_n \not = \gamma_n^2$. In 
\cite{bickel_sarkar_2013}, it is assume that the number of vertices assigned to block $1$ and block $2$ are $n \pi_1$ and $n \pi_2$, respectively. For ease of exposition and without loss of generality, suppose that the row indices of $\mathbf{A}$ are such that the first $n \pi_1$ rows correspond to vertices assigned to block $1$ and the last $n \pi_2 = n - n \pi_1$ rows correspond to vertices assigned to block $2$.
Let $\bm{v}_1$ and $\bm{v}_2$ denote the eigenvectors corresponding to the largest and second largest eigenvector of $\mathbf{P} = \mathbf{Z} \mathbf{B} \mathbf{Z}^{\top}$ where $\mathbf{Z}$ is a $n \times 2$ matrix whose $i$-th row is $(1,0)$ for $i = 1,2,\dots, n\pi_1$ and is $(0,1)$ for $i = n \pi_1 + 1, n \pi_1 + 2, \dots, n$. We then have that $\bm{v}_1 = (x_1, x_1, \dots, x_1, y_1, y_1, \dots, y_1)$ for some $x_1$, $y_1$, i.e., the first $n \pi_1$ elements of $\bm{v}_1$ are $x_1$ and the remaining $n \pi_2$ elements are $x_2$.
Similarly, we have $\bm{v}_2 = (x_2, x_2, \dots, x_2, y_2, y_2, \dots, y_2)$ for some $x_2, y_2$. Then Eq.~(3.1) in \cite{bickel_sarkar_2013} states that (the notation $a_n \sim b_n$ in \cite{bickel_sarkar_2013} means $a_n/b_n = 1 + o_{\mathbb{P}}(1)$)
\begin{equation}
\label{eq:Sarkar1}
 \hat{d}_{11} \sim \Bigl[ \Bigl(\frac{x_1^{2}}{\lambda_1^2} + \frac{x_2^{2}}{\lambda_2^2}\Bigr) n \pi_1 \alpha_n (1 - \alpha_n) + 
\Bigl(\frac{y_1^{2}}{\lambda_1^2} + \frac{y_2^{2}}{\lambda_2^2}\Bigr) n \pi_2 \gamma_n(1 - \gamma_n) \Bigr]
\end{equation}
where $\lambda_1$ and $\lambda_2$ are the largest and second largest eigenvalues of $\mathbf{P}$. We can rewrite Eq.~\eqref{eq:Sarkar1} as
\begin{equation}
\hat{d}_{11} \sim \mathrm{tr} \,\, (\mathbf{P}^{\dagger})^{2} \mathrm{diag}((\alpha_n(1 - \alpha_n), \dots \gamma_n(1 - \gamma_n), \dots))
\end{equation}
where $\mathbf{P}^{\dagger}$ is the Moore-Penrose pseudo-inverse of $\mathbf{P}$ and the first $n \pi_1$ entries of the diagonal matrix $\mathrm{diag}(\alpha_n(1 - \alpha_n), \dots \gamma_n(1 - \gamma_n), \dots)$ are $\alpha_n(1 - \alpha_n)$ while the remaining $n \pi_2$ diagonal entries are $\gamma_n(1 - \gamma_n)$. As $\mathbf{Z}$ is of full-column rank, we have $\mathbf{Z}^{\dagger} = (\mathbf{Z}^{\top} \mathbf{Z})^{-1} \mathbf{Z}^{\top} = \mathrm{diag}((1/(n \pi_1), 1/(n \pi_2))) \mathbf{Z}^{\top}$. Furthermore, $\mathbf{B}$ is invertible and hence
\begin{equation*} 
\mathbf{P}^{\dagger}  = (\mathbf{Z} \mathbf{B} \mathbf{Z}^{\top})^{\dagger} = (\mathbf{Z}^{\top})^{\dagger} \mathbf{B}^{-1} \mathbf{Z}^{\dagger} 
= n^{-2} \mathbf{Z} \mathrm{diag}(\bm{\pi})^{-1} \mathbf{B}^{-1} \mathrm{diag}(\bm{\pi})^{-1} \mathbf{Z}^{\top}.
\end{equation*}
Therefore,
\begin{equation*}
\begin{split}
(\mathbf{P}^{\dagger})^{2}  %
& = 
n^{-3} \mathbf{Z} \mathrm{diag}(\bm{\pi})^{-1} \mathbf{B}^{-1} \mathrm{diag}(\bm{\pi})^{-1} \mathbf{B}^{-1} \mathrm{diag}(\bm{\pi})^{-1} \mathbf{Z}^{\top}
\end{split}
\end{equation*}
and hence
\begin{equation*}
\begin{split}
n^2 \hat{d}_{11} & \sim n^2 \mathrm{tr} \,\, (\mathbf{P}^{\dagger})^{2} \mathrm{diag}((\alpha_n(1 - \alpha_n), \dots, \gamma_n(1 - \gamma_n), \dots)) \\ & \sim n^{-1} \mathrm{tr} \,\, \mathbf{Z} (\mathrm{diag}(\bm{\pi})^{-1} \mathbf{B}^{-1})^{2} \mathrm{diag}(\bm{\pi})^{-1} \mathbf{Z}^{\top} \mathrm{diag}((\alpha_n(1 - \alpha_n), \dots \gamma_n(1 - \gamma_n), \dots)) \\
& \sim \mathrm{tr} \,\, (\mathrm{diag}(\bm{\pi})^{-1} \mathbf{B}^{-1})^{2} \mathrm{diag}((\alpha_n(1 - \alpha_n),\gamma_n(1 - \gamma_n)))
\end{split}
\end{equation*}
which is a special case of Eq.~\eqref{eq:d_kk_invertible}. Theorem~\ref{THM:between-ASE} is thus an extension of Theorem~3.1 in \cite{bickel_sarkar_2013} to general $K$-block stochastic blockmodels, provided that the block probability matrix is positive semidefinite.

We now consider $n^{2} \tilde{d}_{kk}$. When $\mathbf{B}$ is invertible, Eq.~\eqref{eq:d_kk_tilde1} in Theorem~\ref{THM:between-LSE} can be simplified in a manner similar to the derivation of Eq.~\eqref{eq:d_kk_invertible}. Let $\bm{\mu} = (\mu_1, \mu_2, \dots \mu_d)$ where $\mu_k = \nu_k^{\top} \mu$. Then $\tilde{\Delta} = \bm{\nu} (\mathrm{diag}(\bm{\pi})\mathrm{diag}(\bm{\mu})^{-1}) \bm{\nu}^{\top}$. The right hand side of Eq.~\eqref{eq:d_kk_tilde1} can be decompose as $\zeta_1 - \zeta_2 + \zeta_3$ with $\zeta_1$ given by 
\begin{equation*}
\begin{split}
\zeta_1 &= \mathrm{tr} \,\, \tilde{\Delta}^{-3} \mathbb{E}\Bigl[ \frac{X_1 X_1^{\top}}{(X_1^{\top} \mu)^2} \frac{\nu_k^{\top} X_1 - \nu_k^{\top} X_1 X_1^{\top} \nu_k}{\mu_k} \Bigr] \\
&= \frac{1}{\mu_k} \mathrm{tr} \,\, \tilde{\Delta}^{-3} \bm{\nu}^{\top} (\mathrm{diag}(\bm{\pi}) \mathrm{diag}(\bm{\mu})^{-2} \mathrm{diag}(\bm{z}_k)) \bm{\nu}^{\top} \\
&= \frac{1}{\nu_k^{\top} \mu} \mathrm{tr} \,\, (\mathrm{diag}(\bm{\pi})^{-1}\mathrm{diag}(\bm{\mu}) \bm{\nu}^{-1} (\bm{\nu}^{\top})^{-1})^{2} \mathrm{diag}(\bm{\mu})^{-1}\mathrm{diag}(\bm{z}_k) \\
&= \frac{1}{\mu_k} \mathrm{tr} \,\, (\mathrm{diag}(\bm{\pi})^{-1}\mathrm{diag}(\bm{\mu}) \mathbf{B}^{-1})^{2} \mathrm{diag}(\bm{\mu})^{-1} \mathrm{diag}(\bm{z}_k) \\
&= \frac{1}{\mu_k} \mathrm{tr} \,\, (\mathrm{diag}(\bm{\pi})^{-1/2} \mathrm{diag}(\bm{\mu})^{1/2} \mathbf{B}^{-1} \mathrm{diag}(\bm{\mu})^{1/2} \mathrm{diag}(\bm{\pi})^{-1/2})^{2} \mathrm{diag}(\bm{\mu})^{-1}\mathrm{diag}(\bm{z}_k) \\
&= \sum_{l} \sum_{l'} \frac{(\mathbf{B}^{-1}_{ll'})^{2} \mu_{l} \mu_{l'}}{\pi_l \pi_{l'}} \frac{\nu_k^{\top} \nu_l (1 - \nu_k^{\top} \nu_l)}{\mu_l \mu_k}
= \sum_{l} \sum_{l'} \frac{\mathbf{B}_{kl}(1 - \mathbf{B}_{kl})(\mathbf{B}^{-1}_{ll'})^{2} \mu_{l'}}{\pi_l \pi_{l'} \mu_k}. 
\end{split}
\end{equation*}
Let $\bm{e}_k$ denote the vector whose $i$-th element is $1$ if $i = k$ and $0$ otherwise. For $\zeta_2$, we have 
\begin{equation*}
\begin{split}
\zeta_2 &= \mathrm{tr} \,\, \tilde{\Delta}^{-2} \mathbb{E}\Bigl[ \frac{X_1 
 \nu_k^{\top}}{X_1^{\top} \mu} \frac{\nu_k^{\top} X_1 - \nu_k^{\top} X_1 X_1^{\top} \nu_k}{\mu_k^2} \Bigr] \\
&= \frac{1}{\mu_k^2} \mathrm{tr} \,\, \tilde{\Delta}^{-2} \bm{\nu}^{\top} (\mathrm{diag}(\bm{\pi}) \mathrm{diag}(\bm{\mu})^{-1} \mathrm{diag}(\bm{z}_k)) \bm{1} \nu_k^{\top} \\
&= \frac{1}{\mu_k^2} \mathrm{tr} \,\, \mathrm{diag}(\bm{\pi})^{-1} \mathrm{diag}(\bm{\mu}) \bm{\nu}^{-1} (\bm{\nu}^{\top})^{-1} \mathrm{diag}(\bm{z}_k) \bm{1} \bm{e}_k^{\top} \\
&= \frac{1}{\mu_k^{2}} \mathrm{tr} \,\, \mathrm{diag}(\bm{\pi})^{-1}\mathrm{diag}(\bm{\mu}) \mathbf{B}^{-1} \mathrm{diag}(\bm{z}_k) \bm{1} \bm{e}_k^{\top} \\
&= \frac{1}{\pi_k \mu_k} \sum_{l} \nu_k^{\top} \nu_l (1 - \nu_k^{\top} \nu_l) \mathbf{B}^{-1}_{kl} = \frac{1}{\pi_k \mu_k} \sum_{l} \mathbf{B}_{kl}(1 - \mathbf{B}_{kl}) \mathbf{B}^{-1}_{kl}.
\end{split}
\end{equation*}
Finally for $\zeta_3$ we have
\begin{equation*}
\begin{split}
\zeta_3 &= \frac{1}{4 \mu_k^3} \tr \,\, \tilde{\Delta}^{-1} \nu_k \nu_k^{\top} \mathbb{E}[\nu_k^{\top} X_1 - \nu_k^{\top} X_1 X_1^{\top} \nu_k] \\
&= \frac{\mathbb{E}[\nu_k^{\top} X_1 - \nu_k^{\top} X_1 X_1^{\top} \nu_k]}{4 \mu_k^3} \mathrm{tr} \,\, \nu_k^{\top} \tilde{\Delta}^{-1} \nu_k \\
&= \frac{\mathbb{E}[\nu_k^{\top} X_1 - \nu_k^{\top} X_1 X_1^{\top} \nu_k]}{4 \mu_k^3} \mathrm{tr} \,\, \nu_k^{\top} (\bm{\nu}^{\top})^{-1}
(\mathrm{diag}(\bm{\pi})^{-1} \mathrm{diag}(\bm{\mu})) \bm{\nu}^{-1} \nu_k \\
&= \frac{\mathbb{E}[\nu_k^{\top} X_1 - \nu_k^{\top} X_1 X_1^{\top} \nu_k]}{4 \mu_k^3} \frac{\mu_k}{\pi_k} = \frac{\sum_{l} \pi_l \mathbf{B}_{kl}(1 - \mathbf{B}_{kl})}{4 \pi_k \mu_k^2}.
\end{split}
\end{equation*}
As $\mu_k = \sum_{l} \pi_l \nu_k^{\top} \nu_l = \sum_{l} \pi_l \mathbf{B}_{kl}$, $\zeta_1$, $\zeta_2$ and $\zeta_3$ can also be written purely in terms of the entries of $\mathbf{B}$ and $\bm{\pi}$.

For the two-block stochastic blockmodel, Eq.~(3.3) in \cite{bickel_sarkar_2013} states that
\begin{equation}
\label{eq:sarkar-2}
n^{2} \tilde{d}_{11} \sim \frac{\alpha_n(1 - \alpha_n)}{\mu_1^2} \Bigl(\frac{1}{4} + \frac{\pi_2 \gamma_n}{\mu_1 \tilde{\lambda}_2^2}\Bigr) + \frac{\gamma_n (1 - \gamma_n)}{\mu_1^2} \Bigl(\frac{\pi_2}{4 \pi_1} + \frac{\pi_1 \alpha_n}{\mu_2 \tilde{\lambda}_2^2}\Bigr)
\end{equation}
where $\tilde{\lambda}_2 = \pi_1 \pi_2 (\alpha_n \beta_n - \gamma_n^2)/(\mu_1 \mu_2)$ is the second largest eigenvalue of $\mathcal{L}(\mathbf{P})$ (c.f. Lemma~6.1 in \cite{bickel_sarkar_2013}).
Verifying that $\zeta_1 - \zeta_2 + \zeta_3$ does indeed yield Eq.~\eqref{eq:sarkar-2} for the two-block stochastic blockmodel is a straightforward computation. We omit the details. Theorem~\ref{THM:between-LSE} is thus an extension of Theorem~3.2 in \cite{bickel_sarkar_2013} for general $K$-blocks stochastic blockmodels whenever the matrix of block probabilities is positive semidefinite. 

\end{document}